\documentclass{article}

\PassOptionsToPackage{numbers, compress}{natbib}
\usepackage[preprint]{neurips_2025}

\usepackage[utf8]{inputenc} % allow utf-8 input
\usepackage[T1]{fontenc}    % use 8-bit T1 fonts
\usepackage[colorlinks,linkcolor=blue,citecolor=blue,pagebackref=true]{hyperref}
\usepackage{url}            % simple URL typesetting
\usepackage{booktabs}       % professional-quality tables
\usepackage{amsfonts}       % blackboard math symbols
\usepackage{nicefrac}       % compact symbols for 1/2, etc.
\usepackage{microtype}      % microtypography
\usepackage{xcolor}         % colors

\usepackage{eqnarray,amsmath}
\usepackage[utf8]{inputenc} % allow utf-8 input
\usepackage[T1]{fontenc}    % use 8-bit T1 fonts
\usepackage{booktabs}       % professional-quality tables
\usepackage{amsfonts}       % blackboard math symbols
\usepackage{nicefrac}       % compact symbols for 1/2, etc.
\usepackage{microtype}      % microtypography

\usepackage{caption}
\usepackage{subcaption}
\usepackage[utf8]{inputenc}
\usepackage[english]{babel}
\usepackage{amsthm,amsmath,mathtools,amsfonts,amssymb}

\usepackage{pdfpages} 

\usepackage{authblk}
\newcommand\blfootnote[1]{%
  \begingroup
  \renewcommand\thefootnote{}\footnote{#1}%
  \addtocounter{footnote}{-1}%
  \endgroup
}

\usepackage{caption}
\usepackage{subcaption}
\usepackage[subrefformat=parens]{subcaption}

\usepackage{graphicx,float}
\usepackage{enumitem}

\usepackage{bbm}

\usepackage{dirtytalk}

\usepackage{url}% for url's in bib

\renewcommand*{\backrefalt}[4]{%
	\ifcase #1 \footnotesize{(Not cited.)}%
	\or        \footnotesize{(Cited on page~#2.)}%
	\else      \footnotesize{(Cited on pages~#2.)}%
	\fi}

\usepackage[nameinlink,capitalize,noabbrev]{cleveref}

\usepackage{algorithm}
\usepackage{algorithmicx}
\usepackage{algpseudocode}

\usepackage{nicefrac}

\usepackage[many]{tcolorbox}
\tcbset{
    every box/.style={
            highlight math style={sharp corners,colback=white,colframe=red},
        }
}

\newcounter{algsubstate}
\renewcommand{\thealgsubstate}{\alph{algsubstate}}

\DeclareMathOperator*{\TV}{D_{TV}}
\DeclareMathOperator*{\hels}{D_{h}^2}
\DeclareMathOperator*{\hel}{D_{h}}
\DeclareMathOperator*{\rep}{rep}
\DeclareMathOperator*{\nums}{num}
\DeclareMathOperator*{\VD}{D_{V}}
\DeclareMathOperator{\VDT}{\widetilde{D}}
\DeclareMathOperator{\VDS}{D_2}
\DeclareMathOperator{\VDC}{D_{Conjecture}}
\DeclareMathOperator{\W}{W}
\DeclareMathOperator*{\dsc}{DSC}
\DeclareMathOperator{\doout}{\R}

\let\inf\relax 
\DeclareMathOperator*\inf{\vphantom{p}inf}

\theoremstyle{plain}
\newtheorem{theorem}{Theorem}
\newtheorem{lemma}[theorem]{Lemma}

\newtheorem{fact}[theorem]{Fact}

\newtheorem{assumption}{Assumption}[section]
\theoremstyle{definition}
\newtheorem{definition}[theorem]{Definition}

\crefname{lemma}{Lemma}{Lemmas}
\crefname{fact}{Fact}{Facts}
\crefname{theorem}{Theorem}{Theorems}
\crefname{proposition}{Proposition}{Propositions}
\theoremstyle{plain}
\usepackage{epsf}
\usepackage{epsfig}
\usepackage{fancyhdr}
\usepackage{graphics}
\usepackage{graphicx}
\usepackage{psfrag}

\usepackage{url}% for url's in bib
\usepackage{color}

\usepackage{amsmath}
\usepackage{amssymb,bbm}
\usepackage{caption}
\usepackage{textcomp}
\usepackage{siunitx}
\usepackage{wrapfig}
\usepackage{multirow}
\usepackage{multicol}% colors

\usepackage{comment}

\newcommand{\Ns}{\mathbb{N}} % Define natural number without 0.
\newcommand{\N}{\mathbb{N}} % Define natural number.

\newcommand{\cD}{\mathcal{D}}

\newcommand{\cE}{\mathcal{E}}
\newcommand{\cF}{\mathcal{F}}

\newcommand{\cI}{\mathcal{I}}
\newcommand{\cJ}{\mathcal{J}}

\newcommand{\cL}{\mathcal{L}}

\newcommand{\cN}{\mathcal{N}}
\newcommand{\cO}{\mathcal{O}}

\newcommand{\cS}{\mathcal{S}}

\newcommand{\cT}{\mathcal{T}}

\newcommand{\cX}{\mathcal{X}}

\newcommand{\cY}{\mathcal{Y}}

\newcommand{\norm}[1]{\|#1\|}
\newcommand{\ck}{\vc_k^0}
\newcommand{\gk}{\mGamma_k^0}
\newcommand{\ak}{\mA_k^0}
\newcommand{\bk}{\vb_k^0}
\newcommand{\sk}{\mSigma_k^0}

\DeclareMathOperator*{\argmax}{arg\,max}
\DeclareMathOperator*{\argmin}{arg\,min}

\usepackage{amsmath,amsfonts,bm}

\usepackage{xcolor}
\definecolor{forestgreen}{rgb}{0.13, 0.55, 0.13}
\newcommand{\forestgreen}[1]{{\color{forestgreen}#1}}

\definecolor{frenchblue}{rgb}{0.0, 0.45, 0.73}

\newcommand{\update}[1]{{\color{black}#1}}

\definecolor{cherryblossompink}{rgb}{1.0, 0.72, 0.77}

\definecolor{bittersweet}{rgb}{1.0, 0.44, 0.37}

\definecolor{navyblue}{rgb}{0.0, 0.0, 0.5}

\def\eqref#1{equation~\ref{#1}}
\def\1{\bm{1}}

\newcommand{\divclus}{\mathsf{d}}
\newcommand{\height}{\mathsf{h}}

\def\rz{{\textnormal{z}}}

\def\rve{e} %% For univariate output.

\def\rvx{{\mathbf{x}}}
\def\rvy{\textnormal{y}}

\def\vmu{{\bm{\mu}}}

\def\vtheta{{\bm{\theta}}}
\def\vTheta{{\bm{\Theta}}}
\def\veta{{\bm{\eta}}}

\def\valpha{{\bm{\alpha}}}

\def\va{{\bm{a}}}
\def\vb{{b}}
\def\vc{{\bm{c}}}

\def\vp{{\bm{p}}}

\def\vv{{\bm{v}}}

\def\vx{{\bm{x}}}
\def\vy{y}

\def\evalpha{{\alpha}}

\def\evmu{{\mu}}

\def\evtheta{{\theta}}
\def\evpi{{\pi}}

\def\evp{{p}}

\def\evv{{v}}

\def\mA{{\bm{A}}}
\def\mA{{\bm{a}}}
\def\mI{{\bm{I}}}

\def\mQ{{\bm{Q}}}

\def\mSigma{\sigma}
\def\mGamma{{\bm{\Gamma}}}

\DeclareMathAlphabet{\mathsfit}{\encodingdefault}{\sfdefault}{m}{sl}
\SetMathAlphabet{\mathsfit}{bold}{\encodingdefault}{\sfdefault}{bx}{n}

\def\gE{{\mathcal{E}}}

\def\gH{{\mathcal{H}}}

\def\gV{{\mathcal{V}}}

\def\sN{{\mathbb{N}}}

\def\sP{{\mathbb{P}}}

\def\sR{{\mathbb{R}}}
\def\sS{{\mathbb{S}}}

\def\sV{{\mathbb{V}}}

\def\emQ{{Q}}

\newcommand{\E}{\mathbb{E}}

\newcommand{\R}{\mathbb{R}}

\newcommand{\KL}{D_{\mathrm{KL}}}

\usepackage[capitalize,noabbrev]{cleveref}

\title{Model Selection for Gaussian-gated Gaussian Mixture of Experts Using Dendrograms of Mixing Measures}

\author{Tuan Thai$^{\star,1}$, TrungTin Nguyen$^{\star,\dagger,2,3}$, Dat Do$^{4}$, Nhat Ho$^{5}$, Christopher Drovandi$^{2,3}$\\
	$^{1}$University of Science - VNUHCM, Ho Chi Minh City, Vietnam.\\
	$^{2}$ARC Centre of Excellence for the Mathematical Analysis of Cellular System.\\
	$^{3}$School of Mathematical Sciences, Queensland University of Technology, Brisbane City, Australia.\\
	$^{4}$Department of Statistics, University of Michigan, Ann Arbor, USA.\\
	$^{5}$Department of Statistics and Data Science, University of Texas at Austin, Austin, USA.
}

\begin{document}
	
	\maketitle
	
	\begin{abstract}
		Mixture of Experts (MoE) models constitute a widely utilized class of ensemble learning approaches in statistics and machine learning, known for their flexibility and computational efficiency. They have become integral components in numerous state-of-the-art deep neural network architectures, particularly for analyzing heterogeneous data across diverse domains. Despite their practical success, the theoretical understanding of model selection, especially concerning the optimal number of mixture components or experts, remains limited and poses significant challenges. These challenges primarily stem from the inclusion of covariates in both the Gaussian gating functions and expert networks, which introduces intrinsic interactions governed by partial differential equations with respect to their parameters. In this paper, we revisit the concept of dendrograms of mixing measures and introduce a novel extension to Gaussian-gated Gaussian MoE models that enables consistent estimation of the true number of mixture components and achieves the pointwise optimal convergence rate for parameter estimation in overfitted scenarios. Notably, this approach circumvents the need to train and compare a range of models with varying numbers of components, thereby alleviating the computational burden, particularly in high-dimensional or deep neural network settings. Experimental results on synthetic data demonstrate the effectiveness of the proposed method in accurately recovering the number of experts. It outperforms common criteria such as the Akaike information criterion, the Bayesian information criterion, and the integrated completed likelihood, while achieving optimal convergence rates for parameter estimation and accurately approximating the regression function.

		%%
		%\iffalse
		\blfootnote{$^\star$Co-first author, $^{\dagger}$Corresponding author.}
		%\fi
	\end{abstract}

	\section{Introduction}\label{sec_introduction}

	{\bf Mixture of experts (MoE) models.} 
	MoE models were initially introduced as neural network-based architectures in \cite{jacobs_adaptive_1991,jordan_hierarchical_1994}. These conditional mixture density models have gained considerable attention, primarily due to their universal approximation capabilities~\citep{norets_approximation_2010, nguyen_universal_2016, nguyen_approximation_2019_MoE, nguyen_approximations_2021}, and their favorable convergence properties, as demonstrated in mixture of regression settings~\cite{do_strong_2025,ho_convergence_2022} and various MoE frameworks~\cite{nguyen_bayesian_2024,jiang_hierarchical_1999, nguyen_towards_2024, nguyen_demystifying_2023, nguyen_general_2024}. Notably, these advancements enhance both the approximation and convergence behavior compared to unconditional mixture models, as shown in earlier works~\cite{genovese_rates_2000, rakhlin_risk_2005, chong_risk_2024, nguyen_convergence_2013, shen_adaptive_2013, ho_convergence_2016, ho_strong_2016, nguyen_approximation_2020, nguyen_approximation_2023}. For comprehensive discussions on the theoretical foundations and practical implementations of MoE models across various fields, including biological data analysis, natural language processing, and computer vision, see \cite{yuksel_twenty_2012, nguyen_practical_2018, do_hyperrouter_2023, nguyen_model_2021, chen_towards_2022}.
	Moreover, MoE models offer a more expressive alternative to traditional mixture and mixture of regression models~\citep{fruhwirth_schnatter_handbook_2019,mclachlan_finite_2000}, primarily due to their ability to model both the expert components and the gating functions as functions of the input covariates. In regression contexts, a commonly adopted formulation consists of Gaussian experts in conjunction with either softmax-based or Gaussian-based gating mechanisms. A representative example is the Gaussian-gated Gaussian mixture of experts (GGMoE) model~\citep{deleforge_high_dimensional_2015}, which generalizes several classical approaches, including the affine mixture of experts model~\citep{xu_alternative_1995} and inverse regression techniques such as those proposed in~\citep{hoadley_bayesian_1970,li_sliced_1991}. One limitation of softmax-gated MoE models lies in the complexity of their parameter estimation via the expectation-maximization (EM) algorithm~\citep{dempster_maximum_1977}, where the M-step involves solving an internal optimization subroutine, typically through Newton-Raphson iterations. In contrast, the parameter updates for GGMoE models admit closed-form expressions, leading to more computationally efficient inference procedures. Hence, this work focuses on parameter estimation and model selection within the GGMoE framework, which is particularly effective for modeling complex, nonlinear input-output relationships arising from heterogeneous subpopulations.
	
	{\bf Related work on model selection in MoE models.} Universal approximation theorems establish that, given a sufficiently large number of components, both mixture models and MoEs possess the capacity to approximate broad families of unconditional and conditional probability density functions, respectively, with arbitrary accuracy. Despite this theoretical guarantee, the practical determination of an adequate number of components remains a challenging task in real-world applications. This underscores the importance of principled model selection strategies for MMs and MoEs, a topic that has garnered sustained interest in the statistics and machine learning communities for several decades. For a recent synthesis of developments in this area, see, e.g., \cite{celeux_model_2019,gormley_mixture_2019}.
	In practice, the selection of the number of components in MoEs typically relies on model selection criteria that balance fit and complexity. Commonly employed information criteria include the Akaike information criterion (AIC) \citep{akaike_new_1974,fruhwirth_schnatter_analysing_2018}, the Bayesian information criterion (BIC) \citep{schwarz_estimating_1978} and its adaptations MoE models~\citep{khalili_estimation_2024,forbes_mixture_2022,berrettini_identifying_2024,forbes_summary_2022}, the integrated classification likelihood (ICL) \citep{biernacki_assessing_2000,fruhwirth_schnatter_labor_2012}, and the extended BIC (eBIC) used in graphical model contexts \citep{foygel_extended_2010,nguyen_joint_2024}. For dependent data settings, the Sin and White information criterion (SWIC)~\citep{sin_information_1996,westerhout_asymptotic_2024} is also utilized. A common limitation of these methods, however, is their asymptotic nature; they lack finite-sample guarantees and may be unreliable in small-sample regimes, rendering their practical use somewhat heuristic.
	To address these concerns, recent work by \cite{nguyen_non_asymptotic_2022,nguyen_non_asymptotic_2023,montuelle_mixture_2014,nguyen_non_asymptotic_2021,nguyen_model_2022,nguyen_non_asymptotic_Lasso_2023} has developed non-asymptotic risk bounds in the form of weak oracle inequalities. These results justify model selection via penalized likelihood, provided appropriate lower bounds on the penalties are enforced, and are applicable in high-dimensional regression settings across various MoE formulations, including GGMoE.
	\iffalse
	Complementarily, \cite{nguyen_order_2022} proposed a sequential testing methodology based on the closed testing principle, which yields confidence statements about the model order and supports optimal component selection in finite samples.
	\fi
	Another widely used paradigm is Bayesian model selection via marginal likelihood evaluation, which provides a probabilistic framework for selecting the number of components~\citep{fruhwirth_schnatter_keeping_2019,zens_bayesian_2019}.
	Nonetheless, a key computational challenge persists: existing model selection methods generally require training and evaluating multiple candidate models across different component configurations. This exhaustive search becomes computationally infeasible in high-dimensional regimes or when MoE models are embedded within deep neural networks~\citep{lathuiliere_deep_2017,shazeer_outrageously_2017,pham_competesmoeeffective_2024,do_hyperrouter_2023}. To address this issue, \cite{nguyen_bayesian_2024_JNPS} introduced a computationally efficient Bayesian nonparametric formulation of the GGMoE model that circumvents the need to pre-specify the number of components. This method employs the merge-truncate-merge (MTM) post-processing strategy proposed by \cite{guha_posterior_2021} to achieve posterior consistency. Nevertheless, its practical implementation is constrained by the presence of a sensitive tuning parameter, which undermines its robustness and limits its direct applicability to real-world datasets. 
	
	{\bf Main contributions.} To address the challenge of model selection in GGMoE models, we revisit the concept of dendrograms of mixing measures introduced by \cite{do_dendrogram_2024} and propose a novel extension: the dendrogram selection criterion (DSC) algorithm tailored for GGMoE models. This method enables consistent estimation of the true number of mixture components while significantly mitigating the dependence on sensitive tuning parameters, a common limitation in existing approaches, including the GGMoE framework combined with the MTM post-processing procedure \cite{nguyen_bayesian_2024_JNPS}. In addition to its sensitivity to tuning parameters, the MTM approach does not alleviate the fundamental limitation of slow convergence rates in parameter estimation, even after merging. In contrast, the proposed DSC algorithm is theoretically guaranteed to improve the convergence rate of the overfitted mixing measure. Moreover, it constructs a hierarchical tree of components without requiring repeated model fitting across different tuning parameter configurations. To the best of our knowledge, the DSC algorithm represents the first statistically consistent and computationally efficient method for model selection via merging of experts in GGMoE models.

	{\bf Organization.} The remainder of this paper is structured as follows. In \cref{sec_background}, we begin by reviewing the formal definition of GGMoE models and the corresponding maximum likelihood estimation procedure in the overfitted regime. \Cref{sec_dendrogram_GGMoE} introduces our main methodological contribution, which involves a model selection framework based on the dendrogram of mixing measures, along with an analysis of its asymptotic properties. These theoretical findings are subsequently validated through simulation studies presented in \cref{sec_simulation_study}. Finally, we offer concluding remarks, limitations and future work in \cref{sec_conclusion}, while detailed proofs of all theoretical results are provided in the supplementary material.
	
	\paragraph{Notation.} Throughout the paper, 
	$\{1, 2, \ldots, N\}$ is abbreviated as $[N]$ for any natural number $N \in \Ns$.
	Given any two sequences of positive real numbers $\{a_N\}_{N = 1}^{\infty}$ and $\{b_N\}_{N = 1}^{\infty}$, we write 
	% $a_N = \cO(b_N)$ or 
	$a_N \lesssim b_N$  to indicate that there exists a constant $C > 0$ such that $a_N \leq C b_N$ for all $ N \in \Ns$. We write $a_N \asymp b_N$ if $a_N \lesssim b_N$ and $b_N \lesssim a_N$. We write $a_N \ll b_N$ or $a_N = o(b_N)$ if $a_N/b_N\to 0$ as $N \to \infty.$ Next, for any vector $\vv\in\R^D, D \in \Ns$, let $|\vv|:=\evv_1+\evv_2+\cdots+\evv_D$, whereas $\|\vv\|_{p}$ stands for its $p$-norm with a note that $\|\vv\|$ implicitly indicates the $2$-norm unless stating otherwise.
	\iffalse
	\forestgreen{By abuse of notation, we also denote by $\|\mA\|$ the Frobenius norm of any matrix $\mA\in\R^{D\times D}$.}
	\fi
	Additionally, the notation $|\sS|$ represents the cardinality of any set $\sS$.
	%Furthermore, we write $a_N\asymp b_N$ if $a_N \lesssim b_N \lesssim a_N$. 
	Finally, given two probability density functions $p,q$ with respect to the Lebesgue measure $\vmu$, we define $\TV(p,q) := \frac 1 2 \int |p-q| d\vmu$ as their Total Variation distance, while $\hels(p,q) := \frac 1 2 \int (\sqrt p - \sqrt q)^2 d\vmu$ denotes the squared Hellinger distance between them.
	Let $\vTheta$ represent the parameter space. We denote by $\cE_K(\vTheta)$ the space of discrete distributions on $\vTheta$ with exactly $K$ atoms, and by $\cO_K(\vTheta) = \cup_{K' \le K} \cE_{K'}(\vTheta)$ the space of discrete distributions on $\vTheta$ with at most $K$ atoms. For a mixing measure $G_K = \sum_{k=1}^{K} \pi_k \delta_{\vtheta_k}$, we adopt a slight abuse of notation by referring to each term $ \pi_k \delta_{\vtheta_k} $ as an ``atom'', which encompasses both the weight $ \pi_k $ and the parameter $ \vtheta_k $. We drop $\vTheta$ in $\cE_K$ and $\cO_K$ when there is no confusion.
	
	\section{Background}\label{sec_background}
	
	{\bf GGMoE models.} Gaussian-gated Gaussian mixture of experts (GGMoE) models are designed to capture nonlinear and heterogeneous relationships between the univariate response variables $\rvy \in \cY \subset \doout$ and the covariates $\rvx \in \cX \subset \R^D$, where $D \in \Ns$.
	%%typically satisfies $L \gg D$ in high-dimensional settings.
	In the affine GGMoE model, the univariate response $\rvy$ is modeled as a weighted sum of $K_0$ local affine transformations, that is, $\rvy = \sum_{k=1}^{K_0} \1_{\rz = k} \left( \mA^{0}_k \rvx + \vb^0_k + \rve^0_k \right),$
	\iffalse
	\begin{align}\label{eq_local_affine}
		\rvy = \sum_{k=1}^{K_0} \1_{\rz = k} \left( \mA^{0}_k \rvx + \vb^0_k + \rve^0_k \right),
	\end{align}
	\fi
	where $\1_{\rz = k}$ is the indicator function of the latent variable $\rz$, which encodes the cluster membership of the observation such that $\rz = k$ if $\rvy$ is generated from the $k$-th cluster, for $k \in [K_0]$.
	Each component is parameterized by a cluster-specific affine transformation defined by $\mA^0_k \in \R^{1 \times D}$ and $\vb^0_k \in \doout$. The noise term $\rve^0_k \in \doout$ accounts for both model approximation error (arising from the local affine form) and observation noise.
	Let $\cF_D : = \left\{\cN(\cdot;\vc,\mGamma):\vc\in\R^D,\mGamma\in\cS^+_{D}\right\}$ be the family of $D$-dimensional Gaussian density functions with mean $\vc$ and positive-definite covariance matrix $\mGamma$, where $\cS_D^{+}$ indicates the set of all symmetric positive-definite matrices on $\R^{D \times D}$.
	Assuming that $e^0_k$ is a zero-mean Gaussian random variable with covariance $\mSigma^0_k \in \cS^+_1$, the conditional distribution of $\rvy$ given $\rvx = \vx$ and $\rz = k$ is given by
	\begin{align*}
		p\left(\vy \mid \vx, \rz = k\right) = f_{\cD}\left(\vy \mid \mA^{0}_k \vx + \vb^0_k, \mSigma^0_k\right), \quad f_{\cD} \in \cF_1.
	\end{align*}
	To promote the locality of the affine transformations, the covariates $\rvx$ is a mixture of $K_0$ components:
	\begin{align}\label{eq_marginal_forward}
		p\left(\vx \mid \rz = k\right) = f_{\cL}\left(\vx \mid \vc^0_k, \mGamma^0_k\right), \quad
		p(\rz = k) = \pi^0_k, \quad f_{\cL} \in \cF_D,
	\end{align}
	where $\pi^0_k > 0$ are the mixing proportions, constrained such that $\sum_{k=1}^{K_0} \pi^0_k = 1$. We refer to $f_{\cD}$ as the conditional data density and $f_{\cL}$ as the local (or gating) density.
	By applying the law of total probability, the joint density of the GGMoE model of order $K_0$ is expressed as
	\begin{align}\label{eq_true_joint_GLLiM}
		p_{G_0}(\vx, \vy) = \sum_{k=1}^{K_0} \pi^0_k f_{\cL}(\vx \mid \vc^0_k, \mGamma^0_k) \cdot f_{\cD}(\vy \mid \mA^{0}_k \vx + \vb^0_k, \mSigma^0_k).
	\end{align}
	The mixing measure $G_0$ underlying the model is defined as
	$G_0 := \sum_{k=1}^{K_0} \pi^0_k \delta_{(\vc^0_k, \mGamma^0_k, \mA^0_k, \vb^0_k, \mSigma^0_k)} = \sum_{k=1}^{K_0} \pi^0_k \delta_{\vtheta_k^0},$ where $\delta$ denotes the Dirac measure and each component $\vtheta_k^0 := (\vc^0_k, \mGamma^0_k, \mA^0_k, \vb^0_k, \mSigma^0_k)$ lies in the parameter space $\vTheta \subset \R^D \times \cS^+_D \times \R^{ D} \times \R \times \cS^+_1$.
	We assume that the observed data $\{(\vx_n, \vy_n)\}_{n \in [N]}$ consists of i.i.d.\ samples drawn from the GGMoE model in~\cref{eq_true_joint_GLLiM}. To facilitate our theoretical guarantee, we assume that $\vTheta$ is compact and $\cX$ is bounded.
	
	\iffalse
	To facilitate our theoretical guarantee, we assume that $\vTheta$ is compact and $\cX$ is bounded. We will only present the results when $D>1,~L=1$ for simplicity, while those for the setting $D>1,~L>1$ can be argued in a similar fashion but with more complex notations.
	\fi
	
	\textbf{Maximum likelihood estimation.} In an overfitted regime, where the number of components $K$ exceeds the true number $K_0$, which is typically unknown in practice, we develop a general theoretical framework to investigate the statistical properties and model selection behavior of the maximum likelihood estimator (MLE) for parameter inference under the GGMoE model defined as: $\widehat{G}_N \in \argmax_{G \in \cO_K} \sum_{n=1}^{N} \log\left(p_G(\vx_n, \vy_n)\right).$
	\iffalse
	\begin{align}
		\label{eq_MLE_formulation}
		\widehat{G}_N \in \argmax_{G \in \cO_K} \sum_{n=1}^{N} \log\left(p_G(\vx_n, \vy_n)\right).
	\end{align}
	%%
	\begin{align}
		\label{eq_MLE_formulation}
		\widehat{G}_N \in \argmax_{G \in \cO_K(\vTheta)} \sum_{n=1}^{N} \log\left(p_G(\vx_n, \vy_n)\right),
	\end{align}
	where $\cO_K(\vTheta)$ denotes the collection of all mixing measures with at most $K$ components: $    \cO_K(\vTheta) := \left\{G = \sum_{k=1}^{K'} \pi_k \delta_{\vtheta_k} : 1 \leq K' \leq K,\ \sum_{k=1}^{K'} \pi_k = 1,\ \vtheta_k \in \vTheta\right\}.$
	\fi
	
	\section{Dendrogram Selection Criterion for GGMoEs}\label{sec_dendrogram_GGMoE}
	As discussed in the previous sections, the MoE models possess extremely slow parameter estimation rates when being overfitted. Indeed, when the number of experts is misspecified by a large number, many estimated experts might try to approximate the same true expert and cancel each other's estimation rates out. This phenomenon motivates us to merge similar experts to improve their convergence rate. We start by presenting the experts' merging rule and discussing general concept of convergence rate in~\cref{subsec:merging-dendrogram} and~\cref{subsec:characterizing-convergence}. We then prove the fast estimation rate of parameters on the dendrogram of MLE in~\cref{subsec:applied-to-MLE} and propose a consistent model selection method in~\cref{subsec_model_selection_DSC}. Notably, our procedure only requires an over-fitted MLE of the GGMoE model rather than training models with various numbers of experts that can be expensive in practice.
	
	\subsection{Merging of Experts and Dendrogram of Mixing Measures}\label{subsec:merging-dendrogram}
	For a mixing measure $G = \sum_{i=1}^{K} \pi_{i} \delta_{(\vc_i, \mGamma_i, \mA_i, \vb_i, \mSigma_i)}$ having $K$ distinct atoms, 
	we define a dissimilarity between $i$- and $j$-th atoms (corresponding to the parameters of $i$- and $j$-th experts) as follows:
	\iffalse
	$$ \divclus(\pi_{l} \delta_{(\vc_{l}, \mGamma_{l}, \mA_{l}, \vb_{l}, \mSigma_{l})},\pi_{k} \delta_{(\vc_{k}, \mGamma_{k}, \mA_{k}, \vb_{j}, \mSigma_{j})}) = S_{l,k}(2,1,1,2,1)/(\pi_l^{-1} + \pi_{k}^{-1}),$$
	\fi
	\begin{align*}
		\divclus(\pi_{i} \delta_{(\vc_i, \mGamma_i, \mA_i, \vb_i, \mSigma_i)},\pi_{j} \delta_{(\vc_j, \mGamma_j, \mA_{j}, \vb_{j}, \mSigma_{j})}) &= \dfrac{\pi_i \pi_j}{\pi_i + \pi_j}\big( \lVert\vc_{i} - \vc_{j}\rVert^2 +\lVert \mGamma_{i} - \mGamma_{j}\rVert\\
		&\quad \hspace{2cm} +  \lVert \mA_{i} - \mA_{j}\rVert + \lVert \vb_{i} - \vb_{j}\rVert^2 + \lVert \mSigma_{i} - \mSigma_{j} \rVert\big ).
	\end{align*}
	%%
	\iffalse
	where we define the the mappings: $S_{l,k}:\N^5 \rightarrow \R$ such that $S_{l,k}(s_1,s_2,s_3,s_4,s_5):=\lVert\vc_{l} - \vc_{k}\rVert^{s_1} +\lVert \mGamma_{l} - \mGamma_{k}\rVert^{s_2}+  \lVert \mA_{l} - \mA_{k}\rVert^{s_3} + \lVert \vb_{l} - \vb_{k}\rVert^{s_4} + \lVert \mSigma_{l} - \mSigma_{k} \rVert^{s_5}$ for any $l,k \in \N$.
	\fi
	%%
	Intuitively, this dissimilarity is small when the sets of parameters $(\vc_i, \mGamma_i, \mA_i, \vb_i, \mSigma_i)$ and $(\vc_j, \mGamma_j, \mA_j, \vb_j, \mSigma_j)$ are close to each other (i.e., the $i$- and $j$-th experts are similar) or when the expert's weight $\pi_i$ or $\pi_j$ is small. Starting from $G$, we choose two atoms minimizing this dissimilarity to merge into a single atom and keep the other atoms fixed. This procedure results in a new mixing measure with $(K-1)$ atoms. Details of the merging rule are presented in~\cref{algorithm_mergeAtom_GGMoE}. 
	
	Having introduced the algorithm for selecting and merging two atoms into one, we now describe the construction of the dendrogram associated with a given mixing measure $G$ via the repeated application of this merging procedure. Starting from $G = G^{(K)}$ having $K$ atoms, we iteratively apply \cref{algorithm_mergeAtom_GGMoE} in a backward fashion for each step $\kappa$ from $K$ down to $2$, thereby generating a sequence of mixing measures $\{G^{(\kappa)}\}_{\kappa=2}^{K}$, where $G^{(\kappa)}$ contains $\kappa$ atoms. The dendrogram of a mixing measure $G$ is defined as a tuple $\cT(G) = (\gV, \gE, \gH)$, where $\gV$ is a set comprising $K$ hierarchical levels, with the $\kappa$-th level containing $\kappa$ atoms of $G^{(\kappa)}$ as vertices. The set $\gE$ encodes the edges that specify which vertices are merged across levels. The dissimilarity vector $\gH = (\height^{(K)}, \height^{(K-1)}, \dots, \height^{(2)})$ contains the minimal pairwise dissimilarity $\divclus$ between atoms at each level, where $\height^{(\kappa)}$ denotes the minimum $\divclus$ among all atom pairs in $G^{(\kappa)}$. When visualized as a hierarchical tree, the dendrogram $\cT(G)$ encodes the hierarchical structure of the merging process, where the quantity $\height^{(\kappa)}$ denotes the height between the $\kappa$-th and $(\kappa-1)$-th levels. The complete procedure for constructing the dendrogram of $G$ is formalized in the dendrogram selection criterion (DSC) algorithm, detailed in \cref{algorithm_DSC_GGMoE} and illustrated in \cref{fig_dendrogram_illustration}.
	
	%%
	\iffalse
	\begin{definition}[Dendrogram of mixing measures]\label{def_dendrogram_mixing_measure}
		The dendrogram of a mixing measure $G$ is defined as a tuple $\cT(G) = (\gV, \gE, \gH)$, where $\gV$ is a set comprising $K$ hierarchical levels, with the $\kappa$-th level containing $\kappa$ atoms of $G^{(\kappa)}$ as vertices. The set $\gE$ encodes the edges that specify which vertices are merged across levels. The dissimilarity vector $\gH = (\height^{(K)}, \height^{(K-1)}, \dots, \height^{(2)})$ contains the minimal pairwise dissimilarity $\divclus$ between atoms at each level, where $\height^{(\kappa)}$ denotes the minimum $\divclus$ among all atom pairs in $G^{(\kappa)}$.
	\end{definition} 
	\fi
	%%

	%%
	\begin{algorithm}[!ht]
		\caption{Merging of atoms for GGMoE models}
		\label{algorithm_mergeAtom_GGMoE}
		\begin{algorithmic}[1]
			\Require A mixing measure $G^{(K)} = \sum_{k=1}^{K} \pi_k \delta_{(\vc_k, \mGamma_k, \mA_k, \vb_k, \mSigma_k)}$.
			\State Select atoms to merge:
			$\ell_1, \ell_2 = \argmin_{k \neq k' \in [K]} \divclus\left(\pi_{k} \delta_{(\vc_k, \mGamma_k, \mA_k, \vb_k, \mSigma_k)}, \pi_{k'} \delta_{(\vc_{k'}, \mGamma_{k'}, \mA_{k'}, \vb_{k'}, \mSigma_{k'})}\right).$
			\State Merge the selected atoms using weighted aggregation:
			\begin{align*}
				\pi_{*} &= \pi_{\ell_1} + \pi_{\ell_2}, \quad \vc_{*} = \dfrac{\pi_{\ell_1}}{\pi_{*}}\,\vc_{\ell_1} + \dfrac{\pi_{\ell_2}}{\pi_{*}}\,\vc_{\ell_2},\quad \vb_{*} = \dfrac{\pi_{\ell_1}}{\pi_{*}}\, \vb_{\ell_1} + \dfrac{\pi_{\ell_2}}{\pi_{*}}\, \vb_{\ell_2},\\
				\mGamma_* &= \dfrac{\pi_{\ell_1}}{\pi_{*}} \left(\mGamma_{\ell_1} + (\vc_{\ell_1} - \vc_*)(\vc_{\ell_1} - \vc_*)^{\top} \right) + \dfrac{\pi_{\ell_2}}{\pi_{*}} \left(\mGamma_{\ell_2} + (\vc_{\ell_2} - \vc_*)(\vc_{\ell_2} - \vc_*)^{\top} \right),\\
				\mA_* &= \dfrac{\pi_{\ell_1}}{\pi_{*}} \left(\mA_{\ell_1} + (\vb_{\ell_1} - \vb_*)(\vc_{\ell_1} - \vc_*)^{\top} \right) + \dfrac{\pi_{\ell_2}}{\pi_{*}} \left(\mA_{\ell_2} + (\vb_{\ell_2} - \vb_*)(\vc_{\ell_2} - \vc_*)^{\top} \right),\\
				\mSigma_* &= \dfrac{\pi_{\ell_1}}{\pi_{*}} \left(\mSigma_{\ell_1} + (\vb_{\ell_1} - \vb_*)^2 \right) + \dfrac{\pi_{\ell_2}}{\pi_{*}} \left(\mSigma_{\ell_2} + (\vb_{\ell_2} - \vb_*)^2 \right).
			\end{align*}
			\State \Return Merged mixing measure:
			$ G^{(K-1)} = \pi_{*} \delta_{(\vc_*, \mGamma_*, \mA_*, \vb_*, \mSigma_*)} + \sum_{k \neq \ell_1, \ell_2} \pi_{k} \delta_{(\vc_k, \mGamma_k, \mA_k, \vb_k, \mSigma_k)}.$
		\end{algorithmic}
	\end{algorithm}
	\begin{algorithm}[!ht]
		\caption{Dendrogram of GGMoE's mixing measures}
		\label{algorithm_DSC_GGMoE}
		\begin{algorithmic}[1]
			\Require A mixing measure $G^{(K)} = \sum_{k=1}^{K} \pi_k \delta_{(\vc_k, \mGamma_k, \mA_k, \vb_k, \mSigma_k)}$. 
			\State Initialize $\cT(G) = (\gV, \gE, \gH)$, where the $K$-th level of $\gV$ contains all atoms of $G^{(K)}$; set $\gE = \varnothing$ and define $\gH = (\height^{(K)}, \height^{(K-1)}, \dots, \height^{(2)})$ as an array of length $K-1$.
			\For{$\kappa = K$ \textbf{down to} $2$} 
			\State Apply \cref{algorithm_mergeAtom_GGMoE} to $G^{(\kappa)} \in \cE_{\kappa}$ to obtain $G^{(\kappa-1)} \in \cE_{\kappa-1}$,
			
			\State Add all atoms of $G^{(\kappa-1)}$ to the $(\kappa-1)$-th level of $\gV$,
			
			\State Add two edges to $\gE$ connecting the atoms in $G^{(\kappa)}$ that were merged into one atom in $G^{(\kappa-1)}$,
			
			\State Record $\height^{(\kappa)} = \divclus(\pi \delta_{\vtheta}, \omega \delta_{\veta})$, where $\pi \delta_{\vtheta}$ and $\omega \delta_{\veta}$ are the merged atoms.
			\EndFor
			\State \Return Dendrogram $\cT(G) = (\gV, \gE, \gH)$ and sequence of merged mixing measures $\{G^{(k)}\}_{k=1}^{K}$.
		\end{algorithmic}
	\end{algorithm}

	\subsection{Characterizing Parameter Estimation Rate of GGMoE}\label{subsec:characterizing-convergence}
	We now define the loss function for characterizing parameter estimation rate in GGMoE and provide important inequalities that motivates the merging rule and are useful for proving estimation rate. 
	
	\noindent
	\textbf{Loss function and Voronoi cells.} It is known that the parameter estimation rate of overfitted parameters in Gaussian mixture models and MoE depends on the amount of overfitting (how much larger $K$ is compared to $K_0$) and the parameters' types (covariance matrices often have better rates than mean)~\cite{ho_singularity_2019, ho_convergence_2022}. Indeed, true parameters that are well captured by a single fitted component tend to exhibit faster estimation rates compared to those approximated by multiple components~\cite{manole_refined_2022}. To accurately characterize the convergence behavior of parameters, we define the following loss function. For an estimated mixing measure $G = \sum_{\ell =1}^{K} \pi_{\ell} \delta_{(\vc_{\ell}, \mGamma_{\ell}, \mA_{\ell}, \vb_{\ell}, \mSigma_{\ell})}$ of true mixing measure $G_0 = \sum_{k=1}^{K_0} \pi_k^0 \delta_{(\ck, \gk, \ak, \bk, \sk)}$, and for any $\ell \in [K], k\in [K_0]$, let the estimation loss of $(\vc_{\ell}, \mGamma_{\ell}, \mA_{\ell}, \vb_{\ell}, \mSigma_{\ell})$ to $(\ck, \gk, \ak, \bk, \sk)$ be denoted by
	$$S_{\ell,k}(s_1,s_2,s_3,s_4,s_5):=\lVert\vc_{\ell} - \vc_{k}^0\rVert^{s_1} +\lVert \mGamma_{\ell} - \mGamma_{k}^0\rVert^{s_2}+  \lVert \mA_{\ell} - \mA_{k}^0\rVert^{s_3} + \lVert \vb_{\ell} - \vb_{k}^0\rVert^{s_4} + \lVert \mSigma_{\ell} - \mSigma_{k}^0 \rVert^{s_5}.$$
	This loss function captures different estimation rate of parameters via the exponent $(s_i)_{i=1}^{5}$.
	We further define a collection of $K_0$ index sets, referred to as Voronoi cells, which quantify the number of fitted components associated with each of the $K_0$ true components. For any $k \in [K_0]$, the Voronoi cell $\sV_k \equiv \sV_k(G)$ associated with the true component $\vtheta^0_k := (\ck, \gk, \ak, \bk, \sk)$ is defined as
	\begin{align}\label{eq_Voronoi_cell}
		\sV_k := \left\{ \ell \in [K] : S_{\ell,k}(2,1,1,2,1) \leq S_{\ell,k'}(2,1,1,2,1),\ \forall k' \neq k \right\}.
	\end{align}
	The cardinality of each Voronoi cell $\sV_k$ reflects the number of fitted components that approximate the corresponding true component $\vtheta^0_k$.
	
	\textbf{System of polynomial equations.} 
	For each integer $M \ge 2$, we define $\overline{r}_k := \overline{r}(M)$ as the smallest positive integer number $R$ for which the following system of polynomial equations:
	\begin{align}
		\label{eq_system_GGMoE}
		\sum_{m=1}^{M} \sum_{\valpha \in \mathcal{J}_{\ell_1,\ell_2}} \frac{p_m^2 \, q_{1m}^{\evalpha_1} \, q_{2m}^{\evalpha_2} \, q_{3m}^{\evalpha_3} \, q_{4m}^{\evalpha_4} \, q_{5m}^{\evalpha_5}}{\evalpha_1! \, \evalpha_2! \, \evalpha_3! \, \evalpha_4! \, \evalpha_5!} = 0,
	\end{align}
	admits no non-trivial solution for all integers $\ell_1, \ell_2 \geq 0$ such that $1 \leq \ell_1 + \ell_2 \leq R$. The index set $\mathcal{J}_{\ell_1,\ell_2}$ is defined as $\mathcal{J}_{\ell_1,\ell_2} := \left\{ \valpha = (\evalpha_l)_{l=1}^5 \in \Ns^5 : \evalpha_1 + 2\evalpha_2 + \evalpha_3 = \ell_1,\ \evalpha_3 + \evalpha_4 + 2\evalpha_5 = \ell_2 \right\}.$
	A solution is considered non-trivial if all $p_m$ are non-zero and at least one of the $q_{4m}$ is non-zero. The following \cref{lemma_r_bar_GGMoE} provides specific values of $\overline{r}(M)$ for selected cases of $M$, which is proved in appendix.
	\begin{lemma}
		\label{lemma_r_bar_GGMoE}
		We have $\overline{r}(2) = 4$, $\overline{r}(3) = 6$, and for $M \geq 4$, it holds that $\overline{r}(M) \geq 7$.
	\end{lemma}

	\noindent{\bf Voronoi loss.} Given $\overline{r}_k=\overline{r}(|\sV_k|)$, for any $G = \sum_{k=1}^{K} \pi_k \delta_{(\vc_k, \mGamma_k, \mA_k, \vb_k, \mSigma_k)} \in \cO_{K}$, let
	\begin{align}
		& \VD(G, G_0)  =\sum_{k=1}^{K_0} \left|\sum_{\ell\in V_{k}} \pi_\ell - \pi_k^0 \right|+ \sum_{\substack{k : |\sV_k| = 1 \\ \ell \in \sV_k}}\pi_\ell S_{\ell,k}(1,1,1,1,1)+ \sum_{\substack{k : |\sV_k| > 1 \\ \ell \in \sV_k}}\pi_\ell S_{\ell,k}(\overline{r}_k,\frac{\overline{r}_k}{2},\frac{\overline{r}_k}{2},\overline{r}_k,\frac{\overline{r}_k}{2})\nonumber\\
		& + \sum_{k : |\sV_k| > 1}\Big(\norm{\sum_{\ell\in \sV_k}\pi_\ell (\vc_\ell - \vc_k^0)} +   \norm{\sum_{\ell\in \sV_k}\pi_\ell (\vb_\ell - \vb_k^0)} +\norm{\sum_{\ell\in \sV_k}\pi_\ell \left( \mGamma_\ell - \mGamma_k^0 + (\vc_\ell - \vc_k^0) (\vc_\ell - \vc_k^0)^{\top}\right)}\nonumber\\
		& \quad + \norm{\sum_{\ell\in \sV_k}\pi_\ell (\mA_\ell - \mA_k^0 + (\vb_\ell - \vb_k^0)(\vc_\ell - \vc_k^0)^{\top})} + \norm{\sum_{\ell\in \sV_k}\pi_\ell \left(\mSigma_\ell - \mSigma_k^0 + (\vb_\ell - \vb_k^0)^2 \right)}\Big),\label{eq_voronoi_loss_definition}
	\end{align}
	where the matrices norm is Frobenius, and we take  \update{$\overline{r}(1) = 1$} as convention. This novel notion of Voronoi loss is stronger than that of~\cite{nguyen_demystifying_2023, nguyen_towards_2024} as it also captures the convergence of merged parameters (represented by the terms in the second and third lines).
	Showing parameter estimation rate in MoE often involves two steps: (i) Proving the \textit{density estimation rate} $p_{G} \to p_{G_0}$ and (ii) Translate it into \textit{parameter estimation rate} by providing lower bounds of distance between densities by distance between parameters, which we refer to as inverse bounds~\cite{nguyen_convergence_2013}. Proving the first step (density estimation) can be done by standard empirical process theory~\cite{geer2000empirical}. Hence, our contribution here is developing the inverse bound that characterizes the parameter estimation rate of both overfitted and merged parameters.
	The following inverse bound, which is proved in \cref{sec_proof_theorem_V_DG_GLLiM}, plays the central role in our theoretical analysis:
	\begin{theorem}\label{theorem_V_DG_GLLiM}
		For $G \in \mathcal{O}_{K}$, as $\TV(p_G, p_{G_0}) \rightarrow 0$, we have $\TV(p_G, p_{G_0}) \gtrsim \VD(G, G_0)$.
	\end{theorem}
	This inverse bound and the definition of Voronoi loss motivate the merging rule as presented in~\cref{subsec:merging-dendrogram}. Indeed, the density estimation rate of MLE $D_{\text{TV}}(p_{\widehat{G}_N}, p_{G_0})$ is often of order $N^{-1/2}$ (see, e.g.,~\cite{nguyen_towards_2024}). Combining this fact with the inverse bound implies that the true parameter $(\ck, \gk, \ak, \bk, \sk)$ having $|\sV_k| = 1$ will be estimated with rate $N^{-1/2}$, but the one having $|\sV_k| > 1$ will be estimated with rate $N^{-1/(2 \overline{r}_k)}$, which is significantly slower. The Voronoi loss further suggests that merging all atoms in the Voronoi cell $\sV_k$ will produce a new atom with a fast estimation rate (as the exponents of all parameters in the second and third lines of~\cref{eq_voronoi_loss_definition} are 1). The following result shows that sequentially merging atoms as in~\cref{algorithm_DSC_GGMoE} asymptotically does exactly that. Its proof appears in~\cref{subsec:proof_theorem_inequality_DK_DK0}.
	
	\begin{theorem}\label{theorem_inequality_DK_DK0}
		For $G^{(K)} \in \cE_{K}$, denote by $G^{(K-1)}, \ldots, G^{(1)}$ the sequence of latent mixing measures derived from the dendrogram constructed via \cref{algorithm_DSC_GGMoE}, where each $G^{(\kappa)}$ consists of $\kappa$ atoms for $\kappa \in [K]$. As $\VD(G^{(K)}, G_0) \rightarrow 0,$ given multiplicative constants only depend on $G_0, \Theta$, $K$, we have
		\begin{equation*}
			\VD(G^{(K)}, G_0) \gtrsim\VD(G^{(K-1)}, G_0) \gtrsim \dots \gtrsim \VD(G^{(K_0)}, G_0).
		\end{equation*}
		% where the multiplicative constants in these inequalities only depend on $G_0, \Theta$ and $K$.
	\end{theorem}
	
	%%%%%%%%%%%%%%%%%%%%%%%%%%%%%%%%%%%%%%%%%%%%%%%%%%%%%%%%%%%%%%
	\subsection{Asymptotic Properties of the Dendrogram of Mixing Measures in GGMoE}\label{subsec:applied-to-MLE}
	We now investigate the asymptotic properties of dendrograms constructed from overfitted mixing measures in GGMoE models.  
	In particular, we focus on the convergence rates of the mixing measures, as well as the behavior of the dendrogram height and the model likelihood across different levels of the dendrogram. All proofs in this section appear in~\cref{sec_proof_theorem_rate_DIC_GLLiM} and~\cref{sec_proof_theorem_likelihood_DIC_GLLiM}.
	
	{\bf Fast convergence of mixing measures and heights arising in the dendrogram of the MLE.} Recall that $\widehat{G}_N$ is the MLE mixing measure with $K > K_0$ atoms and $G_0$ is the true mixing measure. Let $(\widehat{G}_{N}^{\kappa})_{\kappa=1}^{K}$ and $(G_{0}^{\kappa})_{\kappa=1}^{K_0}$ be the mixing measures on the dendrogram of $\widehat{G}_{N}$ and $G_0$, respectively. Letting $\bar{r}(\widehat{G}_N) = \max_{k \in [K_0]}\bar{r}(|\sV_k|)$ and for all $\kappa \in [K]$, we denote the sequence of heights by $ \height_N^{(\kappa)}:=\min_{k \neq k' \in [K]} \divclus\left(\pi_{k} \delta_{(\vc_k, \mGamma_k, \mA_k, \vb_k, \mSigma_k)}, \pi_{k'} \delta_{(\vc_{k'}, \mGamma_{k'}, \mA_{k'}, \vb_{k'}, \mSigma_{k'})}\right)$.
	\begin{theorem}\label{theorem_rate_DIC_GLLiM}
		There exist universal constants $c_1$ and $c_2$, depend only on $G_0$, $\Theta \times \Omega,$ and $K$, such that for probability at least $1 - c_1N^{-c_2},$ for each  $\kappa \in [K_0+1,K]$ and $\kappa' \in [K_0]$, it holds that
		\begin{align*}
			\VD(\widehat{G}_N^{(\kappa)}, G_0) \lesssim (\log N/N)^{1/2},\quad 
			\height_N^{(\kappa)} \lesssim  (\log N/N)^{1/\bar{r}(\widehat{G}_N)},\quad | \height_N^{(\kappa')} - \height_0^{(\kappa')} | \lesssim  (\log N/N)^{1/2}.
		\end{align*}
	\end{theorem}
	%%
	\iffalse
	\begin{theorem}\label{theorem_rate_DIC_GLLiM}
		There exist universal constants $c_1$ and $c_2$ such that for probability at least $1 - c_1N^{-c_2},$ we have all the inequalities for $\kappa \in [K_0 + 1, K]$ and $\kappa' \in [K_0],$ of which the multiplicative constants depend on $G_0$, $\Theta \times \Omega,$ and $K$ such that:
		\begin{align*}
			\iffalse
			W_{\bar{r}(\widehat{G}_N)}(\widehat{G}_N^{(\kappa)}, G_0) &\lesssim \left( \dfrac{\log N}{N} \right)^{1/2\bar{r}(\widehat{G}_N)}, \quad \quad W_{1}(\widehat{G}_N^{(\kappa')}, G_0^{(\kappa')}) \lesssim \left(\log N/N \right)^{1/2},\\
			%%
			\fi
			\VD(\widehat{G}_N^{(\kappa)}, G_0) &\lesssim (\log N/N)^{1/2}, \quad\VD(\widehat{G}_N^{(\kappa')}, G_0^{(\kappa')}) \lesssim (\log N/N)^{1/2},\nonumber\\
			%%
			\height_N^{(\kappa)} &\lesssim  (\log N/N)^{1/\bar{r}(\widehat{G}_N)},\quad \left | \height_N^{(\kappa')} - \height_0^{(\kappa')} \right| \lesssim  (\log N/N)^{1/2}.
		\end{align*}
	\end{theorem}
	\fi
	%%
	Note that for any $\kappa' \leq K_0$, $\widehat{G}_N^{(\kappa')}$ and $G_0^{(\kappa')}$ have the same number of atoms ($\kappa$) so that the atoms of $\widehat{G}_N^{(\kappa')}$ tend to that of $G_0^{(\kappa')}$ with the fast rate $(\log N / N)^{1/2}$ (up to a permutation). This is remarkable since we construct these merged measures from slowly convergent overfitted MLE mixing measure. 
	
	\subsection{Model Selection via DSC for GGMoE Models}\label{subsec_model_selection_DSC}
	
	\iffalse
	\subsubsection{Height of the dendrogram}
	%%
	For all $\kappa \in [K]$, we denote the sequence of heights by $ \height_N^{(\kappa)}:=\min_{k \neq k' \in [K]} \divclus\left(\pi_{k} \delta_{(\vc_k, \mGamma_k, \mA_k, \vb_k, \mSigma_k)}, \pi_{k'} \delta_{(\vc_{k'}, \mGamma_{k'}, \mA_{k'}, \vb_{k'}, \mSigma_{k'})}\right)$
	\iffalse
	\begin{equation*}
		\height_N^{(\kappa)} = \min\dfrac{1}{\pi_i^{-1} + \pi_j^{-1}}(\lVert \vc_{i} - \vc_{j}\rVert^2 + \lVert \mGamma_{i} - \mGamma_{j}\rVert +\lVert \mA_{i} - \mA_{j}\rVert + \lVert \vb_{i} - \vb_{j}\rVert^2 + \lVert \mSigma_{i} - \mSigma_{j}\rVert ).
	\end{equation*}
	where the minimum is taken over all pairs of atoms $(\pi_{i} \delta_{(\mA_i, \vb_i, \mSigma_i)},\pi_{j} \delta_{(\mA_{j}, \vb_{j}, \mSigma_{j})})$ of $\widehat{G}_N^{(\kappa)}, \kappa \in [K]$.
	\fi
	%%
	\begin{theorem}\label{theorem_height_DIC_GLLiM}
		With the same condition and probability as in \cref{theorem_rate_DIC_GLLiM}, we have 
		\begin{equation*}
			\height_N^{(\kappa)} \lesssim  (\log N/N)^{1/\bar{r}(\widehat{G}_N)}, \quad \text{and} \quad \left | \height_N^{(\kappa')} - \height_0^{(\kappa')} \right| \lesssim  (\log N/N)^{1/2}.
		\end{equation*}
	\end{theorem}
	\fi
	
	%%%%%%%%%%%%%%%%%%%%%%%%%%%%%%%%%%%%%%%%%%%%%%%%%%%%%%%%%%%%%%%%%%%%%%%%%%%%%%%%%%%%%%%%%%%%%%%%%%%%%%%%%%%%
	% \subsubsection{Likelihood of mixing measures in the dendrogram}
	
	{\bf Likelihood of mixing measures in the dendrogram.} For a given density function $p_G$, we define the population average log-likelihood and the empirical average log-likelihood as
	$
	\cL(p_G) := \E_{(\vx, \vy) \sim p_{G_0}}[\log p_G(\vx, \vy)] \quad \text{and} \quad \bar{l}_N(p_G) := (1/N) \sum_{n=1}^N \log p_G(\vx_n, \vy_n), \text{ respectively.}
	$
	
	\textbf{Condition K.} There exist positive constants $c_\alpha$ and $c_\beta$ such that for all sufficiently small $\epsilon $ and $\vtheta_0, \vtheta \in \vTheta, \lVert \vtheta - \vtheta_0 \rVert \leq \epsilon,$ we have \update{ $\log f(\vx, \vy|\vtheta) \geq (1 + c_\beta \epsilon) \log f(\vx, \vy|\vtheta_0) - c_\alpha\epsilon$}.
	\cref{theorem_likelihood_DIC_GLLiM} is proved in \cref{sec_proof_theorem_likelihood_DIC_GLLiM}.
	\begin{theorem}\label{theorem_likelihood_DIC_GLLiM}
		Suppose the conditions in \cref{theorem_rate_DIC_GLLiM} and  Condition K hold. Then, for any $\kappa \in [K_0 + 1, K]$, the following inequality holds:
		$$
		\bar{l}_N(\widehat{G}_N^{(\kappa)}) - \cL(p_{G_0}) \lesssim (\log N/N)^{1/{2\bar{r}(\widehat{G}_N)}}.
		$$
		Moreover, for any $\kappa' \in [K_0]$, the average log-likelihood $\bar{l}_N(\widehat{G}_N^{(\kappa')})$ converges in $\mathbb{P}_{G_0}$-probability to the population log-likelihood $\cL(p_{G_0^{(\kappa')}})$ as $N \rightarrow \infty$.
	\end{theorem}
	
	% \subsection{Where to cut the dendrogram}
	
	{\bf Model selection via DSC.} For each $\kappa \in [K]$, let $\height_N^{(\kappa)}$ denote the height at the $\kappa$-th level of the dendrogram, and let $\bar{l}_N^{(\kappa)}$ represent the average log-likelihood corresponding to the model with mixing measure $G_N^{(\kappa)}$. We define the DSC as
	$
	\dsc_N^{(\kappa)} := -(\height_N^{(\kappa)} + \omega_N \bar{l}_N^{(\kappa)}),
	$
	where the weight $\omega_N$ satisfies $1 \ll \omega_N \ll (N/\log N)^{1/{2\bar{r}(\widehat{G}_N)}}$. Although the value of $\bar{r}(\widehat{G}_N)$ is generally unknown, a practical choice is $\omega_N = \log N$. The estimated number of components is then given by
	$
	\widehat{K}_N := \argmin_{\kappa \in [2, K]} \dsc_N^{(\kappa)}.
	$
	\begin{theorem}{\label{DIC}}
		Assume that $K_0 \geq 2$, then as $N \rightarrow \infty$, $\widehat{K}_N \rightarrow K_0$ in $\mathbb{P}_{G_0}$-probability.
	\end{theorem}
	While AIC, BIC, and ICL evaluate models solely based on their likelihood and the number of components, DIC incorporates additional information from the mixing measures. Specifically, the criterion $\dsc_N^{(\kappa)}$ imposes a stronger penalty when $\height_N^{(\kappa)}$ is small, which typically occurs either due to the presence of closely located atoms or atoms with small weights in the mixing measure. Hence, DSC demonstrates greater robustness to model misspecification than AIC, BIC, and ICL, as supported by empirical evidence in \cref{sec_simulation_study}.

	\section{Simulation Studies}\label{sec_simulation_study}
	We begin by demonstrating that, despite being derived from an overfitted mixing measure that converges slowly, the merged mixing measure obtained via the dendrogram can attain a fast convergence rate to the true underlying mixing measure. Then, we investigate the model selection properties of the proposed DSC-based approach and compare its performance against standard criteria such as AIC, BIC, and ICL. Unlike these traditional methods, which typically yield a single estimate of the number of subpopulations, the dendrogram of mixing measures provides a hierarchical representation of the merging process and the structure of atoms. This hierarchical insight facilitates a more nuanced understanding of population heterogeneity and supports more informative data exploration. All simulation studies were conducted using Python 3.12.2 on a standard Unix-based system.
	
	%%
	
	\iffalse
	A representative realization of the data generated from GGMoE Model is shown in \cref{fig_draw_data_GGMoE1}. The true clustering structure, along with the corresponding global and local means of a typical realization from GGMoE Model, is presented in \cref{fig_true_cluster_mean_GGMoE1}.
	\fi
	
	{\bf Numerical details.} The true mixing measure $G_0 = \sum_{k=1}^{3} \pi^0_k \delta_{(\vc^0_k, \mGamma^0_k, \mA^0_k, \vb^0_k, \mSigma^0_k)}$ is specified as:
	\begin{align*} 
		G_0 = 0.3\delta_{(-0.1,\hspace{.05cm} 0.04,\hspace{.05cm} 0.40,\hspace{.05cm} 0.34,\hspace{.05cm} 0.01)} 
		+ 0.4\delta_{(0.1,\hspace{.05cm} 0.02,\hspace{.05cm} -0.71,\hspace{.05cm} -0.33,\hspace{.05cm} 0.03)} 
		+ 0.3\delta_{(0.5,\hspace{.05cm} 0.01,\hspace{.05cm} 0,\hspace{.05cm} 0.2,\hspace{.05cm} 0.02)}.
	\end{align*}
	For each experimental setting, we vary the logarithm of the sample size, $\log_{10}(N)$, from $\log_{10}(N_{\min})$ to $\log_{10}(N_{\max})$, corresponding to $N_{\nums}$ different values of $N$ within the interval $[N_{\min}, N_{\max}]$. For each sample size, we generate $N_{\rep}$ independent synthetic datasets from the ground-truth distribution corresponding to the true mixing measure $G_0$. We estimate both the exact-fitted MLE, denoted by $\widehat{G}^{e}_N \in \cE_3$, and the overfitted MLE, denoted by $\widehat{G}^{o}_N \in \cO_5$, $K=5$, using an adaptation of the EM algorithm as described in~\cite[Section~5]{deleforge_high_dimensional_2015}. The convergence criterion is set to $\epsilon = 10^{-5}$, with a maximum of 2000 EM iterations.
	To emphasize the theoretical properties of the estimator $\widehat{G}_N$, the EM algorithm is initialized in a favorable manner. Specifically, for each replication and each $(K, K_0)$ pair, we randomly partition the set $[K]$ into $K_0$ disjoint index subsets $\sS_1, \ldots, \sS_{K_0}$, each containing at least one element. For each $k \in \sS_t$, the initial parameters $\vc^0_k$ (respectively, $\mGamma^0_k$, $\mA^0_k$, $\vb^0_k$, $\mSigma^0_k$) are sampled from distinct Gaussian distributions centered at $\vc^0_t$ (respectively, $\mGamma^0_t$, $\mA^0_t$, $\vb^0_t$, $\mSigma^0_t$), with vanishing covariance to promote identifiability. After estimating the overfitted model $\widehat{G}^{o}_N$, we apply the merging procedure in~\cref{algorithm_DSC_GGMoE} to obtain the merged estimator $\widehat{G}^{m}_N \in \cE_3$.
	\iffalse
	\subsection{Fast parameter estimation via the dendrogram}
	{\color{red} How to emphasize that overfitted parameters have worse rates?}
	\fi
	
	{\bf Fast parameter estimation via the dendrogram.} In this experimental setup, we fix $K = 5$, $N_{\nums} =N_{\min} = 10^2$, $N_{\max} = 10^5$, and $N_{\rep} = 20$. To quantify how well $\widehat{G}^{e}_N$, $\widehat{G}^{o}_N$, and $\widehat{G}^{m}_N$ approximate the ground-truth mixing measure $G_0$, we evaluate the Voronoi loss distance as defined in~\cref{eq_voronoi_loss_definition}. The average performance, along with error bars corresponding to twice the empirical standard deviation of the log error, is depicted in~\cref{fig_convergence_rate_DV_1_K5_n0_0_nmin100_nmax100000_rep20_nnum100}, \cref{fig_convergence_rate_atoms_1_K5_n0_0_nmin100_nmax100000_rep20_nnum100} (best viewed in color). Our results in \cref{fig_convergence_rate_DV_1_K5_n0_0_nmin100_nmax100000_rep20_nnum100} indicate that the average discrepancy $\VD$ between each estimator $\widehat{G}^{e}_N$, $\widehat{G}^{o}_N$, and $\widehat{G}^{m}_N$ and the true mixing measure $G_0$ decreases at a rate proportional to $N^{-1/2}$, which is aligned with the theoretical guarantees in \cref{theorem_rate_DIC_GLLiM}. Particularly noteworthy is that $\widehat{G}^{m}_N$, which corresponds to the exact-fit level in the dendrogram hierarchy, attains the optimal root-$N$ rate of convergence to $G_0$. Importantly, this is achieved without retraining the model, even though the original parameter estimator from $\widehat{G}^{o}_N$ was overfitted and converges more slowly as shown in \cref{fig_convergence_rate_atoms_1_K5_n0_0_nmin100_nmax100000_rep20_nnum100} and \cref{fig_dendrogram_illustration}.
	
	% \subsubsection{Numerical setup}
	% In Model I,
	
	\iffalse
	Model II adopts the same setup as Model I, except that we set $\vc^0_1 = 0$ and $\vb^0_1 = 0.3$.
	
	To demonstrate that the empirical convergence rates of parameter estimation under the Type I (Model III) and Type II (Model IV) settings are preserved in higher-dimensional settings, we conduct simulations with $D = 2$, $L = 1$, and $K_0 = 3$. 
	In Model III, we specify $G_0$ as:
	\begin{align*} 
		G_0 &=
		0.3\delta_{(-0.1\cdot\mathbf{1}_d,\hspace{.01cm}  0.04\cdot\mI_D,\hspace{.01cm} 0.4\cdot\mathbf{1}_d,\hspace{.01cm} 0.34,\hspace{.01cm} 0.01)} + 0.4\delta_{(0.1\cdot\mathbf{1}_d,\hspace{.01cm} 0.02\cdot\mI_D,\hspace{.01cm} -0.71\cdot\mathbf{1}_d,\hspace{.01cm} -0.33,\hspace{.01cm} 0.03)} + 0.3\delta_{(0.5\cdot\mathbf{1}_d,\hspace{.01cm} 0.01\cdot\mI_D,\hspace{.01cm} \mathbf{0}_d,\hspace{.01cm} 0.2,\hspace{.01cm} 0.02)},
	\end{align*}
	where $\mathbf{1}_d = (1,1)$, $\mathbf{0}_d = (0,0)$, and $\mI_D$ is the identity matrix of size $D$. Model IV uses the same configuration for $G_0$ as Model III, with the exception that $\vc_1^0 = \mathbf{0}_d$ and $\vb_1^0 = 0.3$.
	\fi
	
	%%
	\begin{figure*}[!ht]
		\centering
		\begin{subfigure}[t]{.5\textwidth}
			\centering
			\includegraphics[width=\linewidth]{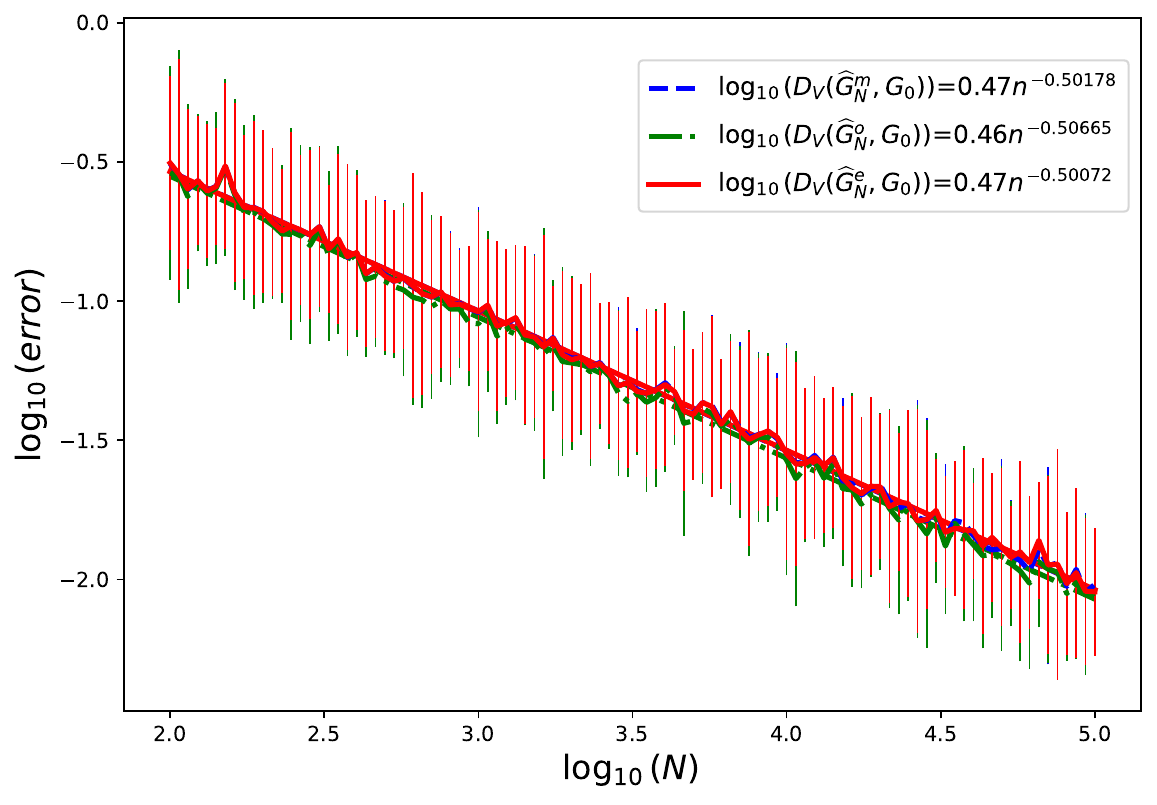}
			\caption{Voronoi loss distances.}
			\label{fig_convergence_rate_DV_1_K5_n0_0_nmin100_nmax100000_rep20_nnum100}
		\end{subfigure}%
		~
		\begin{subfigure}[t]{.5\textwidth}
			\centering
			\includegraphics[width=\linewidth]{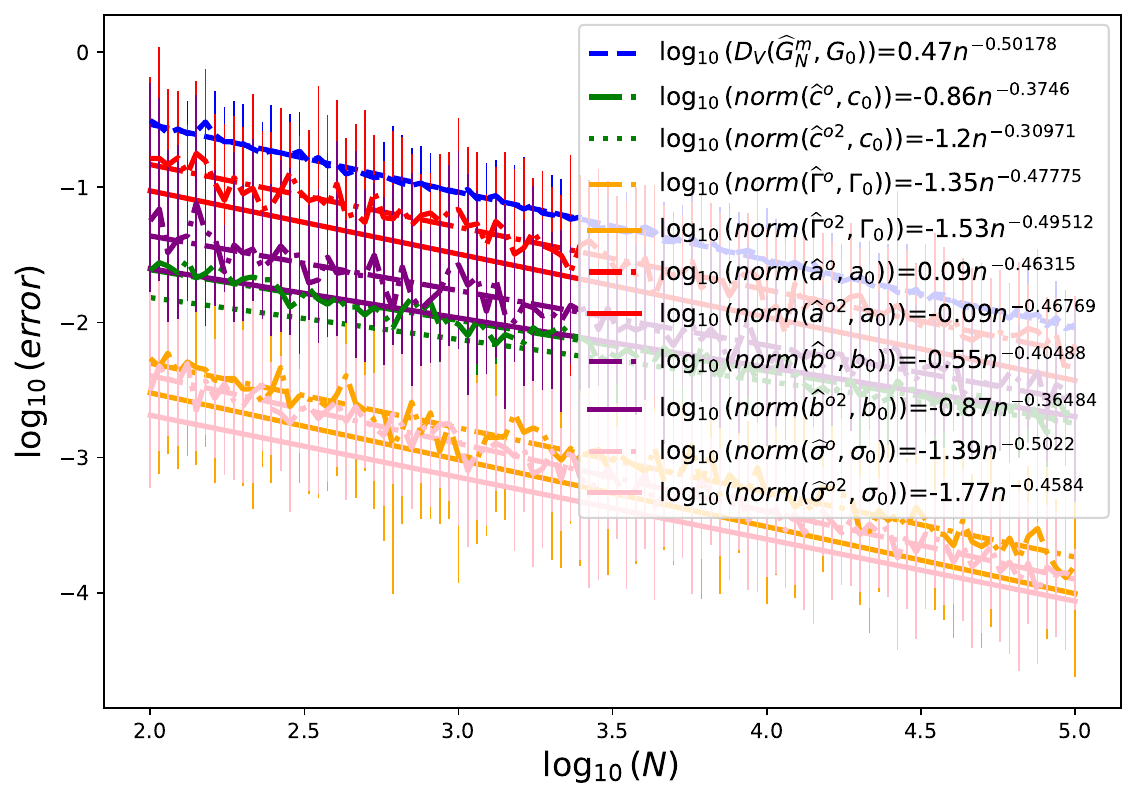}
			\caption{Parameter distances.}
			\label{fig_convergence_rate_atoms_1_K5_n0_0_nmin100_nmax100000_rep20_nnum100}
		\end{subfigure}
		\caption{Convergence rates of overfitted, exact-fitted, and merged mixing measures for GGMoE.}
	\end{figure*}
	
	%%% Mixing measures on the dendrogram with k = 10 and k0 = 3 (displaying in the clockwise order for ease of comparing between the merged measure and the truth)
	\begin{figure*}[!ht]
		\centering
		\begin{subfigure}[t]{.32\textwidth}
			\centering
			\includegraphics[width=\linewidth]{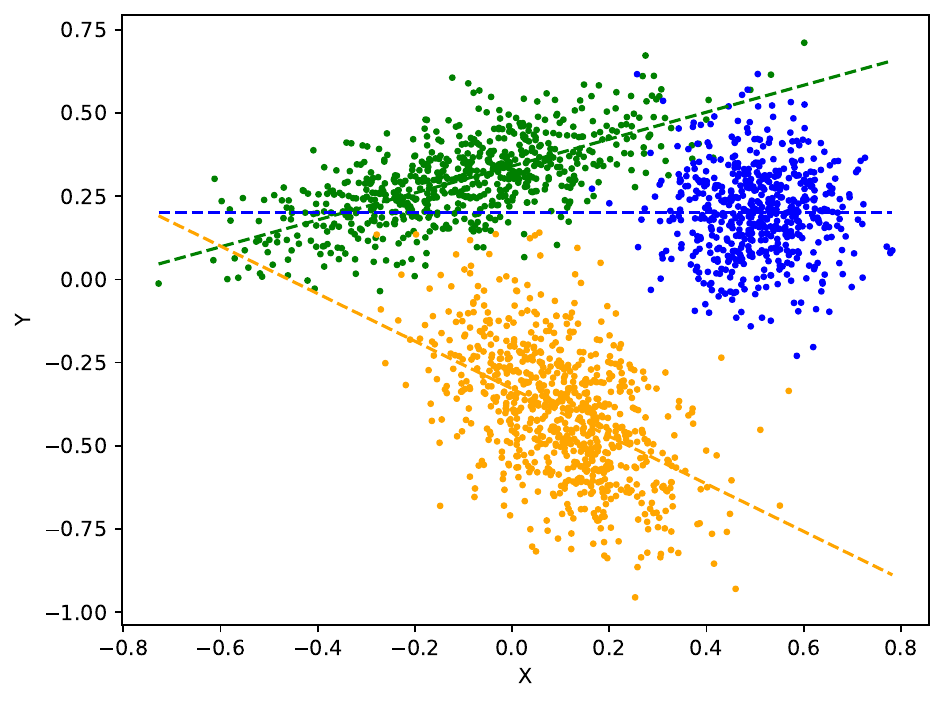}
			\caption{Data with true regression.}
			\label{fig_plot_true_GGMoE_model_1_K10_nmin100_nmax10000_rep20_nnum21_initAtom1}
		\end{subfigure}%
		~
		\begin{subfigure}[t]{.32\textwidth}
			\centering
			\includegraphics[width=\linewidth]{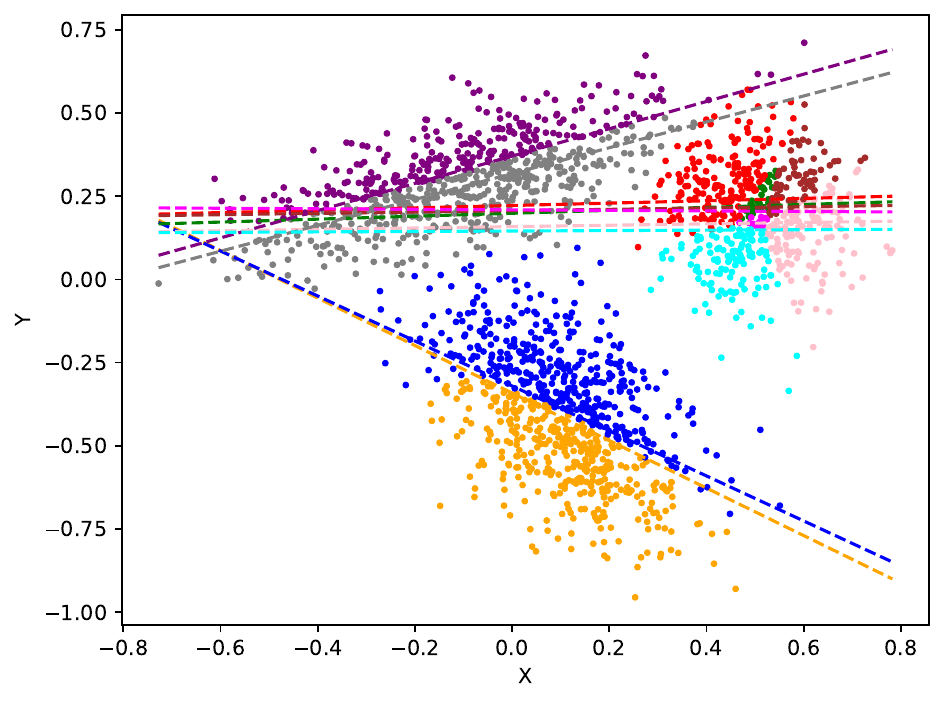}
			\caption{Overfitting $K = 10$.}
			\label{fig_plot_estimated_GGMoE_saved_data_model_1_kappa10_nmin100_nmax10000_rep20_nnum21_initAtom1}
		\end{subfigure}
		~
		\begin{subfigure}[t]{0.32\textwidth}
			\centering
			\includegraphics[width=\linewidth]{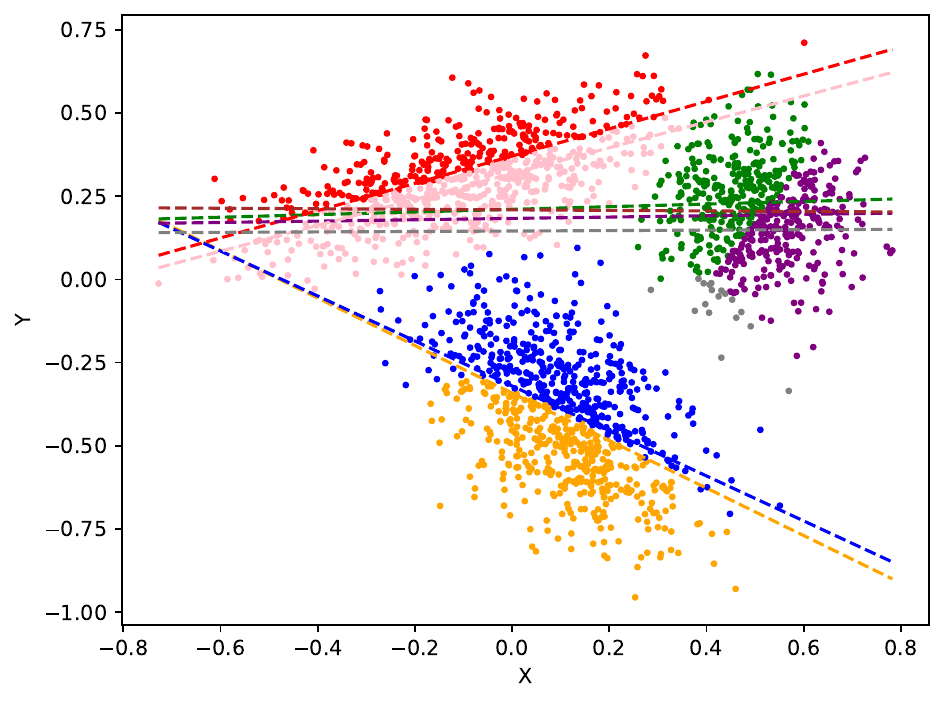}
			\caption{Merge $K = 8$.}
			\label{fig_plot_estimated_GGMoE_saved_data_model_1_kappa8_nmin100_nmax10000_rep20_nnum21_initAtom1}
		\end{subfigure}
		
		\begin{subfigure}[t]{0.32\textwidth}
			\centering
			\includegraphics[width=\linewidth]{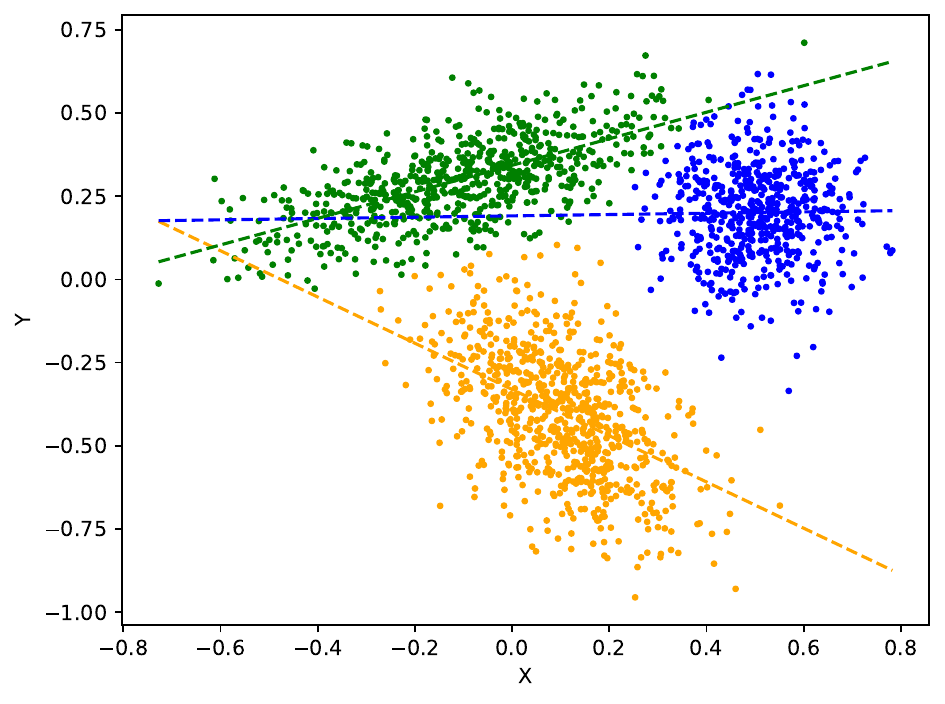}
			\caption{Merge $K = 3$.}
			\label{fig_plot_estimated_GGMoE_saved_data_model_1_kappa3_nmin100_nmax10000_rep20_nnum21_initAtom1}
		\end{subfigure}
		~
		\begin{subfigure}[t]{.32\textwidth}
			\centering
			\includegraphics[width=\linewidth]{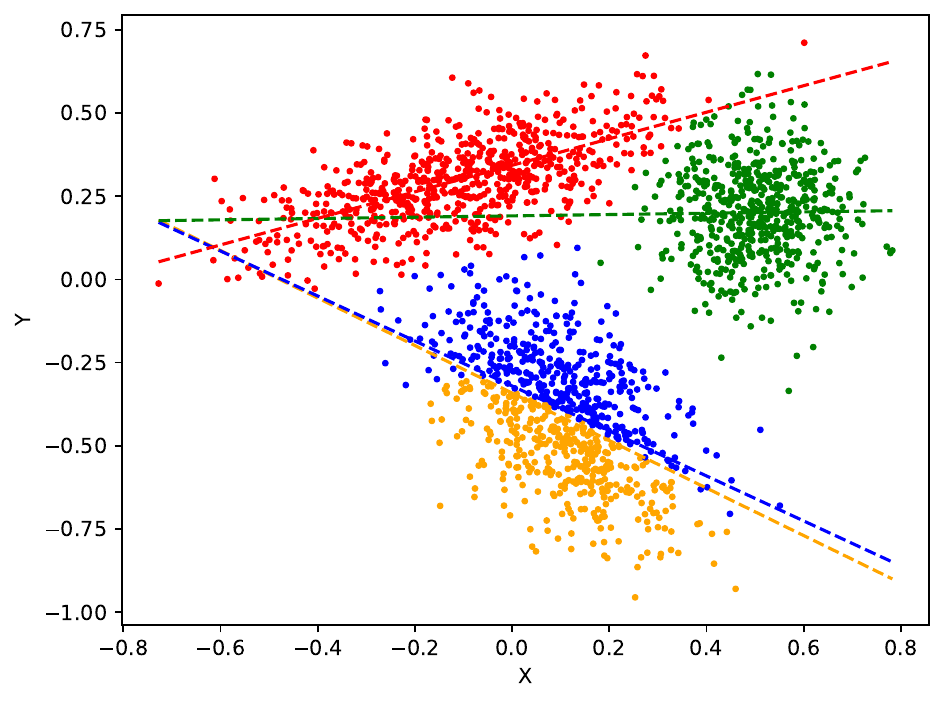}
			\caption{Merge $K = 4$.}
			\label{fig_plot_estimated_GGMoE_saved_data_model_1_kappa4_nmin100_nmax10000_rep20_nnum21_initAtom1}
		\end{subfigure}
		~
		\begin{subfigure}[t]{.32\textwidth}
			\centering
			\includegraphics[width=\linewidth]{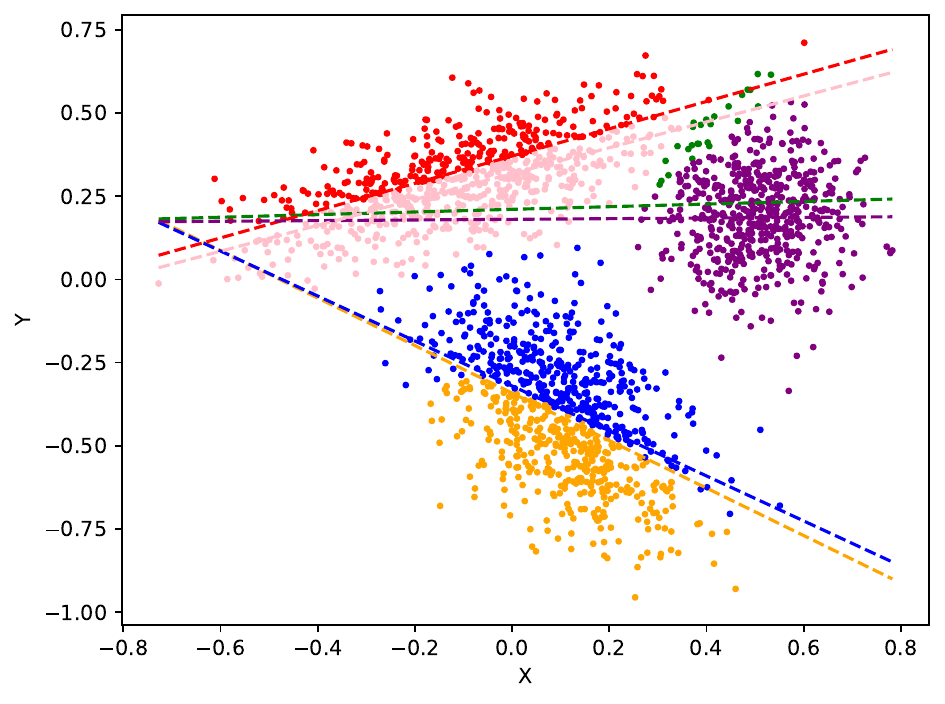}
			\caption{Merge $K = 6$.}
			\label{fig_plot_estimated_GGMoE_saved_data_model_1_kappa6_nmin100_nmax10000_rep20_nnum21_initAtom1}
		\end{subfigure}%
		
		\caption{Dendrogram mixing measures with $K = 10$ and $K_0 = 3$, arranged in a clockwise layout to enhance visual comparison between the merged estimates and the true mixing measure.}\label{fig_dendrogram_illustration}
	\end{figure*}
	
	{\bf Model selection with DSC.} \update{In this experimental setup, we fix $K =20$, $N_{\nums} = 64$, $N_{\min} = 200$, $N_{\max} = 2*10^4$, and set the number of repetitions to $N_{\rep} = 80$.} For each scenario, we record both the frequency in \cref{fig_model_selection_average_model_1_GGMoE_update} with which each method correctly identifies the true number of mixture components $K_0$ and the average number of components selected in \cref{fig_model_selection_proportion_model_1_GGMoE_update}.  For the AIC, BIC, and ICL criteria, the procedure involves fitting GGMoE models with $\kappa$ components using the EM algorithm for each $\kappa \in [K]$. The information criteria are then computed and minimized to select the optimal model. In contrast, the proposed DSC approach requires fitting the GGMoE model only once with $K = 10$ components. Then, a dendrogram is constructed from the inferred mixing measure, and the DSC is evaluated based on this hierarchical structure as detailed in~\cref{subsec_model_selection_DSC}, using a penalty scaling of $\omega_N = \log(N)$.
	The results indicate that AIC and BIC exhibit similar tendencies: they consistently overestimate the number of mixture components, often approaching the upper bound of $10$ as the sample size increases. These criteria tend to perform reasonably only when the sample size is neither too small nor excessively large. In contrast, ICL initially behaves similarly to AIC and BIC for small sample sizes but demonstrates improved performance as the sample size increases, ultimately converging to the true number of components, $K_0 = 3$. Our proposed DSC method outperforms all baseline criteria in accurately recovering the true number of mixture components and selecting a suitable number of components. Remarkably, DSC maintains strong performance even in small-sample regimes and achieves high empirical accuracy with moderately large sample sizes.
	
	\begin{figure*}[!ht]
		\centering
		\begin{subfigure}[t]{.45\textwidth}
			\centering
			\includegraphics[width=\linewidth]{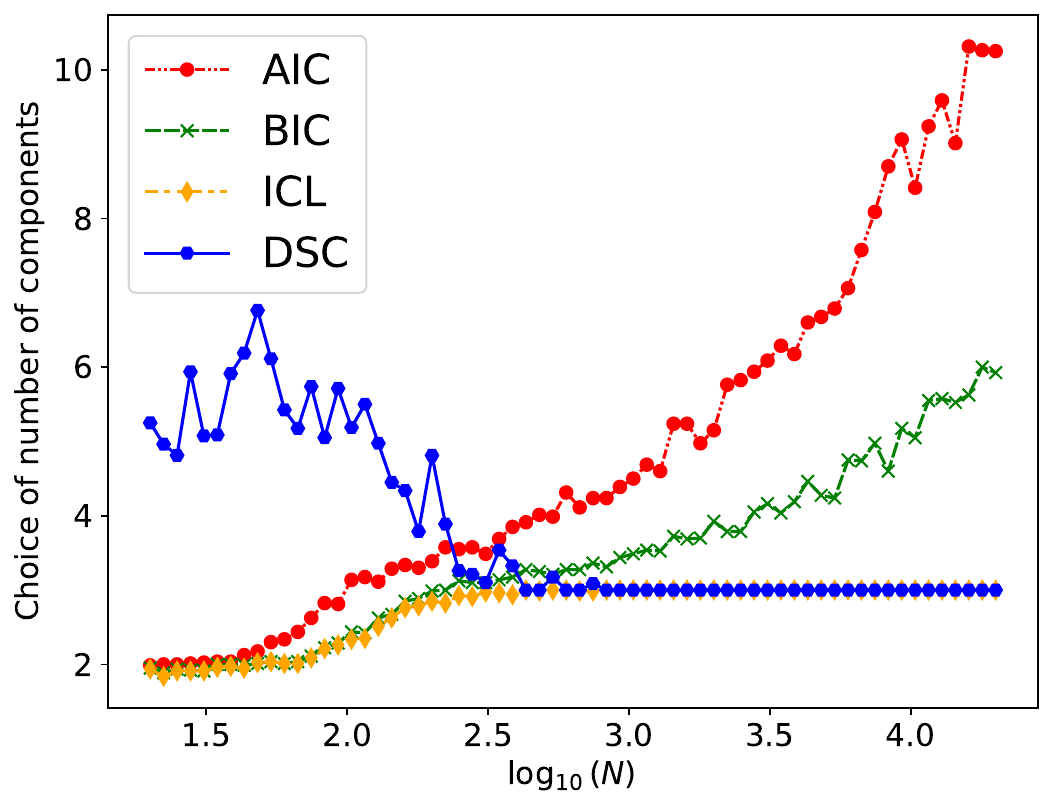}
			\caption{Average choices number of components.}
			\label{fig_model_selection_average_model_1_GGMoE_update}
		\end{subfigure}%
		~
		\begin{subfigure}[t]{.45\textwidth}
			\centering
			\includegraphics[width=\linewidth]{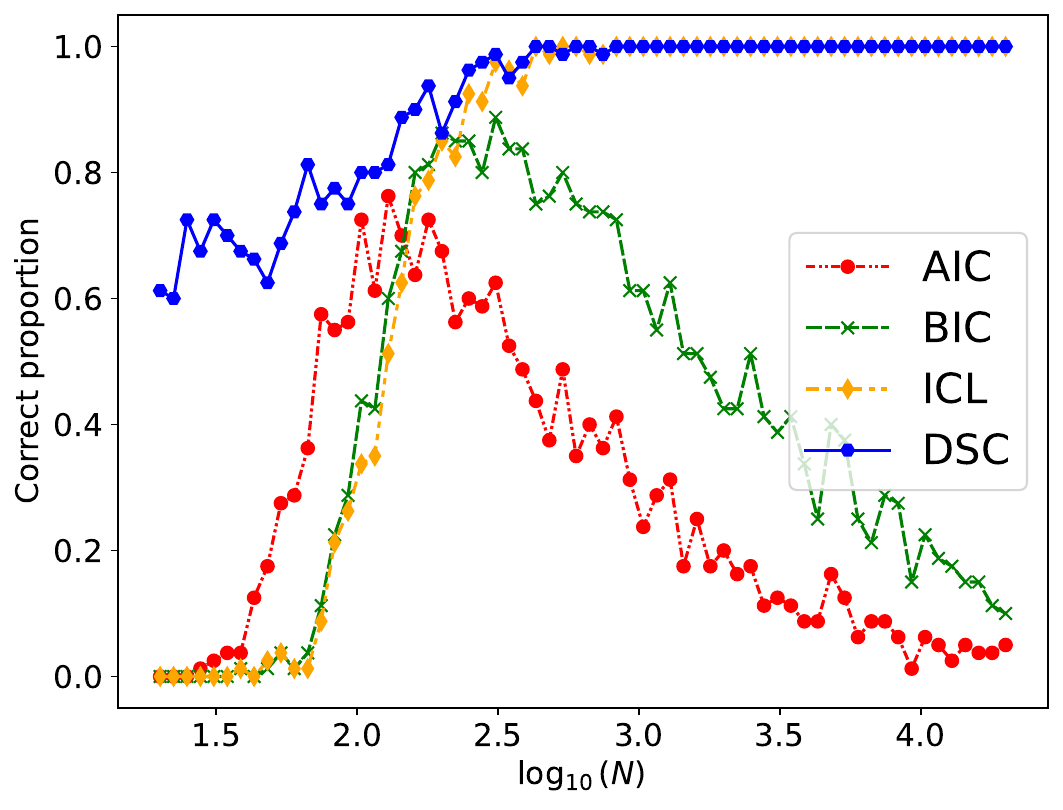}
			\caption{Proportion of selecting $K_0 = 3$.}
			\label{fig_model_selection_proportion_model_1_GGMoE_update}
		\end{subfigure}
		\caption{\update{Model selection with $K =20, K_0 = 3$, $N_{\nums} = 64$, $N_{\min} = 200$, $N_{\max} = 2*10^4$, $N_{\rep} = 80$.}}
	\end{figure*}

	\section{Conclusion, Limitations and Perspectives}\label{sec_conclusion}
	
	This work addresses a central challenge in the theoretical analysis of MoE models: the selection of the optimal number of mixture components in the presence of covariate-dependent gating and expert networks. By revisiting the framework of dendrograms of mixing measures, we proposed a novel and principled extension to GGMoE models. Our approach enables consistent estimation of the true number of components and attains pointwise optimal convergence rates in overfitted settings, all while avoiding the computational cost of fitting multiple models with different configurations. Through extensive simulations, we demonstrated that the proposed method not only achieves superior statistical accuracy in parameter estimation but also outperforms standard model selection criteria such as AIC, BIC, and ICL in recovering the correct model complexity. These findings highlight the utility of our approach in applications involving complex, heterogeneous data structures. 
	Looking forward, several avenues merit further investigation. In this work, we restrict our presentation to the univariate response case for clarity and simplicity. However, the extension to multiple responses can be developed analogously, albeit with more intricate notation and technical details. Next, extending the proposed framework to hierarchical or non-Gaussian MoE models could enhance its applicability in broader real-world scenarios. Moreover, theoretical analysis under weaker assumptions or in the presence of model misspecification would further deepen our understanding of model selection in MoE architectures. Finally, integrating this approach into scalable deep learning pipelines remains an exciting direction for future research, with the potential to inform the design of more robust and interpretable neural network models.
	
	%\iffalse
	\subsection*{Acknowledgments} TrungTin Nguyen and Christopher Drovandi acknowledge funding from the Australian Research Council Centre of Excellence for the Mathematical Analysis of Cellular Systems (CE230100001). Christopher Drovandi was aslo supported by an Australian Research Council Future Fellowship (FT210100260).
	%\fi

	%%%%%%%%%%%%%%%%%%%%%%%%%%%%%%%%%%%%%%%%%%%%%%%%%%%%%%%%%%%%
	% \newpage
	
	\appendix
	
	\iffalse
	\section*{Supplementary Materials}
	\fi
	\begin{center}
		\textbf{\Large Supplementary Materials for\\
			``Model Selection for Gaussian-gated Gaussian Mixture of Experts Using Dendrograms of Mixing Measures''}
	\end{center}

	In this supplementary material, we first provide additional details on the theoretical challenges, technical contributions and practical implication in \cref{sec_challenges}, aiming to enhance the reader's understanding of our main theoretical developments. Next, we revisit the identifiability and joint density estimation rates for GGMoE models in \cref{sec_identifiability_GGMoE}. We then present the proofs of the main theoretical results in \cref{sec_proofs_main_results}. Finally, we include additional experimental results in \cref{sec_additional_experiments}.

	\tableofcontents

	\section{From Theoretical Challenges and Technical Contributions to Practical Implications}\label{sec_challenges}
	
	\subsection{Characterizing Parameter Estimation Rates of GGMoE}
	
	The main challenges arise in analyzing the convergence rates of the estimated parameters in GGMoE models. We study these rates by applying the Taylor theorem to decompose the difference $p_{\widehat{G}_N}(\vy, \vx) - p_{G_0}(\vy, \vx)$ into a linear combination of linearly independent components. A key difficulty lies in the presence of linearly dependent derivative terms in the Taylor expansion, such as
	\begin{equation*}
		\dfrac{\partial^2 f_{\cL}}{\partial \vc \partial \vc^{\top}} = 2\dfrac{\partial f_{\cL}}{\partial \mGamma}, \quad \dfrac{\partial^2 f_{\cD}}{\partial \vb^2} = 2\dfrac{\partial f_{\cD}}{\partial \mSigma}.
	\end{equation*}
	To address this issue, we transform all left-hand side terms in the above expressions into their corresponding right-hand side forms. Nevertheless, a linear dependence remains between $\dfrac{\partial f_{\cL}}{\partial \vc}$ and $\vx$, since
	\[
	\dfrac{\partial f_{\cL}}{\partial \vc} = (\mGamma^{-1})(\vx -\vc)f_{\cL}.
	\]
	To fully clarify the challenge, we describe it in detail. Specifically, for $\valpha = (\evalpha_1, \evalpha_2, \evalpha_3, \evalpha_4, \evalpha_5) \in \N^D \times \N^{D \times D} \times\N^{D} \times \N \times \N$, define
	\begin{equation*}
		\cI_{l_1, l_2}(k) := \left\{ \valpha : \tau_0(\evalpha_1, \evalpha_2) + \evalpha_3 \geq l_1,\; |\evalpha_3| + \evalpha_4 + 2\evalpha_5 = l_2,\; 1 \leq |\valpha| \leq \bar{r}(|\sV_k|) \right\}.
	\end{equation*}
	We then define
	\begin{equation*}
		T_{l_1, l_2}^N(k) = \sum_{\ell \in \sV_k} \sum_{\valpha \in \cI_{l_1, l_2}} \frac{\evpi_\ell^N}{2^{|\evalpha_2| + |\evalpha_5|}\valpha!}t_{l_1-\evalpha_3, \tau_0(\evalpha_1, \evalpha_2)}^{(k)} (\Delta \vc_{\ell k}^N)^{\evalpha_1} (\Delta \mGamma_{\ell k}^N)^{\evalpha_2} (\Delta \mA_{\ell k}^N)^{\evalpha_3} (\Delta \vb_{\ell k}^N)^{\evalpha_4} (\Delta \mSigma_{\ell k}^N)^{\evalpha_5},
	\end{equation*}
	where $t_{w, \tau_0(\evalpha_1, \evalpha_2)}^{(k)}$ is defined by the following expansion:
	\begin{equation*}
		\dfrac{\partial^{|\evalpha_1| + 2|\evalpha_2|}}{\vc^{\tau_0(\evalpha_1, \evalpha_2)}} f_{\cL}(\vx \mid \vc_k^0, \mGamma_k^0) = \sum_{|w| = 0}^{|\evalpha_1| + 2|\evalpha_2|} t_{w, \tau_0(\evalpha_1, \evalpha_2)}^{(k)} \vx^w.
	\end{equation*}
	Although we can show that $T_{l_1, l_2}^N(k) / \VD(G_N, G_0) \to 0$, we aim to establish the stronger property:
	\begin{equation*}
		\frac{\sum\limits_{\ell \in \sV_k} \sum\limits_{\substack{\tau_0(\evalpha_1, \evalpha_2) + \evalpha_3 = l_1 \\ |\evalpha_3| + \evalpha_4 + 2\evalpha_5 = l_2}} \frac{\evpi_\ell^N}{2^{|\evalpha_2| + |\evalpha_5|} \valpha!} (\Delta \vc_{\ell k}^N)^{\evalpha_1} (\Delta \mGamma_{\ell k}^N)^{\evalpha_2} (\Delta \mA_{\ell k}^N)^{\evalpha_3} (\Delta \vb_{\ell k}^N)^{\evalpha_4} (\Delta \mSigma_{\ell k}^N)^{\evalpha_5}}{\VD(G_N, G_0)} \rightarrow 0.
	\end{equation*}
	
	In~\cite{nguyen_towards_2024}, there are errors in deriving this property, which we correct through updated proofs. A critical step involves understanding the relationship between $r(M)$ and $\bar{r}(M)$ in the systems of~\cref{eq_system_r_bar,eq_system_GGMoE}. While~\cite{nguyen_towards_2024} claims that $\bar{r}(M) \leq r(M)$, we rigorously establish that $\bar{r}(M) = r(M)$.
	
	\subsection{Model Selection via DSC for GGMoE Models}
	
	Another major technical challenge lies in verifying Condition \textbf{K} (see \cref{subsec_model_selection_DSC}) for model selection via DSC in GGMoE models, which has only been proven for simple location-scale models in~\cite{do_dendrogram_2024}.
	
	Furthermore, we introduce a novel Voronoi loss function in~\cref{eq_voronoi_loss_definition}, which generalizes the versions proposed in~\cite{nguyen_towards_2024} and~\cite{do_dendrogram_2024}. While the former was designed exclusively for parameter estimation, the latter was tailored to location-scale Gaussian mixtures of GGMoE models in clustering tasks. This new loss enables our merging procedure for GGMoE models in \cref{algorithm_mergeAtom_GGMoE}, which accounts for both the proximity between clusters and the similarity between expert regression functions, as illustrated in \cref{fig_dendrogram_illustration}.
	
	In particular, it holds that
	$$\VD(G, G_0) \gtrsim \VDT(G, G_0),$$ 
	where the Voronoi-type loss function $\VDT$ is defined in Equation (13) of~\cite{nguyen_towards_2024} for GGMoE models. Extending this approach, we conjecture that a similar relation holds for a novel Voronoi loss $\VDC(G, G_0)$ satisfying 
	$$\VDC(G, G_0) \gtrsim \VDS(G, G_0),$$ 
	where $\VDS$ is defined in Equation (13) of~\cite{nguyen_demystifying_2023}.
	
	Moreover, letting $\bar{r}(\widehat{G}) = \max_{k \in [K_0]} \bar{r}(|\sV_k|)$, we further establish that
	$$\VD(G, G_0) \gtrsim \W^{\bar{r}(\widehat{G})}_{\bar{r}(\widehat{G})}(G, G_0),$$
	where $\W^r_r$ denotes the Wasserstein-$r$ distance between mixing measures, as defined in \cref{eq_define_Wr}.
	Recall that for two discrete mixing measures $G = \sum_{k=1}^K \evp_k \delta_{\evtheta_k}$ and $G' = \sum_{l=1}^{K'} \evp'_l \delta_{\evtheta'_l}$, the Wasserstein-$r$ distance ($r \geq 1$) is given by
	\begin{equation}\label{eq_define_Wr}
		\W_r^r(G, G') := \inf_{\mQ \in \Pi(\vp, \vp')} \sum_{k=1}^K \sum_{l=1}^{K'} \emQ_{kl} \|\evtheta_k - \evtheta'_l\|^r,
	\end{equation}
	where $\Pi(\vp, \vp')$ is the set of all couplings between the mixing proportions $\vp = (\evp_1, \dots, \evp_K)$ and $\vp' = (\evp'_1, \dots, \evp'_{K'})$, that is,
	$$\Pi(\vp, \vp') = \left\{ \mQ \in \mathbb{R}_+^{K \times K'} : \sum_{k=1}^K \emQ_{kl} = \evp'_l, \sum_{l=1}^{K'} \emQ_{kl} = \evp_k, \ \forall k \in [K], \ l \in [K'] \right\}.$$
	
	Hence, our novel inverse bound in \cref{theorem_V_DG_GLLiM} provides a significantly stronger result than existing bounds for GGMoE models~\cite{manole_refined_2022, ho_convergence_2016, ho_singularity_2019, nguyen_towards_2024}.

	Finally, our new expert merging scheme serves as a bridge between model-based and model-free approaches and contributes to the theoretical understanding of the recent success of expert merging in deep neural network communities. This direction is particularly promising for improving computational efficiency in sparse MoE models~\cite{li_merge_2024, he_merging_2023, zhong_lory_2024, chen_retraining_free_2024, he_towards_2025, nguyen_camex_2025}.

	\section{Recap Identifiability and Joint Density Estimation Rate for GGMoE Models}\label{sec_identifiability_GGMoE}
	
	We begin by establishing the identifiability of the GGMoE model through the following \cref{fact_identifiability,fact_density_rate}:
	\begin{fact}[Identifiability of the GGMoE model, Proposition 1 in \cite{nguyen_towards_2024}]
		\label{fact_identifiability}
		Consider two mixing measures $G, G' \in \cO_{K}$. If their corresponding joint densities satisfy $p_G(\vx,\vy) = p_{G'}(\vx,\vy)$ almost everywhere on $\cX \times \cY$, then the measures are equivalent, i.e., $G \equiv G'$.
	\end{fact}
	This result implies that the joint density function uniquely determines the underlying mixing measure within the GGMoE framework.
	
	Subsequently, we investigate the convergence rate of the MLE of the joint density, denoted $p_{\widehat{G}_N}$, toward the true density $p_{G_0}$ under the Total Variation (TV) distance. Recall that the MLE is defined as:
	\begin{align}
		\label{eq_MLE_formulation}
		\widehat{G}_N \in \argmax_{G \in \cO_K} \sum_{n=1}^{N} \log\left(p_G(\vx_n, \vy_n)\right).
	\end{align}
	
	\begin{fact}[Joint Density Convergence Rate, Proposition 2 in \cite{nguyen_towards_2024}]
		\label{fact_density_rate}
		Let $\widehat{G}_N$ be the MLE defined in~\cref{eq_MLE_formulation}. Then, for some universal constants $C_1$ and $C_2$, the following inequality holds:
		\begin{align*}
			\mathbb{P}(\TV(p_{\widehat{G}_N},p_{G_0}) > C_1\sqrt{\log(N)/N}) \lesssim N^{-C_2}.
		\end{align*}
		This shows that the estimation error in TV distance vanishes at a parametric rate of order $\cO(N^{-1/2})$, up to logarithmic factors.
	\end{fact}
	
	This convergence result is instrumental in analyzing the parameter recovery problem for GGMoE models. Specifically, if one can relate the TV distance to a task-specific loss function $\VD$ as in \cref{eq_voronoi_loss_definition}, such that $\TV(p_{G},p_{G_0}) \gtrsim \VD(G,G_0)$ for any $G \in \cO_K(\vTheta)$, then it follows that $\widehat{G}_N$ also approximates $G_0$ at the rate $\cO(N^{-1/2})$ in terms of $\VD$. This connection enables the derivation of finite-sample guarantees for parameter estimation via analysis of $\VD(\widehat{G}_N, G_0)$.
	
	\section{Proofs of Main Results}\label{sec_proofs_main_results}
	
	\subsection{Proof of Theorem \ref{theorem_V_DG_GLLiM}}\label{sec_proof_theorem_V_DG_GLLiM}
	\subsubsection{Proof Sketch}
	Assume, by contradiction, that the claimed lower bound does not hold. Then, there exists a sequence of mixing measures $\{G_N\}_{N \in \Ns}$ such that $\TV(p_{G_N}, p_{G_0}) \to 0$ and simultaneously $\TV(p_{G_N}, p_{G_0}) / \VD(G_N, G_0) \to 0$ as $N \to \infty$. Since the parameter space $\Theta$ is compact and the model is identifiable, it follows that $G_N = \sum_{\ell=1}^K \evpi_\ell^N \delta_{\vtheta_\ell^N} \in \cO_K(\Theta)$, where each component $\vtheta_\ell^N := (\vc_\ell^N, \mGamma_\ell^N, \va_\ell^N, \vb_\ell^N, \mSigma_\ell^N)$ converges to its corresponding true parameter in $G_0$.
	
	To analyze the asymptotic behavior, we apply a Taylor expansion to the normalized difference 
	$$
	\frac{p_{G_N}(\vy, \vx) - p_{G_0}(\vy, \vx)}{\VD(G_N, G_0)},
	$$ 
	which yields a linear combination of basis functions composed of products of monomials in $\vx$ with $f_{\cL}(\vx \mid \vc_k^0, \mGamma_k^0)$, and higher-order derivatives of $f_{\cD}(\vy \mid (\va_k^0)^{\top}\vx + \vb_k^0, \mSigma_k^0)$ with respect to $h_1 := \mA^{\top} \vx + \vb$. A key step in this expansion leverages the identity for the Gaussian density,
	$$
	\frac{\partial^2 \cN}{\partial \vmu \partial \vmu^{\top}}(\vx \mid \vmu, \mathbf{\Sigma}) =2 \frac{\partial \cN}{\partial \mathbf{\Sigma}}(\vx \mid \vmu, \mathbf{\Sigma}),
	$$
	which allows us to represent second-order effects in terms of derivatives with respect to the covariance.
	
	Furthermore, due to the structure of the system defined by~\cref{eq_system_r_bar,eq_system_GGMoE}, and the convergence of $G_N$ to $G_0$, at least one coefficient in this expansion remains bounded away from zero as $N \to \infty$. This contradicts the assumption that the ratio $\TV(p_{G_N}, p_{G_0}) / \VD(G_N, G_0)$ vanishes, completing the argument.
	
	\textbf{System of polynomial equations.} To rigorously characterize the convergence rates of the parameter estimates, it is necessary to analyze the solvability of a certain system of polynomial equations, which was initially studied in \cite{ho_convergence_2016}. Specifically, for each integer $M \geq 2$, define ${r}(M)$ as the minimal positive integer $R$ such that the system
	\begin{align}
		\label{eq_system_r_bar}
		\sum_{m=1}^{M} \sum_{\substack{\evalpha_4,\evalpha_5 \in \mathbb{N}: \\ \evalpha_4 + 2 \evalpha_5 = \ell_2}} \frac{p_m^2 \, q_{1m}^{\evalpha_4} \, q_{2m}^{\evalpha_5}}{\evalpha_4! \, \evalpha_5!} = 0, \quad \forall~ \ell_2 \in [R],
	\end{align}
	admits no non-trivial solutions in the unknowns $\{p_m, q_{1m}, q_{2m}\}_{m=1}^M$. Here, a solution is deemed non-trivial if all $p_m \neq 0$ and at least one $q_{1m} \neq 0$. The following \cref{lemma_r_bar} summarizes known values of ${r}(M)$ for certain cases:
	
	\begin{fact}[Proposition 2.1, \cite{ho_convergence_2016}]
		\label{lemma_r_bar}
		For $M=2$, we have ${r}(2) = 4$, and for $M=3$, ${r}(3) = 6$. When $M \geq 4$, it holds that ${r}(M) \geq 7$.
	\end{fact}
	The proof of~\cref{lemma_r_bar} is given in \cite{ho_convergence_2016}.
	Moreover, part of the proof of~\cref{lemma_r_bar_GGMoE} is established in \cite{nguyen_towards_2024}, where it is shown that $\bar{r}(M) \leq r(M)$. We now prove the converse inequality, thereby proving the equality $r(M) = \bar{r}(M)$.
	
	\begin{proof}[Proof of \cref{lemma_r_bar_GGMoE}]
		Consider the polynomial system
		\begin{equation*}
			\sum_{m=1}^M \sum_{\valpha \in \mathcal{J}_{\ell_1, \ell_2}} \frac{p_m^2 \, q_{1m}^{\evalpha_1} q_{2m}^{\evalpha_2} q_{3m}^{\evalpha_3} q_{4m}^{\evalpha_4} q_{5m}^{\evalpha_5}}{\evalpha_1! \, \evalpha_2! \, \evalpha_3! \, \evalpha_4! \, \evalpha_5!} = 0,
		\end{equation*}
		where the unknowns are $\{(p_m, q_{1m}, q_{2m}, q_{3m}, q_{4m}, q_{5m})\}_{m=1}^M$ and indices $(\ell_1, \ell_2)$ satisfy $1 \leq \ell_1 + \ell_2 \leq R$. Here,
		\begin{equation*}
			\mathcal{J}_{\ell_1, \ell_2} := \{\valpha = (\evalpha_l)_{l=1}^5 \in \mathbb{N}^5 : \evalpha_1 + 2\evalpha_2 + \evalpha_3 = \ell_1, \quad \evalpha_3 + \evalpha_4 + 2\evalpha_5 = \ell_2 \}.
		\end{equation*}
		Restricting attention to the case where $q_{1m} = q_{2m} = q_{3m} = 0$ for all $m \in [M]$, the system reduces to
		\begin{equation*}
			\sum_{m=1}^M \sum_{\evalpha_4 + 2\evalpha_5 = \ell_2} \frac{p_m^2 \, q_{4m}^{\evalpha_4} q_{5m}^{\evalpha_5}}{\evalpha_4! \, \evalpha_5!} = 0, \quad \forall \ \ell_2 \in [R],
		\end{equation*}
		which matches the form of system~\cref{eq_system_r_bar}. From this, we conclude that $\bar{r}(M) \geq r(M)$.
		
		Conversely, by focusing on the subsystem with $\ell_1 = 0$, we consider
		\begin{equation*}
			\sum_{m=1}^M \sum_{\evalpha_3 + 2\evalpha_5 = \ell_2} \frac{p_m^2 \, q_{4m}^{\evalpha_4} q_{5m}^{\evalpha_5}}{\evalpha_4! \, \evalpha_5!} = 0, \quad \forall \ \ell_2 \in [R],
		\end{equation*}
		which implies $\bar{r}(M) \leq r(M)$, completing the proof.
	\end{proof}

	\subsubsection{Proof Detail}
	{\bf Notation.} Before presenting the proofs, we introduce some notations used throughout the text. For any vector $\vv \in \sR^{D}$, we denote its $d$-th entry by either $\evv_d$ or $\evv^{(d)}$. For any $\vp \in \sN^{D}$, define $\vv^{\vp} := \evv_1^{p_1} \evv_2^{p_2} \dots \evv_D^{p_D}$ and $\vp! := \evp_1! \dots \evp_D!$. Furthermore, given $\vmu \in \sR^{D}$, $\valpha \in \sN^{D}$, and a differentiable function $f$ of $\vmu$, we write the partial derivative of order $|\valpha|$ as
	\begin{equation*}
		\dfrac{\partial^{|\valpha|}}{\partial^{\valpha} \vmu}f(\vmu) := \dfrac{\partial^{|\valpha|}}{\partial \evmu_1^{\evalpha_1} \dots \partial \evmu_D^{\evalpha_D}} f(\vmu).
	\end{equation*}
	
	Throughout the remainder of this document, we assume that all mixing measures $G_N$ belong to the set
	\begin{equation*}
		\mathcal{O}_{K, c_0}(\Theta) := \left\{ G = \sum_{\ell = 1}^{K'} \pi_\ell \delta_{(\vc_\ell, \mGamma_\ell, \mA_\ell, \vb_\ell, \mSigma_\ell)} : 1 \leq K' \leq K,\ \sum_{\ell=1}^{K'} \pi_\ell = 1,\ \pi_\ell \geq c_0,\ (\vc_\ell, \mGamma_\ell, \mA_\ell, \vb_\ell, \mSigma_\ell) \in \Theta \right\}.
	\end{equation*}
	
	Suppose, by contradiction, that there exists a sequence of mixing measures $(G_N)_N$ such that $\TV(p_{G_N}, p_{G_0}) \to 0$ and $\TV(p_{G_N}, p_{G_0}) / \VD(G_N, G_0) \to 0$ as $N \to \infty$. Since each $G_N$ consists of at most $K$ atoms, we can extract a subsequence (still denoted $G_N$ for simplicity) with a fixed number of atoms $K^* \leq K$.
	
	Due to the compactness of the parameter space $\Theta$, we may extract a further subsequence such that
	\begin{equation*}
		\pi_k^N \to \pi_k^*,\quad \vc_k^N \to \vc_k^*,\quad \mGamma_k^N \to \mGamma_k^*,\quad \mA_k^N \to \mA_k^*,\quad \vb_k^N \to \vb_k^*,\quad \mSigma_k^N \to \mSigma_k^* \quad \text{as } N \to \infty.
	\end{equation*}
	
	Let $G^* = \sum_{k=1}^{K^*} \pi_k^* \delta_{(\vc_k^*, \mGamma_k^*, \mA_k^*, \vb_k^*, \mSigma_k^*)}$. By Fatou's lemma,
	\begin{align*}
		\frac{1}{2} \int \liminf_N |p_{G_N} - p_{G_0}| 
		&\leq \frac{1}{2} \liminf_N \int |p_{G_N} - p_{G_0}| \\
		&= \liminf_N \TV(p_{G_N}, p_{G_0}) = 0.
	\end{align*}
	
	Next, observe that
	\begin{align*}
		\liminf_N |p_{G_N} - p_{G_0}| 
		&= \liminf_N \left| \sum_{k=1}^{K^*} f_{\cD}(\vy \mid (\mA_k^N)^{\top} \vx + \vb_k^N, \mSigma_k^N) f_{\cL}(\vx \mid \vc_k^N, \mGamma_k^N) - p_{G_0}(\vy, \vx) \right| \\
		&= \left| \sum_{k=1}^{K^*} f_{\cD}(\vy \mid (\mA_k^*)^{\top} \vx + \vb_k^*, \mSigma_k^*) f_{\cL}(\vx \mid \vc_k^*, \mGamma_k^*) - p_{G_0}(\vy, \vx) \right| \\
		&= |p_{G^*}(\vy, \vx) - p_{G_0}(\vy, \vx)|.
	\end{align*}
	The second equality holds because both $f_{\cD}$ and $f_{\cL}$ are continuous with respect to their parameters. Hence, $\TV(p_{G^*}, p_{G_0}) = 0$, implying $p_{G^*}(\vy, \vx) = p_{G_0}(\vy, \vx)$ almost surely in $(\vx, \vy)$. By the identifiability of the model, it follows that $G^* = G_0$.
	
	By rearranging indices and using the assumption $\pi_k^N \geq c_0$ (which guarantees that each atom lies in a distinct Voronoi cell and converges to a unique true atom), we conclude that
	\begin{equation*}
		\sum_{\ell \in \sV_k} \evpi_\ell^N \to \pi_k^0,\quad \vc_\ell^N \to \vc_k^0,\quad \mGamma_\ell^N \to \mGamma_k^0,\quad \mA_\ell^N \to \mA_k^0,\quad \vb_\ell^N \to \vb_k^0,\quad \mSigma_\ell^N \to \mSigma_k^0,\quad \forall k \in [K_0],\ \ell \in \sV_k.
	\end{equation*}
	
	Also, each Voronoi cell has the same number for all N large enough, so that $\sV_k$ does not change with large $N$ for all $k \in [K_0]$ and 
	\begin{equation*}
		\VD(G_N, G_0) \rightarrow 0 \quad \text{as } N \rightarrow \infty.
	\end{equation*}
	%%%%%%%%%%%%%%%%%
	\textbf{Step 1 - Taylor expansion. } Now, we consider the quantity
	\begin{align*}
		&p_{G_N}(\vy, \vx) - p_{G_0}(\vy, \vx) \\
		&= \sum_{k=1}^{K_0}\sum_{\ell \in \sV_k}\evpi_\ell^Nf_{L}(\vx|\vc_k^N, \mGamma_k^N)f_{\cD}(\vy| (\mA_\ell^N)^{\top} \vx + \vb_\ell^N, \mSigma_\ell^N ) - \sum_{k=1}^{K_0}\pi_k^0f_{L}(\vx|\vc_k^0, \mGamma_k^0)f_{\cD}(\vy| \mA_k^0 \vx + \vb_k^0, \mSigma_k^0)\\
		&= \sum_{k: |\sV_k|>1}\sum_{\ell \in \sV_k}\evpi_\ell^N\left [f_{\cL}(\vx|\vc_k^N, \mGamma_k^N)f_{\cD}(\vy| (\mA_\ell^N)^{\top} \vx + \vb_\ell^N, \mSigma_\ell^N ) - f_{\cL}(\vx|\vc_k^0, \mGamma_k^0)f_{\cD}(\vy| \mA_k^0 \vx + \vb_k^0, \mSigma_k^0) \right] \\
		&+\sum_{k: |\sV_k|=1}\sum_{\ell \in \sV_k}\evpi_\ell^N\left [f_{\cL}(\vx|\vc_k^N, \mGamma_k^N)f_{\cD}(\vy| (\mA_\ell^N)^{\top} \vx + \vb_\ell^N, \mSigma_\ell^n ) - f_{\cL}(\vx|\vc_k^0, \mGamma_k^0)f_{\cD}(\vy| \mA_k^0 \vx + \vb_k^0, \mSigma_k^0) \right]
		\\
		&+ \sum_{k=1}^{K_0}\left( \sum_{\ell \in \sV_k}\evpi_\ell^N - \pi_k^0 \right)f_{\cL}(\vx|\vc_k^0, \mGamma_k^0)f_{\cD}(\vy| (\mA_k^0)^{\top} \vx + \vb_k^0, \mSigma_k^0)\\
		&= A_N + B_N + C_N.
	\end{align*}
	Performing Taylor expansion up to the $\bar{r}(|\sV_k|)$-th order, and then rewrite $A_N$ with a note that $\valpha = (\evalpha_1, \evalpha_2, \evalpha_3, \evalpha_4, \evalpha_5) \in \N^D \times \N^{D \times D} \times\N^{ D} \times \N \times \N$ as follows :
	\begin{align*}
		A_N &= \sum_{k: |\sV_k| > 1}\sum_{\ell \in \sV_k}\sum_{|\valpha| = 1}^{\bar{r}(|\sV_k|)}\dfrac{\evpi_\ell^N}{\valpha!}(\Delta \vc_{\ell k})^{\evalpha_1}(\Delta \mGamma_{\ell k})^{\evalpha_2}(\Delta \mA_{\ell k})^{\evalpha_3}(\Delta \vb_{\ell k})^{\evalpha_4}(\Delta \mSigma_{\ell k})^{\evalpha_5}\\
		&\times\dfrac{\partial^{|\evalpha_1| + |\evalpha_2|}}{\partial \vc^{\evalpha_1}\mGamma^{\evalpha_2}}f_{\cL}(\vx|\vc_k^0, \mGamma_k^0)\cdot\dfrac{\partial^{ |\evalpha_3| + |\evalpha_4| + |\evalpha_5|}}{\partial \mA^{\evalpha_3} \partial \vb^{\evalpha_4} \partial \mSigma^{\evalpha_5}}f_{\cD}(\vy| \mA_k^0 \vx + \vb_k^0, \mSigma_k^0)+ R_1^n(\vy, \vx),
	\end{align*}
	where $R_1^n(\vy, \vx)$ is the remainder term such that
	\begin{equation*}
		R_1^n(\vy, \vx) = o(\sum_{k: |\sV_k| > 1}\sum_{\ell \in \sV_k}\evpi_\ell^N (\rVert \Delta \vc_{\ell k} \lVert^{\bar{r}(|\sV_k|)}+ \rVert \Delta \mGamma_{\ell k} \lVert^{\bar{r}(|\sV_k|)}+ \rVert \Delta \mA_{\ell k} \lVert^{\bar{r}(|\sV_k|)} + \rVert \Delta \vb_{\ell k} \lVert^{\bar{r}(|\sV_k|)} + \rVert \Delta \mSigma_{\ell k} \lVert^{\bar{r}(|\sV_k|)}).
	\end{equation*}
	Hence $R_1^n(\vy, \vx) = o(\VD(G, G_0))$.Since $f_{\cL}$ $D-$dimensional Gaussian density functions, we have the following property,
	\begin{equation*}
		\dfrac{\partial f_{\cL}}{\partial \mGamma} = \dfrac{1}{2}\dfrac{\partial^2 f_{\cL}}{\partial \vc \partial \vc^{\top}},
	\end{equation*}
	then
	\begin{equation*}
		\dfrac{\partial^{|\evalpha_1| + |\evalpha_2|}}{\partial \vc^{\evalpha_1}\partial\mGamma^{\evalpha_2}}f_{\cL}(\vx|\vc_k^0, \mGamma_k^0) = \dfrac{1}{2^{|\evalpha_2|}}\dfrac{\partial ^{|\evalpha_1| + 2|\evalpha_2| }}{\partial c^{\tau_0(\evalpha_1. \evalpha_2)}}f_{\cL}(\vx| \vc_k^0, \mGamma_k^0),
	\end{equation*}
	where $\tau_0(\evalpha_1, \evalpha_2) := \left ( \evalpha_1^{(v)} + \sum_{u=1}^{D}(\evalpha_2^{(u,v) }+ \evalpha_2^{(v,u)}) \right)_{v=1}^{D} \in \N^{D}$. Furthermore, we have the following equality, define $h_1 = \mA^{\top} \vx + \vb,$
	\begin{equation*}
		\dfrac{\partial ^{|\evalpha_3|}}{\partial \mA^{\evalpha_3}}f_{\cD}(\vy| (\mA_k^0)^{\top} \vx + \vb_k^0, \mSigma_k^0) = \vx^{\evalpha_3}\dfrac{\partial ^{|\evalpha_3|}}{\partial h_1^{|\evalpha_3|}}f_{\cD}(\vy| (\mA_k^0)^{\top} \vx + \vb_k^0, \mSigma_k^0).
	\end{equation*}
	Also, 
	\begin{equation*}
		\dfrac{\partial ^{|\evalpha_4| + |\evalpha_5|}}{\partial \vb^{\evalpha_4} \partial\mSigma^{\evalpha_5}}f_{\cD}(\vy| (\mA_k^0)^{\top} \vx + \vb_k^0, \mSigma_k^0) = \dfrac{1}{2^{|\evalpha_5| }}\dfrac{\partial ^{|\evalpha_4| + 2|\evalpha_5|}}{\partial h_1^{\evalpha_4 + 2\evalpha_5}}f_{\cD}(\vy| (\mA_k^0)^{\top} \vx + \vb_k^0, \mSigma_k^0).
	\end{equation*}
	Combining these results together, $A_N$, can be represented as follows:
	\begin{align*}
		A_N &= \sum_{k: |\sV_k| > 1}\sum_{\ell \in \sV_k}\sum_{|\valpha| = 1}^{\bar{r}(|\sV_k|)}\dfrac{\evpi_\ell^N}{2^{|\evalpha_2| +|\evalpha_5| }\valpha!}(\Delta \vc_{\ell k})^{\evalpha_1}(\Delta \mGamma_{\ell k})^{\evalpha_2}(\Delta \mA_{\ell k})^{\evalpha_3}(\Delta \vb_{\ell k})^{\evalpha_4}(\Delta \mSigma_{\ell k})^{\evalpha_5}\\
		&\times\dfrac{\partial ^{|\evalpha_1| + 2|\evalpha_2|}}{\vc^{\tau_0(\evalpha_1, \evalpha_2)}}f_{\cL}(\vx|\vc_k^0, \mGamma_k^0) \cdot \vx^{\evalpha_3}\dfrac{\partial ^{|\evalpha_3| + |\evalpha_4| + 2|\evalpha_5|}}{\partial h_1^{|\evalpha_3| + \evalpha_4 + 2\evalpha_5}}f_{\cD}(\vy| (\mA_k^0)^{\top} \vx + \vb_k^0, \mSigma_k^0) + R_1^n(\vy, \vx).
	\end{align*}
	Noticeably, if a density function $f(X) \sim \mathcal{N}_{d}(\vmu, \mathbf{\Sigma}), $ then $\dfrac{\partial f}{\partial \vmu} = (f)\mathbf{\Sigma}^{-1}(X - \vmu)$, so we let 
	\begin{equation*}
		\dfrac{\partial ^{|\evalpha_1| + 2|\evalpha_2|}}{\vc^{\tau_0(\evalpha_1, \evalpha_2)}}f_{\cL}(\vx|\vc_k^0, \mGamma_k^0) = \sum_{|w| = 0}^{|\evalpha_1| + 2|\evalpha_2|}t_{w, \tau_0(\evalpha_1, \evalpha_2)}^{(k)}\vx^{w},
	\end{equation*}
	where $t_{w, \tau_0(\evalpha_1, \evalpha_2)}^{(k)} $ is a constant depend on $w , \tau_0(\evalpha_1, \evalpha_2)$ and $\vc_k^0, \mGamma_k^0$. Define, with $l_1 \in \N^{D}$ and $l_2 \in \N$, and by convenience, with $\tau \in \N^D$, define $\tau \geq l_1$ as greater for all coordinates $\tau^{(u)} \geq l_1^{(u)}, \forall u \in [D].$
	\begin{equation*}
		\cI_{l_1, l_2}(k) := \lbrace \valpha : \tau_0(\evalpha_1, \evalpha_2) + \evalpha_3 \geq l_1 , |\evalpha_3| + \evalpha_4 + 2\evalpha_5 = l_2 , 1 \leq |\valpha| \leq \bar{r}(|\sV_k|)
		\rbrace.
	\end{equation*}
	We can rewrite $A_N$ as
	\begin{equation*}
		A_N = \sum_{k: |\sV_k| > 1}\sum_{|l_1| + |l_2| = 1}^{2\bar{r}(|V_K|)}T_{l_1, l_2}^n(k)\vx^{l_1}\dfrac{\partial ^{|l_2|}}{\partial h_1^{l_2}}f_{\cD}(\vy| (\mA_k^0)^{\top}\vx + \vb_k^0, \mSigma_k^0)f_{\cL}(\vx| \vc_k^0, \mGamma_k^0) + R_1^N(\vy,\vx),
	\end{equation*}
	where $T_{l_1, l_2}^N(k)$ be defined as follow:
	\begin{equation*}
		T_{l_1, l_2}^N(k) = \sum_{\ell \in \sV_k}\sum_{\substack{\valpha \in \cI_{l_1, l_2}} } \dfrac{\evpi_\ell^N}{2^{|\evalpha_2| + |\evalpha_5|}\valpha!}t_{l_1-\evalpha_3, \tau_0(\evalpha_1, \evalpha_2)}^{(k)}(\Delta \vc_{\ell k}^N)^{\evalpha_1}(\Delta \mGamma_{\ell k}^N)^{\evalpha_2}(\Delta \mA_{\ell k}^N)^{\evalpha_3}(\Delta \vb_{\ell k}^N)^{\evalpha_4}(\Delta \mSigma_{\ell k}^N)^{\evalpha_5},
	\end{equation*}
	Likewise, for each $k \in [K_0] : |\sV_k| = 1$, by means of Taylor expansion up to the first order, $B_N$ is written as follows:
	\begin{equation*}
		B_N = \sum_{k: |\sV_k| = 1}\sum_{|l_1| + |l_2| = 1}^{2}T_{l_1, l_2}^N(k)\vx^{l_1}\dfrac{\partial ^{|l_2|}}{\partial h_1^{l_2}}f_{\cD}(\vy| (\mA_k^0)^{\top} \vx + \vb_k^0, \mSigma_k^0)f_{\cL}(\vx| \vc_k^0, \mGamma_k^0) + R_2^N(\vy,\vx),
	\end{equation*}
	where $R_2^N(\vy, \vx)$ is a remainder such that $R_2^N(\vy, \vx) = o(\sum_{k : |\sV_k| = 1}\sum_{\ell \in \sV_k}\evpi_\ell^N\lVert \Delta\theta_{\ell k}\lVert) =o(\VD(G, G_0)).$ Furthermore, with a convention denote $T_{0,0}^N(k) = \sum_{\ell \in \sV_k} \pi_l^N - \pi_k^0.$ So 
	\begin{equation*}
		C_N = \sum_{k=1}^{K_0}T_{0,0}^N(k)f_{\cD}(\vy|(\mA_k^0)^{\top} \vx + \vb_k^0, \mSigma_k^0)f_{\cL}(\vx| \vc_k^0, \mGamma_k^0) ,
	\end{equation*}
	Noticing that $A_N, B_N$ and $C_N$ can be treated as linear combination of elements of the following set:
	\begin{equation*}
		\cF := \left \lbrace \vx^{l_1}\dfrac{\partial ^{|l_2|}}{\partial h_1^{l_2}}f_{\cD}(\vy| (\mA_k^0)^{\top} \vx + \vb_k^0, \mSigma_k^0)f_{\cL}(\vx| \vc_k^0, \mGamma_k^0) : k \in [K_0], 0\leq |l_1| + |l_2| \leq  2\bar{r}(|\sV_k|)| \right \rbrace.
	\end{equation*}
	\textbf{Step 2 - Proof of non-vanishing coefficients by contradiction:} Assume that all coefficients in the representation of $A_N / \VD(G_N, G_0) $,  $B_N / \VD(G_N, G_0) $ and $C_N / \VD(G_N, G_0)$ goes to 0 as $N \rightarrow \infty.$
	Then, by assumption, $T_{0,0}^N(k) / \VD(G_N, G_0) \rightarrow 0 $ for all $k \in [K_0]$, and taking summation of them, we get
	\begin{equation*}
		\dfrac{1}{\VD(G_N, G_0)}\sum_{k=1}^{K_0}\left| \sum_{\ell \in \sV_k} \pi_l^N - \pi_k^0\right | \rightarrow 0
	\end{equation*}
	For $k$ such that $|\sV_k| = 1$, consider all $l_1, l_2$ such that $(|l_1|, |l_2|) \in \lbrace (1, 0), (2,0), (1,1), (0,1), (0,2) \rbrace$, we get that
	\begin{equation*}
		\dfrac{\sum_{k : |\sV_k| = 1}\sum_{\ell \in \sV_k}\pi_\ell(\norm{\vc_\ell^N - \vc_k^0} + \norm{\mGamma_\ell^N - \mGamma_k^0} + \norm{\mA_\ell^N - \mA_k^0} + \norm{\vb_\ell^N - \vb_k^0} + \norm{\mSigma_\ell^N - \mSigma_k^0})}{\VD(G_N, G_0)} \rightarrow 0.
	\end{equation*}
	%%%%%%%%%%%%%%%%%%%%%%%%%%%%%%%%%%%%%%%%%%%%%%%%%%%%%%%%%%%%%%%%%%%%%%%%%%%%%%%%%%%%%%%%%
	Next, we introduce some notation as follows, defining
	\begin{equation*}
		\cJ_{l_1, l_2}^{(i)}(k) := \cI_{l_1, l_2}(k) \cap \lbrace \valpha : |\tau_0(\evalpha_1 , \evalpha_2)| + |\evalpha_3| = |l_1| + i \rbrace 
	\end{equation*}
	and let
	\begin{equation*}
		S_{i, l_1, l_2}^N(k) = \sum_{\ell \in \sV_k}\sum_{\valpha \in \cJ_{l_1, l_2}^{(i)}(k)}\dfrac{\evpi_\ell^N}{2^{|\evalpha_2| + |\evalpha_5|}\valpha!}t_{l_1-\evalpha_3, \tau_0(\evalpha_1, \evalpha_2)}^{(k)}(\Delta \vc_{\ell k}^N)^{\evalpha_1}(\Delta \mGamma_{\ell k}^N)^{\evalpha_2}(\Delta \mA_{\ell k}^N)^{\evalpha_3}(\Delta \vb_{\ell k}^N)^{\evalpha_4}(\Delta \mSigma_{\ell k}^N)^{\evalpha_5}.
	\end{equation*}
	And we will prove that for every $k$ such that $|\sV_k| > 2$ and every $l_1$ and $l_2$, then $S_{0,l_1, l_2}^N(k) / \VD(G_N, G_0) \rightarrow 0.$ We will prove it by backward induction, fixed $l_2$, and for any $|l_1| = 2\bar{r}(|\sV_k|) - |l_2|$, we have $T_{l_1, l_2}^N(k) / \VD(G_N, G_0) \rightarrow 0,$ equivalent to 
	\begin{equation*}
		\dfrac{\sum_{\ell \in \sV_k}\sum_{\valpha \in \mathcal{I}_{l_1, l_2}(k)}\dfrac{\evpi_\ell^N}{2^{|\evalpha_2| + |\evalpha_5|}\valpha!}t_{l_1-\evalpha_3, \tau_0(\evalpha_1, \evalpha_2)}^{(k)}(\Delta \vc_{\ell k}^N)^{\evalpha_1}(\Delta \mGamma_{\ell k}^N)^{\evalpha_2}(\Delta \mA_{\ell k}^N)^{\evalpha_3}(\Delta \vb_{\ell k}^N)^{\evalpha_4}(\Delta \mSigma_{\ell k}^N)^{\evalpha_5}}{\VD(G_N, G_0)} \rightarrow 0.
	\end{equation*}
	Moreover, the LHS of above equation is equal to 
	\begin{align*}
		LHS &= \dfrac{\sum_{\ell \in \sV_k}\sum_{\valpha \in \mathcal{J}_{l_1, l_2}^{(0)}(k)}\dfrac{\evpi_\ell^N}{2^{|\evalpha_2| + |\evalpha_5|}\valpha!}t_{l_1-\evalpha_3, \tau_0(\evalpha_1, \evalpha_2)}^{(k)}(\Delta \vc_{\ell k}^N)^{\evalpha_1}(\Delta \mGamma_{\ell k}^N)^{\evalpha_2}(\Delta \mA_{\ell k}^N)^{\evalpha_3}(\Delta \vb_{\ell k}^N)^{\evalpha_4}(\Delta \mSigma_{\ell k}^N)^{\evalpha_5}}{\VD(G_N, G_0)}\\
		&= \dfrac{\sum_{\ell \in \sV_k}\sum_{\valpha \in \mathcal{J}_{l_1, l_2}^{(0)}(k)}\dfrac{\evpi_\ell^N}{2^{|\evalpha_2| + |\evalpha_5|}\valpha!}t_{\tau_0(\evalpha_1, \evalpha_2), \tau_0(\evalpha_1, \evalpha_2)}^{(k)}(\Delta \vc_{\ell k}^N)^{\evalpha_1}(\Delta \mGamma_{\ell k}^N)^{\evalpha_2}(\Delta \mA_{\ell k}^n)^{\evalpha_3}(\Delta \vb_{\ell k}^N)^{\evalpha_4}(\Delta \mSigma_{\ell k}^N)^{\evalpha_5}}{\VD(G_N, G_0)} .
	\end{align*}
	Then,
	\begin{equation}{\label{Denloss_step2_eq1}}
		\dfrac{\sum_{\ell \in \sV_k}\sum_{\substack{\tau_0(\evalpha_1, \evalpha_2) + \tau_1(\evalpha_3) = l_1 \\ \tau_2(\evalpha_3) + \evalpha_4 + \tau_3(\evalpha_5) = l_2}}\dfrac{\evpi_\ell^N}{2^{|\evalpha_2| + |\evalpha_5|}\valpha!}(\Delta \vc_{\ell k}^N)^{\evalpha_1}(\Delta \mGamma_{\ell k}^N)^{\evalpha_2}(\Delta \mA_{\ell k}^N)^{\evalpha_3}(\Delta \vb_{\ell k}^N)^{\evalpha_4}(\Delta \mSigma_{\ell k}^N)^{\evalpha_5}}{\VD(G_N, G_0)} \rightarrow 0,
	\end{equation}
	where the first equality come from the fact that $|\tau_0(\evalpha_1, \evalpha_2) + \evalpha_3| = |\evalpha_1| + 2|\evalpha_2| + |\evalpha_3| \leq 2\bar{r}(|\sV_k|) - |l_2|$. The second is a consequence of the previous reasoning and have the equality $\tau_0(\evalpha_1, \evalpha_2) + \evalpha_3 = l_1$. The last equations are obtained from the remark that $t_{\tau_0(\evalpha_1, \evalpha_2), \tau_0(\evalpha_1, \evalpha_2)} > 0,  \forall \evalpha_1, \evalpha_2$ given on assumption that $\mGamma_k^0$ is positive definite.\\ 
	Then, considering any $|l_1| = 2\bar{r}(|\sV_k|) - |l_2| - 1$.  we have $T_{l_1, l_2}^N(k) / \VD(G_N, G_0) \rightarrow 0,$ equivalent to 
	\begin{equation*}
		\dfrac{S_{0,l_1,l_2}^N(k) + S_{1,l_1,l_2}^N(k)}{\VD(G_N, G_0)} \rightarrow 0.
	\end{equation*}
	But we know that,
	\begin{align*}
		&\dfrac{S_{1,l_1,l_2}^N(k)}{\VD(G_N, G_0)} \\
		&= \dfrac{\sum_{\ell \in \sV_k}\sum_{\valpha \in \mathcal{J}_{l_1, l_2}^{(1)}(k)}\dfrac{\evpi_\ell^N}{2^{|\evalpha_2| + |\evalpha_5|}\valpha!}t_{l_1-\evalpha_3, \tau_0(\evalpha_1, \evalpha_2)}^{(k)}(\Delta \vc_{\ell k}^N)^{\evalpha_1}(\Delta \mGamma_{\ell k}^N)^{\evalpha_2}(\Delta \mA_{\ell k}^N)^{\evalpha_3}(\Delta \vb_{\ell k}^N)^{\evalpha_4}(\Delta \mSigma_{\ell k}^N)^{\evalpha_5}}{\VD(G_N, G_0)}\\
		&=\dfrac{\sum_{\ell \in \sV_k}\sum_{\substack{\tau_0(\evalpha_1, \evalpha_2) + \evalpha_3 > l_1 \\
					|\tau_0(\evalpha_1, \evalpha_2) + \evalpha_3| = |l_1| + 1
					\\
					|\evalpha_3| + \evalpha_4 + 2\evalpha_5 = l_2}}\dfrac{\evpi_\ell^N}{2^{|\evalpha_2| + |\evalpha_5|}\valpha!}t_{l_1-\evalpha_3, \tau_0(\evalpha_1, \evalpha_2)}^{(k)}(\Delta \vc_{\ell k}^N)^{\evalpha_1}(\Delta \mGamma_{\ell k}^N)^{\evalpha_2}(\Delta \mA_{\ell k}^N)^{\evalpha_3}(\Delta \vb_{\ell k}^N)^{\evalpha_4}(\Delta \mSigma_{\ell k}^N)^{\evalpha_5}}{\VD(G_N, G_0)}.
	\end{align*}
	However, as we know that for every $|l_1| = 2\bar{r}(|\sV_k|) - |l_2|$, from (\ref{Denloss_step2_eq1}), we obtain that
	\begin{equation*}
		\dfrac{S_{1,l_1,l_2}^N(k)}{\VD(G, G_0)} \rightarrow 0.
	\end{equation*}
	Hence, 
	\begin{align*}
		&\dfrac{S_{0,l_1,l_2}^N(k)}{\VD(G_N, G_0)}\\
		&=\dfrac{\sum_{\ell \in \sV_k}\sum_{\valpha \in \mathcal{J}_{l_1, l_2}^{(0)}(k)}\dfrac{\evpi_\ell^N}{2^{|\evalpha_2| + |\evalpha_5|}\valpha!}t_{\tau_0(\evalpha_1, \evalpha_2), \tau_0(\evalpha_1, \evalpha_2)}^{(k)}(\Delta \vc_{\ell k}^N)^{\evalpha_1}(\Delta \mGamma_{\ell k}^N)^{\evalpha_2}(\Delta \mA_{\ell k}^N)^{\evalpha_3}(\Delta \vb_{\ell k}^N)^{\evalpha_4}(\Delta \mSigma_{\ell k}^N)^{\evalpha_5}}{\VD(G_N, G_0)} \rightarrow 0.    
	\end{align*}
	Then,
	\begin{equation*}
		\dfrac{\sum_{\ell \in \sV_k}\sum_{\substack{\tau_0(\evalpha_1, \evalpha_2) + \evalpha_3 = l_1 \\ |\evalpha_3| + \evalpha_4 + 2\evalpha_5 = l_2}}\dfrac{\evpi_\ell^N}{2^{|\evalpha_2| + |\evalpha_5|}\valpha!}(\Delta \vc_{\ell k}^N)^{\evalpha_1}(\Delta \mGamma_{\ell k}^N)^{\evalpha_2}(\Delta \mA_{\ell k}^N)^{\evalpha_3}(\Delta \vb_{\ell k}^N)^{\evalpha_4}(\Delta \mSigma_{\ell k}^N)^{\evalpha_5}}{\VD(G, G_0)} \rightarrow 0,
	\end{equation*}
	follow from that fact that  $t_{\tau_0(\evalpha_1, \evalpha_2), \tau_0(\evalpha_1, \evalpha_2)} > 0,  \forall \evalpha_1, \evalpha_2$. We can continue to apply the same argument as above and complete the proof. And following by the previous argument, we can see that for every $l_1, l_2$, we obtain that
	\begin{equation}{\label{denloss_step2_eq2}}
		\dfrac{\sum_{\ell \in \sV_k}\sum_{\substack{\tau_0(\evalpha_1, \evalpha_2) + \evalpha_3 = l_1 \\ |\evalpha_3| + \evalpha_4 + 2\evalpha_5 = l_2}}\dfrac{\evpi_\ell^N}{2^{|\evalpha_2| + |\evalpha_5|}\valpha!}(\Delta \vc_{\ell k}^n)^{\evalpha_1}(\Delta \mGamma_{\ell k}^n)^{\evalpha_2}(\Delta \mA_{\ell k}^n)^{\evalpha_3}(\Delta \vb_{\ell k}^n)^{\evalpha_4}(\Delta \mSigma_{\ell k}^n)^{\evalpha_5}}{\VD(G_N, G_0)} \rightarrow 0.
	\end{equation}
	Now consider all $|l_1| = 1, |l_2| = 0$, then form (\ref{denloss_step2_eq2}) taking the sum all of it, we obtain the following. the following
	\begin{equation*}
		\dfrac{\sum_{u = 1}^{D}|\sum_{\ell \in \sV_k}\evpi_\ell^N (\Delta\vc_{\ell k}^n)^{(u)}|}{\VD(G_N, G_0)} \rightarrow 0,
	\end{equation*}
	then
	\begin{equation*}
		\dfrac{\lVert \sum_{\ell \in \sV_k}\evpi_\ell^N(\Delta\vc_{\ell k})\rVert}{\VD(G_N, G_0)} \rightarrow 0.
	\end{equation*}
	Similarly, for all $|l_1| = 2, |l_2| = 0$, we get that
	\begin{equation*}
		\dfrac{\lVert \sum_{\ell \in \sV_k}\evpi_\ell^N((\Delta\vc_{\ell k})(\Delta \vc_{\ell k})^\top + \Delta \mGamma_{\ell k}^n) \rVert}{\VD(G_N, G_0)} \rightarrow 0.
	\end{equation*}
	For all $|l_1| = 0, |l_2| = 1$, 
	\begin{equation*}
		\dfrac{\lVert \sum_{\ell \in \sV_k}\evpi_\ell^N(\Delta\vb_{\ell k})\rVert}{\VD(G_N, G_0)} \rightarrow 0.
	\end{equation*}
	For all $|l_1| = 0, |l_2| = 2$,
	\begin{equation*}
		\dfrac{\lVert \sum_{\ell \in \sV_k}\evpi_\ell^N((\Delta\vb_{\ell k})(\Delta \vb_{\ell k})^\top + (\Delta \mSigma_{\ell k}))\rVert}{\VD(G_N, G_0)} \rightarrow 0.
	\end{equation*}
	And for all $|l_1| = 1, |l_2| = 1$, 
	\begin{equation*}
		\dfrac{\lVert \sum_{\ell \in \sV_k}\evpi_\ell^N((\Delta\vb_{\ell k})(\Delta \vc_{\ell k})^\top + (\Delta \mA_{\ell k}))\rVert}{\VD(G_N, G_0)} \rightarrow 0.
	\end{equation*}
	%%%%%%%%%%%%%%%%%%%%%%%%%%%%%%%%%%%%%%%%%%%%%%%%%%%%%%%%%%%%%%%%%%%%%%%%%%%%%%%%%%%%%%%%%%%%%%%%%%%%%%%%%%%%%%%%%%%%%%%%%%%%%%%%%%%%
	As a result, with above results and definition of $\VD(G, G_0)$, we obtain that 
	\begin{align*}
		\dfrac{1}{\VD(G_N, G_0)}\sum_{k : |\sV_k| > 1}\sum_{\ell\in \sV_k} &\pi_\ell(\norm{\vc_\ell^N - \vc_k^0}^{\overline{r}(|\sV_k|)} + \norm{\mGamma_\ell^N - \mGamma_k^0}^{\overline{r}(|\sV_k|)/2}  + \norm{\mA_\ell^N - \mA_k^0}^{\overline{r}(|\sV_k|)/2}  \\
		&+  \norm{\vb_\ell^N - \vb_k^0}^{\overline{r}(|\sV_k|)} + \norm{\mSigma_\ell^N - \mSigma_k^0}^{\overline{r}(|\sV_k|)/2}) \rightarrow 1.
	\end{align*}
	As a result, we can find index $k^* \in [K_0]$ and $|V_{k^*}| > 1$ such that 
	\begin{equation*}
		\dfrac{\sum_{\ell\in V_{k^*}} \pi_\ell(\norm{\Delta \vc_{\ell k^*}^N}^{\overline{r}(|V_{k^*}|)} + \norm{\Delta \mGamma_{\ell k^*}^N}^{\overline{r}(|V_{k^*}|)/2}  + \norm{\Delta \mA_{\ell k^*}^N}^{\overline{r}(|V_{k^*}|)/2} +  \norm{\Delta \vb_{\ell k^*}^N}^{\overline{r}(|V_{k^*}|)} + \norm{\Delta \mSigma_{\ell k^*}^N}^{\overline{r}(|V_{k^*}|)/2})}{\VD(G_N, G_0)} \not\rightarrow 0.
	\end{equation*}
	Without loss of generality, we may assume that $k^* = 1$. Divide into the following cases: \\
	%%%%%%%%%%%%%%%%%%%%%%%%%%%%%%%%%%%%%%%%%%%%%%%%%%%%%%%%%%%%%%%%%%%%%%%%%%%%%%%%%%%%%%%%%%%%%%%%%%%%%%%%%%%%%%%%%%%%%%%%%%%%%%%%%%%%%%%%%%%%%%%%%%%%%%%%%%%%%%%%%%%%%%%%%%%%%%%%
	\textbf{Case 1:} $\dfrac{1}{\VD(G_N, G_0)}\sum_{\ell \in \sV_1}\evpi_\ell^N (\lVert \Delta\vc_{\ell 1}^N\rVert^{\bar{r}(|\sV_1|)} + \lVert \Delta \mGamma_{\ell 1}^N\rVert^{\bar{r}(|\sV_1|)/2} )\not \rightarrow 0.$\\
	Now, from equation (\ref{denloss_step2_eq2}), consider cases with $l_2 = 0$, 
	\begin{equation*}
		\dfrac{\sum_{\ell \in \sV_k}\sum_{\substack{\tau_0(\evalpha_1, \evalpha_2) = l_1}}\dfrac{\evpi_\ell^N}{2^{|\evalpha_2| + |\evalpha_5|}\valpha!}(\Delta \vc_{\ell k}^N)^{\evalpha_1}(\Delta \mGamma_{\ell k}^N)^{\evalpha_2}}{\VD(G_N, G_0)} \rightarrow 0.
	\end{equation*}
	Applying the procedure as proof of Lemma 10 in \cite{manole_refined_2022}, we obtain the contradiction.\\
	\textbf{Case 2:} $\dfrac{1}{\VD(G_N, G_0)}\sum_{\ell \in \sV_1}\evpi_\ell^N (\lVert \Delta\vb_{\ell 1}^N\rVert^{\bar{r}(|\sV_1|)} + \lVert \Delta \mSigma_{\ell 1}^N\rVert^{\bar{r}(|\sV_1|)/2} )\not \rightarrow 0.$\\
	Similarly, applying (\ref{denloss_step2_eq2}) with all cases $l_1 = 0$, we have
	\begin{equation*}
		\dfrac{\sum_{\ell \in \sV_k}\sum_{\substack{ \\   \evalpha_4+ 2 \evalpha_5 = l_2}}\dfrac{\evpi_\ell^N}{2^{|\evalpha_2| + |\evalpha_5|}\valpha!}(\Delta \vb_{\ell k}^N)^{\evalpha_4}(\Delta \mSigma_{\ell k}^N)^{\evalpha_5}}{\VD(G_N, G_0)} \rightarrow 0.
	\end{equation*}
	Applying the same procedure in case 1, we obtain the contradiction.\\
	\textbf{Case 3: }$\dfrac{1}{\VD(G_N, G_0)}\sum_{\ell \in \sV_1}\evpi_\ell^N  \lVert \Delta \mA_{\ell 1}^N\rVert^{\bar{r}(|\sV_1|)/2} \not \rightarrow 0.$ \\
	WLOG, we may assume that for some index $u$, say $u = 1$, we get
	\begin{equation*}
		\dfrac{1}{\VD(G_N, G_0)}\sum_{\ell \in \sV_1}\evpi_\ell^N  \lVert (\Delta \mA_{\ell 1}^N)^{(1)}\rVert^{\bar{r}(|\sV_1|)/2} \not \rightarrow 0.
	\end{equation*}
	Clearly, we obtain that
	\begin{align*}
		\dfrac{1}{\VD(G_N, G_0)}\sum_{\ell \in \sV_1}\evpi_\ell^N&\left( \lVert (\Delta \vc_{\ell 1}^N)^{(1)}\rVert^{\bar{r}(|\sV_1|)} + \lVert (\Delta \mGamma_{\ell 1}^N)^{(1,1)}\rVert^{\bar{r}(|\sV_1|)/2}+\lVert (\Delta \mA_{\ell 1}^N)^{(1)}\rVert^{\bar{r}(|\sV_1|)/2} \right)\\
		&+\left \lVert (\Delta \vb_{\ell 1}^N)\rVert^{\bar{r}(|\sV_1|)} + \lVert (\Delta \mSigma_{\ell 1}^N)\rVert^{\bar{r}(|\sV_1|)/2}\right) \not \rightarrow 0.
	\end{align*}
	From (\ref{denloss_step2_eq2}), consider all $l_1, l_2$ such that $l_1^{(u)} = 0$ for all $u = 2,\dots, D$.  We get that, for brevity that we denote $\Delta \vc_{\ell 1}^N := (\Delta\vc_{\ell 1}^N)^{(1)}$, $\Delta \mGamma_{\ell 1}^N := (\Delta \mGamma_{\ell 1}^N)^{(1,1)}$, $\Delta \mA_{\ell 1}^N := (\Delta\mA_{\ell 1}^N)^{(1)}$
	\begin{equation*}
		\dfrac{\sum_{\ell \in \sV_1}\sum_{\substack{ \evalpha_1^{(1) } + 2\evalpha_2^{(1,1)} + \evalpha_3^{(1)} = l_1^{(1)} \\  \evalpha_4 + 2\evalpha_5+ \evalpha_3^{(1)} = l_2}}\dfrac{\evpi_\ell^N}{2^{\evalpha_2^{(1,1)}\evalpha_5 }\valpha!}(\Delta \vc_{\ell 1}^N)^{\evalpha_1^{(1)}}(\Delta \mGamma_{\ell 1}^N)^{\evalpha_2^{(1,1)}}(\Delta \mA_{\ell 1}^N)^{\evalpha_3^{(1)}}(\Delta \vb_{\ell 1}^N)^{\evalpha_4}(\Delta \mSigma_{\ell 1}^N)^{\evalpha_5}}{\VD(G_N, G_0)} \rightarrow 0.
	\end{equation*}
	Applying the same method in step 2 of proof in Theorem 2 of \cite{nguyen_towards_2024}, we obtain the contradiction.
	
	%%%%%%%%%%%%%%%%%%%%%%%%%%%%%%%%%%%%%%%%%%%%%%%%%%%%%%%%%%%%%%%%%%%%%%%%%%%%%%%%%%%%%%%%%%%%%%%%%%%%%%%%%%%%%%%%%%%%%%%%%%%%%%%%%%%%
	\textbf{Step 3 - Application of Fatou's lemma: }Set $m_N := \max_{l_1, l_2, k} \dfrac{|T_{l_1, l_2}(k)|}{\VD(G_N, G_0)} \not \rightarrow 0$ as $N\rightarrow \infty$. Because $\dfrac{|T_{l_1, l_2}(k) / \VD(G_N, G_0)|}{m_N} \leq 1$, by extracting a subsequence if needed, we have $\dfrac{|T_{l_1, l_2}(k) / \VD(G_N, G_0)|}{m_N} \rightarrow \xi_{l_1, l_2}(k)$ for all $(l_1, l_2, k) $ and due to the finiteness of the possible values of $(l_1, l_2, k),$ there is at least one $\xi_{l_1, l_2}(k) = 1$. Applying Fatou's lemma, we get that
	\begin{align*}
		0 = \lim_{N \rightarrow \infty}\dfrac{1}{m_n}2\dfrac{\TV(p_{G_N}, p_{G_0})}{\VD(G_N, G_0)} &\geq \int \liminf_{N \rightarrow \infty}\dfrac{1}{m_n}\dfrac{|p_{G_N}(\vy,\vx) - p_{G_0}(\vy,\vx)|}{\VD(G_N, G_0)}d(\vy,\vx) \\
		&= \int \left| \sum_{(l_1, l_2, k) }\xi_{(l_1, l_2)}(k)\vx^{l_1}\dfrac{\partial ^{|l_2|}}{\partial h_1^{l_2}}f_{\cD}(\vy| (\mA_k^0) ^{\top}\vx + \vb_k^0, \mSigma_k^0)f_{\cL}(\vx| \vc_k^0, \mGamma_k^0)\right|d(\vy,\vx).
	\end{align*}
	As a consequence, we obtain that 
	\begin{equation*}
		\sum_{(l_1, l_2, k) }\xi_{(l_1, l_2)}(k)\vx^{l_1}\dfrac{\partial ^{|l_2|}}{\partial h_1^{l_2}}f_{\cD}(\vy| (\mA_k^0)^{\top} \vx + \vb_k^0, \mSigma_k^0)f_{\cL}(\vx| \vc_k^0, \mGamma_k^0) = 0.
	\end{equation*}
	almost surely $(\vy, \vx)$. Since elements of the set $\cF$ are linearly independent (the proof of this claim is deferred to the end of this proof), the above equations implies that $\xi_{l_1, l_2}(k) = 0$, but this contradicts the claim that there at least one of them is 1. Hence, we have a contradiction and complete proof.\\
	\textbf{Linear independence of elements in $\cF$}: Assume that there exist real numbers $\xi_{l_1, l_2}(k)$ such that the following holds for almost surely $(\vy, \vx):$ 
	\begin{equation*}
		\sum_{(l_1, l_2, k) \in \cS}\xi_{l_1, l_2}(k)\vx^{l_1}\dfrac{\partial ^{|l_2|}}{\partial h_1^{l_2}}f_{\cD}(\vy| \mA_k^0\vx + \vb_k^0, \mSigma_k^0)f_{\cL}(\vx| \vc_k^0, \mGamma_k^0) = 0,
	\end{equation*}
	with $\cS := \lbrace  k \in [K_0], 0\leq |l_1| + |l_2| \leq 2\bar{r}(|\sV_k|) \rbrace.$
	Rewriting as follows:
	\begin{equation*}
		\sum_{k=1}^{K_0}\sum_{l_2 = 0}^{2\bar{r}(|\sV_k|)}\left ( \sum_{l_1 = 0}^{2\bar{r}(|\sV_k|) - l_2 } \xi_{l_1, l_2}(k)\vx^{l_1}f_{\cL}(\vx| \vc_k^0, \mGamma_k^0)  \right)\dfrac{\partial ^{|l_2|}}{\partial h_1^{l_2}}f_{\cD}(\vy| (\mA_k^0)^{\top}\vx + \vb_k^0, \mSigma_k^0) = 0.
	\end{equation*}
	By identifiability of the location Gaussian mixture, we obtain that for almost surely $\vx$:
	\begin{equation*}
		\sum_{l_1 = 0}^{2\bar{r}(|\sV_k|) - l_2} \xi_{l_1, l_2}(k) \vx^{l_1}f_{\cL}(\vx| \vc_k^0, \mGamma_k^0)   = 0,
	\end{equation*}
	for all $k \in [K_0]$ and $0 \leq |l_2 | \leq 2\bar{r}(|\sV_k|)$, With the fact that $f_{\cL}(\vx| \vc_k^0, \mGamma_k^0)  $ is greater than  0 for all $\vx$, so $\xi_{l_1, l_2}(k) = 0$ (property of polynomial function).\\
	%%%%%%%%%%%%%%%%%%%%%%%%%%%%%%%%%%%%%%%%%%%%%%%%%%%%%%%
	\subsection{Proof of Theorem \ref{theorem_inequality_DK_DK0}}\label{subsec:proof_theorem_inequality_DK_DK0}
	{\bf Proof Sketch.} We prove by induction, so we only need justify the case $\VD(G^{(\kappa)}), G_0) \gtrsim \VD(G^{(\kappa-1)}, G_0)$. As $\VD(G^{(\kappa)}, G_0) \rightarrow 0, $we can extract a sequence satisfy $\vtheta_\ell^N \to \vtheta_0$ for all $\ell \in \sV_k$. By assumption that $\evpi_\ell^N$ is bounded below and merging of two atoms with minimum dissimilarity, we get that two merged atoms must belong to same Voronoi cell, WLOG say atoms 1 and 2 in Voronoi cell $\sV_1$. We only need to prove the two kind of inequality, for instance $\left \lVert \sum_{\ell \in \sV_1} \pi_\ell (\vc_\ell - \vc_1^0) \right\rVert $ = $\left \lVert \sum_{\ell \sV_1, \ell \not \in \lbrace1, 2\rbrace}\pi_\ell (\vc_\ell - \vc_1^0) + \pi_*(\vc_* - \vc_1^0) \right \rVert$ and $\pi_1 \lVert \vc_1 - \vc_1^0 \rVert^{\bar{r}} + \pi_2\lVert\vc_2 - \vc_1^0 \rVert^{\bar{r}} \gtrsim \pi_*\lVert\vc_* - \vc_1\rVert^{\bar{r}}$. Where the first equality is followed by definition of procedure to merge atoms, the second is from convex inequality.
	
	Here, we abuse the notation of Dendrogram loss, with the difference that we always fix the value of $\bar{r}(|\sV_k|)$ for all $\VD(G^{(\kappa)}, G_0)$. It suffices to prove the inequality $\VD(G^{(K)}, G_0) \gtrsim\VD(G^{(K-1)}, G_0)$ and the rest are similar. Suppose $G_N = \sum_{k=1}^{K}\pi_k^N\delta_{(\vc_k^N, \mGamma_k^N,\mA_k^N, \vb_k^N, \mSigma_k^N)} \in \cE_{K, c_0}$ varies so that 
	\begin{align*}
		\VD(G, G_0) & = \sum_{k : |\sV_k| = 1}\sum_{\ell \in \sV_k}\pi_\ell(\norm{\vc_\ell - \vc_k^0} + \norm{\mGamma_\ell - \mGamma_k^0} + \norm{\mA_\ell - \mA_k^0} + \norm{\vb_\ell - \vb_k^0} + \norm{\mSigma_\ell - \mSigma_k^0})\\
		&+\sum_{k : |\sV_k| > 1}\sum_{\ell\in \sV_k} \pi_\ell\left(\norm{\vc_\ell - \vc_k^0}^{\overline{r}(|\sV_k|)} 
		+ \norm{\mGamma_\ell - \mGamma_k^0}^{\overline{r}(|\sV_k|)/2}  + \norm{\mA_\ell - \mA_k^0}^{\overline{r}(|\sV_k|)/2}  \right) \\
		&+ \left   \norm{\vb_\ell - \vb_k^0}^{\overline{r}(|\sV_k|)} + \norm{\mSigma_\ell - \mSigma_k^0}^{\overline{r}(|\sV_k|)/2}  \right)\\
		&+
		\sum_{k=1}^{K_0} \left|\sum_{\ell\in V_{k}} \pi_\ell - \pi_k^0 \right| \\
		&+ \sum_{k : |\sV_k| > 1}\left( \norm{\sum_{\ell\in \sV_k}\pi_\ell (\vc_\ell - \vc_k^0)} +\norm{\sum_{\ell\in \sV_k}\pi_\ell ((\vc_\ell - \vc_k^0) (\vc_\ell - \vc_k^0)^{\top} + \mGamma_\ell - \mGamma_k^0)} \right) \\
		&\left( +  \lVert\sum_{\ell\in \sV_k}\pi_\ell ((\vb_\ell - \vb_k^0)(\vc_\ell - \vc_k^0)^{\top}+ \mA_\ell - \mA_k^0)\rVert \right) \\
		&\left( +   \norm{\sum_{\ell\in \sV_k}\pi_\ell (\vb_\ell - \vb_k^0)}   + \norm{\sum_{\ell\in \sV_k}\pi_\ell ((\vb_\ell - \vb_k^0) (\vb_\ell - \vb_k^0)^{\top} + \mSigma_\ell - \mSigma_k^0))}\right).
	\end{align*}
	Because all $\pi_l^N$ are bounded below by $c_0$, then $\vc_\ell \rightarrow \vc_k^0$,$\mGamma_\ell \rightarrow \mGamma_k^0$ ,  $\mA_\ell \rightarrow \mA_k^0 $, $\vb_\ell \rightarrow \vb_k^0$ and $\mSigma_\ell \rightarrow \mSigma_k^0$ for all $\ell \in \sV_k$, $ k \in [K_0].$ By definition of dissimilarity and merging two atoms with minimum dissimilarity, we get that two merged atoms must belong to same Voronoi cell for large enough $N$. WLOG, assuming that Voronoi cell is $\sV_1$ and two atoms merged belong that cell are 1 and 2. Let the merge atom be $\pi_*\delta_{(\vc_*, \mGamma_*, \mA_*, \vb_*, \mSigma_*)}$, i.e,
	\begin{equation*}
		\pi_{*} = \pi_{1} + \pi_{2},\quad  \vc_{*} = \dfrac{\pi_{1}}{\pi_{*}} \vc_{1} + \dfrac{\pi_{2}}{\pi_{*}} \vc_{2},\quad  \vb_{*} = \dfrac{\pi_{1}}{\pi_{*}} \vb_{1} + \dfrac{\pi_{2}}{\pi_{*}} \vb_{2},
	\end{equation*}
	\begin{align*}
		\mGamma_* &= \dfrac{\pi_{1}}{\pi_{*}} \left(\mGamma_{1} + (\vc_{1} - \vc_*)(\vc_{1} - \vc_*)^{\top} \right) + \dfrac{\pi_{2}}{\pi_{*}} \left(\mGamma_{2} + (\vc_{2} - \vc_*)(\vc_{2} - \vc_*)^{\top} \right)\\
		&= \dfrac{\pi_1}{\pi_*}\mGamma_1 + \dfrac{\pi_2}{\pi_*}\mGamma_2 + \dfrac{\pi_1 \pi_2}{\pi_*^2}(\vc_1- \vc_2)(\vc_1-\vc_2)^{\top} ,
	\end{align*}
	\begin{align*}
		\mA_* &= \dfrac{\pi_{1}}{\pi_{*}} \left(\mA_{1} + (\vb_{1} - \vb_*)(\vc_{1} - \vc_*)^{\top} \right) + \dfrac{\pi_{2}}{\pi_{*}} \left(\mA_{2} + (\vb_{2} - \vb_*)(\vc_{2} - \vc_*)^{\top} \right)\\
		&= \dfrac{\pi_1}{\pi_*}\mA_1 + \dfrac{\pi_2}{\pi_*}\mA_2 + \dfrac{\pi_1 \pi_2}{\pi_*^2}(\vb_1- \vb_2)(\vc_1-\vc_2)^{\top} ,
	\end{align*}
	and
	\begin{align}
		\mSigma_* &= \dfrac{\pi_{1}}{\pi_{*}} \left(\mSigma_{1} + (\vb_{1} - \vb_*)^2 \right) + \dfrac{\pi_{2}}{\pi_{*}} \left(\mSigma_{2} + (\vb_{2} - \vb_*)^2 \right)\\
		&= \dfrac{\pi_1}{\pi_*}\mSigma_1 + \dfrac{\pi_2}{\pi_*}\mSigma_2 + \dfrac{\pi_1 \pi_2}{\pi_*^2}(\vb_1- \vb_2)(\vb_1-\vb_2)\label{lem14_sigma}.
	\end{align}
	Moreover, we have that for large $N$ enough, the merged atom must be in the same Voronoi cell $\sV_1$. Hence,
	\begin{equation*}
		\left \lvert \sum_{\ell \in \sV_1} \pi_{\ell} -\pi_1^0 \right \rvert =  \left \lvert \sum_{\ell \in \sV_1, 
			\ell \not \in \lbrace 1,2\rbrace} \pi_{\ell} + \pi_* -\pi_1^0 \right \rvert,
	\end{equation*}
	\begin{equation*}
		\left \lVert \sum_{\ell \in \sV_1} \pi_{\ell}(\vc_\ell - \vc_1^0) \right \rVert = \left \lVert \sum_{\ell \in \sV_1, \ell \not \in \lbrace 1,2\rbrace} \pi_{\ell}(\vc_\ell - \vc_1^0) + \pi_*(\vc_* - \vc_1^0) \right \rVert,
	\end{equation*}
	\begin{equation*}
		\left \lVert \sum_{\ell \in \sV_1} \pi_{\ell}(\vb_\ell - \vb_1^0) \right \rVert = \left \lVert \sum_{\ell \in \sV_1, \ell \not \in \lbrace 1,2\rbrace} \pi_{\ell}(\vb_\ell - \vb_1^0) + \pi_*(\vb_* - \vb_1^0) \right \rVert,
	\end{equation*}
	\begin{align*}
		&\left \lVert \sum_{\ell\in \sV_1}\pi_\ell ((\vc_\ell - \vc_1^0) (\vc_\ell - \vc_1^0)^{\top} + \mGamma_\ell - \mGamma_1^0) \right \rVert\\
		&= \left \lVert \sum_{\ell \in \sV_1, \ell \not \in \lbrace 1,2\rbrace} \pi_\ell ((\vc_\ell - \vc_1^0) (\vc_\ell - \vc_1^0)^{\top} + \mGamma_\ell - \mGamma_1^0) + \pi_* ((\vc_* - \vc_1^0)(\vc_*- \vc_1^0)^{\top} + \mGamma_* - \mGamma_1^0) \right \rVert ,
	\end{align*}
	\begin{align*}
		&\left \lVert \sum_{\ell\in \sV_1}\pi_\ell ((\vb_\ell - \vb_1^0) (\vc_\ell - \vc_1^0)^{\top} + \mA_\ell - \mA_1^0) \right \rVert\\
		&= \left \lVert \sum_{\ell \in \sV_1, \ell \not \in \lbrace 1,2\rbrace} \pi_\ell ((\vb_\ell - \vb_1^0) (\vc_\ell - \vc_1^0)^{\top} + \mA_\ell - \mA_1^0) + \pi_* ((\vb_* - \vb_1^0)(\vc_*- \vc_1^0)^{\top} + \mA_* - \mA_1^0) \right \rVert ,
	\end{align*}
	and 
	\begin{align*}
		&\left \lVert \sum_{\ell\in \sV_1}\pi_\ell ((\vb_\ell - \vb_1^0)^2 + \mSigma_\ell - \mSigma_1^0) \right \rVert\\
		&= \left \lVert \sum_{\ell \in \sV_1, \ell \not \in \lbrace 1,2\rbrace} \pi_\ell ((\vb_\ell - \vb_1^0)^2 + \mSigma_\ell - \mSigma_1^0) + \pi_* ((\vb_* - \vb_1^0)^2 + \mSigma_* - \mSigma_1^0) \right \rVert .
	\end{align*}
	Due to $|\sV_1| \geq 2$, then $\bar{r}(|\sV_1|) \geq 4$, it suffices to show that for any $\bar{r} \geq 4,$
	\begin{align*}
		&\pi_1(\lVert \vc_1 - \vc_1^0\lVert^{\bar{r}}+ \lVert \mGamma_1 - \mGamma_1^0\rVert^{\bar{r}/2} + \lVert \mA_1 - \mA_1^0\rVert^{\bar{r}/2} + \lVert \vb_1 - \vb_1^0\lVert^{\bar{r}} + \lVert \mSigma_1 - \mSigma_1^0 \lVert^{\bar{r}/2})\\
		&+ \pi_2(\lVert \vc_2 - \vc_2^0\lVert^{\bar{r}}+ \lVert \mGamma_2 - \mGamma_2^0\rVert^{\bar{r}/2} +\lVert \mA_2 - \mA_1^0\rVert^{\bar{r}/2} + \lVert \vb_2 - \vb_1^0\lVert^{\bar{r}} + \lVert \mSigma_2 - \mSigma_1^0 \lVert^{\bar{r}/2}) \\
		&\gtrsim \pi_*(\lVert \vc_* - \vc_1^0\lVert^{\bar{r}}+ \lVert \mGamma_* - \mGamma_1^0\rVert^{\bar{r}/2}+ \lVert \mA_* - \mA_1^0\rVert^{\bar{r}/2} + \lVert \vb_* - \vb_1^0\lVert^{\bar{r}} + \lVert \mSigma_* - \mSigma_1^0 \lVert^{\bar{r}/2}).
	\end{align*}
	We know that, $\lVert \cdot \lVert^r$ is a convex function for any $r > 1$, and apply convex inequality,we have
	\begin{equation*}
		\lVert \vc_* - \vc_1^0\lVert^{\bar{r}}  \leq \dfrac{\pi_1}{\pi_*} \lVert \vc_1 - \vc_1^0 \rVert^{\bar{r}}  + \dfrac{\pi_2}{\pi_*}\lVert \vc_2 - \vc_1^0\rVert^{\bar{r}} ,
	\end{equation*}
	\begin{equation*}
		\lVert \vb_* - \vb_1^0\lVert^{\bar{r}} \leq \dfrac{\pi_1}{\pi_*} \lVert \vb_1 - \vb_1^0 \rVert^{\bar{r}} + \dfrac{\pi_2}{\pi_*}\lVert \vb_2 - \vb_1^0\rVert^{\bar{r}},
	\end{equation*}
	and with identity in $(\ref{lem14_sigma})$
	\begin{align*}
		\pi_*\lVert \mSigma_* - \mSigma_1^0 \rVert^{\bar{r}/2} &\lesssim \pi_*(\lVert \dfrac{\pi_1}{\pi_*}\mSigma_1 + \dfrac{\pi_2}{\pi_*}\mSigma_2 - \mSigma_1^0 \rVert^{\bar{r}/2} + \lVert \dfrac{\pi_1 \pi_2}{\pi_*^2}(\vb_1- \vb_2)(\vb_1-\vb_2)^{\top} \rVert^{\bar{r}/2})\\
		&\leq \pi_1\lVert \mSigma_1 - \mSigma_1^0\rVert^{\bar{r}/2} +  \pi_2\lVert \mSigma_2 - \mSigma_1^0\rVert^{\bar{r}/2} + \pi_*(\dfrac{\pi_1 \pi_2}{\pi_*^2})^{\bar{r}/2}\lVert \vb_1 - \vb_2 \rVert^{\bar{r}}.
	\end{align*}
	The last term in the above forms can be bounded as follows:
	\begin{align*}
		\pi_1 \lVert \vb_1 - \vb_1^0 \rVert^{\bar{r}} + \pi_2\lVert \vb_2 - \vb_1^0 \rVert^{\bar{r}} &\geq \min\lbrace \pi_1, \pi_2 \rbrace ( \lVert \vb_1 - \vb_1^0 \rVert^{\bar{r}} +\lvert \vb_2 - \vb_1^0 \rVert^{\bar{r}})\\
		&\gtrsim \min\lbrace \pi_1, \pi_2 \rbrace \lVert \vb_1 - \vb_2 \rVert^{\bar{r}}\\
		&\geq \dfrac{\pi_1 \pi_2}{\pi_*}\lVert \vb_1 - \vb_2 \rVert^{\bar{r}} = \pi_*\dfrac{\pi_1 \pi_2}{\pi_*^2}\lVert \vb_1 - \vb_2 \rVert^{\bar{r}}\\
		&\geq \pi_*(\dfrac{\pi_1 \pi_2}{\pi_*^2})^{\bar{r}/2}\lVert \vb_1 - \vb_2 \rVert^{\bar{r}},
	\end{align*}
	where the last inequality follows from the fact that $\dfrac{\pi_1 \pi_2}{\pi_*^2} < 1$ and $\bar{r}/2 \geq 1.$\\
	Similarly, we obtain
	\begin{equation*}
		\pi_*\lVert \mGamma_* - \mGamma^0 \rVert^{\bar{r}/2} \lesssim \pi_1\lVert \mGamma - \mGamma^0\rVert^{\bar{r}/2} +  \pi_2\lVert \mGamma - \mGamma^0\rVert^{\bar{r}/2} + \pi_1 \lVert \vc_1 - \vc_1^0 \rVert^{\bar{r}} + \pi_2\lVert \vc_2 - \vc_1^0 \rVert^{\bar{r}}.
	\end{equation*}
	Now, consider the following case, we have 
	\begin{align*}
		\pi_*\lVert \mA_* - \mA_1^0 \rVert^{\bar{r}/2} &\lesssim \pi_*(\lVert \dfrac{\pi_1}{\pi_*}\mA_1 + \dfrac{\pi_2}{\pi_*}\mA_2 - \mA_1^0 \rVert^{\bar{r}/2} + \lVert \dfrac{\pi_1 \pi_2}{\pi_*^2}(\vb_1- \vb_2)(\vc_1-\vc_2)^{\top} \rVert^{\bar{r}/2})\\
		&\leq \pi_1\lVert \mA_1 - \mA_1^0\rVert^{\bar{r}/2} +  \pi_2\lVert \mA_2 - \mA_1^0\rVert^{\bar{r}/2} + \pi_*(\dfrac{\pi_1 \pi_2}{\pi_*^2})^{\bar{r}/2}\lVert \vb_1 - \vb_2 \rVert^{\bar{r}/2}\lVert \vc_1 - \vc_2 \rVert^{\bar{r}/2}.
	\end{align*}
	The last term in the above equation can be bounded as follow:
	\begin{align*}
		&\pi_1 \lVert \vb_1 - \vb_1^0 \rVert^{\bar{r}} + \pi_2\lVert \vb_2 - \vb_1^0 \rVert^{\bar{r}} + \pi_1 \lVert \vc_1 - \vc_1^0 \rVert^{\bar{r}} + \pi_2\lVert \vc_2 - \vc_1^0 \rVert^{\bar{r}}\\
		&\geq \min\lbrace \pi_1, \pi_2 \rbrace \max\lbrace \lVert \vb_1 - \vb_1^0 \rVert^{\bar{r}} +\lVert \vb_2 - \vb_1^0 \rVert^{\bar{r}}, \lVert \vc_1 - \vc_1^0 \rVert^{\bar{r}} +\lVert \vc_2 - \vc_1^0 \rVert^{\bar{r}} \rbrace\\
		&\gtrsim \min\lbrace \pi_1, \pi_2 \rbrace \max \lbrace \lVert \vb_1 - \vb_2 \rVert^{\bar{r}}, \lVert \vc_1 - \vc_2 \rVert^{\bar{r}} \rbrace\\
		&\geq \min\lbrace \pi_1, \pi_2 \rbrace \lVert \vb_2 - \vb_1^0 \rVert^{\bar{r}/2}\lVert \vc_2 - \vc_1^0 \rVert^{\bar{r}/2} \\
		&\geq \dfrac{\pi_1 \pi_2}{\pi_*}\lVert \vb_2 - \vb_1^0 \rVert^{\bar{r}/2}\lVert \vc_2 - \vc_1^0 \rVert^{\bar{r}/2}  = \pi_*\dfrac{\pi_1 \pi_2}{\pi_*^2}\lVert \vb_2 - \vb_1^0 \rVert^{\bar{r}/2}\lVert \vc_2 - \vc_1^0 \rVert^{\bar{r}/2} \\
		&\geq \pi_*(\dfrac{\pi_1 \pi_2}{\pi_*^2})^{\bar{r}/2}\lVert \vb_2 - \vb_1^0 \rVert^{\bar{r}/2}\lVert \vc_2 - \vc_1^0 \rVert^{\bar{r}/2}. 
	\end{align*}
	
	\subsection{Proof of Theorem \ref{theorem_rate_DIC_GLLiM}}\label{sec_proof_theorem_rate_DIC_GLLiM}
	Before we go into details of the proof, we introduce the Wasserstein distances~\cite{villani_topics_2003} to measure the difference between two measures. For two mixing measure $G = \sum_{k=1}^K\evp_k\delta_{\evtheta_k}$ and $G' = \sum_{l=1}^{K'}\evp'_l\delta_{\evtheta'_l}$, the Wasserstein-r distance (for $r \geq 1$) between $G$ and $G'$ is defined as 
	\begin{equation*}
		W_r(G, G') := \left(\inf_{\mQ \in \Pi(\vp, \vp')}\sum_{k,l=1}^{K,K'}\emQ_{kl}\norm{\evtheta_k - \evtheta'_l}^r \right)^{1/r},
	\end{equation*}
	where $\Pi(\vp, \vp')$ is the set of all couplings between $\vp = (\evp_1, \dots, \evp_K)$ and $\vp' = (\evp_1', \dots, \evp_{K'}')$, i.e, $\Pi(\vp, \vp') = \lbrace \mQ \in \mathbb{R}_{+}^{K\times K'}: \sum_{k=1}^K \emQ_{kl} = \evp'_l, \sum_{l=1}^{K'} \emQ_{kl} = \evp_k, \forall k\in [K], l \in [K'] \rbrace$.\\
	The following fact is important, it describes the relation between the convergence in Wasserstein distances of mixing measures and the convergence of components~\cite{ho_singularity_2019}. Fix $G_0 = \sum_{k=1}^{K_0}\pi_k^0 \delta_{\theta_k^0} \in \mathcal{E}_{K_0},$ and consider $G =\sum_{\ell = 1}^{K}\pi_\ell \delta_{\theta_\ell} $ such that $W_r(G, G_0) \to 0, $ we have
	\begin{equation*}
		W_r^r(G,G_0) \asymp \sum_{k=1}^{K_0}\left( \left|\sum_{\ell \in \sV_k(G)}\pi_\ell - \pi_k^0 \right|+\sum_{\ell \in \sV_k}\pi_\ell \norm{\theta_\ell - \theta_k^0}^r \right).
	\end{equation*}
	Now, we delve into detail of the proof. From \cref{fact_density_rate}, there exists a constant c depending on $\Theta$ and $K$ so that on an event, we called $A_N$,  with probability at least $1 - c_1N^{-c_2},$ we have
	\begin{equation*}
		\TV(p_{\widehat{G}_N}, p_{G_0}) \leq \sqrt{2}\hel(p_{\widehat{G}_N}, p_{G_0}) \le \left( \dfrac{\log N}{N} \right)^{1/2}.
	\end{equation*}
	Furthermore, from the two theorem above, we obtain that 
	\begin{equation}{\label{thm_bound_dg}}
		\VD(\widehat{G}_N^{(\kappa)}, G_0) \lesssim \left( \dfrac{\log N}{N} \right)^{1/2}.
	\end{equation}
	When $\kappa = K_0$, by defintion of $\VD(\widehat{G}_N^{(\kappa)}, G_0)$, we obtain that $\VD(\widehat{G}_N^{(\kappa_0)}, G_0) \asymp W_1(\widehat{G}_N^{(\kappa_0)}, G_0)$
	Suppose $\kappa \geq K_0 + 1$, then there exist two atoms belonging to the same Voronoi cell. Hence, by (\ref{thm_bound_dg}), there exist $i, j$ and $\sV_t$(for convention, we get rid of dependence of $i, j$ and $\sV_t$ on $N$) such that
	\begin{align*}
		&\pi_i^N\left(\norm{\vc_i^N - \vc_t^0}^{\overline{r}(|\sV_t|)} + \norm{\mGamma_i^N - \mGamma_t^0}^{\overline{r}(|\sV_t|)/2} + \norm{\mA_i^N - \mA_t^0}^{\overline{r}(|\sV_t|)/2} + \norm{\vb_i^N - \vb_t^0}^{\overline{r}(|\sV_t|)} + \norm{\mSigma_i^N - \mSigma_t^0}^{\overline{r}(|\sV_t|)/2}  \right) \\
		&+\pi_j^N\left(\norm{\vc_j^N - \vc_t^0}^{\overline{r}(|\sV_t|)} + \norm{\mGamma_j^N - \mGamma_t^0}^{\overline{r}(|\sV_t|)/2} + \norm{\mA_j^N - \mA_t^0}^{\overline{r}(|\sV_t|)/2} + \norm{\vb_j^N - \vb_t^0}^{\overline{r}(|\sV_t|)} + \norm{\mSigma_j^N - \mSigma_t^0}^{\overline{r}(|\sV_t|)/2}  \right)\\
		&\lesssim \left( \dfrac{\log N}{N} \right)^{1/2}.
	\end{align*}
	With the fact that $\min\lbrace \pi_i, \pi_j\rbrace \geq \dfrac{1}{(\pi_i)^{-1} + (\pi_j)^{-1}}$ and $\bar{r}(\widehat{G}_N) \geq \bar{r}(|\sV_t|) \geq 4$, we get the  lower bound for above left hand side term as follow
	\begin{align*}
		&LHS \gtrsim \\
		&\dfrac{1}{(\pi_i^N)^{-1} + (\pi_j^N)^{-1}}(\norm{\vc_i^N - \vc_j^N}^{\overline{r}(|\sV_t|)} + \norm{\mGamma_i^N - \mGamma_j^N}^{\overline{r}(|\sV_t|)/2} 
		+ \norm{\mA_i^N - \mA_j^N}^{\overline{r}(|\sV_t|)/2}\\
		&+ \norm{\vb_i^N - \vb_j^N}^{\overline{r}(|\sV_t|)} + \norm{\mSigma_i^N- \mSigma_j^N}^{\overline{r}(|\sV_t|)/2}  ) .
	\end{align*}
	Because the height of the Dendrogram is the minimum of dissimilarity $\divclus$ over all pairs $i,j$, we have that
	\begin{align*}
		\height_N^{(\kappa)} &\leq  \dfrac{1}{(\pi_i^N)^{-1} + (\pi_j^N)^{-1}} \left(\norm{\vc_i^N - \vc_j^N}^{2} + \norm{\mGamma_i^N - \mGamma_j^N} + \norm{\mA_i^N - \mA_j^N} +\norm{\vb_i^N - \vb_j^N}^{2} + \norm{\mSigma_i^N - \mSigma_j^N} \right) \\
		&\lesssim   \left( \dfrac{\log N}{N} \right)^{1/\overline{r}(|\sV_t|)}  \lesssim \left( \dfrac{\log N}{N} \right)^{1/\overline{r}(|\sV_t|)} \lesssim \left( \dfrac{\log N}{N} \right)^{1/\bar{r}(\widehat{G}_N)}.
	\end{align*}
	\textbf{Part 2 : convergence rate on under-fitted levels. } Because $\VD(\widehat{G}_N^{(K_0)}, G_0) $ share the same rates with $W_1(G_N^{(K_0)}, G_0)$, we get the convergence rate 
	\begin{equation*}
		W_1(\widehat{G}_N^{(K_0)}, G_0) \lesssim \left( \dfrac{\log N}{N} \right)^{1/2}.
	\end{equation*}
	Let $\widehat{G}_N^{(K_0)} = \sum_{k = 1}^{K_0}\pi_k^N\delta_{(\vc_N, \mGamma_N, \mA_N, \vb_N, \mSigma_N)}$, follow by definition of $\VD(\widehat{G}_N^{(K_0)}, G_0)$, we get for large N enough
	\begin{equation*}
		|\pi_k^N - \pi_k^0| \lesssim \left( \dfrac{\log N}{N} \right)^{1/2}.
	\end{equation*}
	\begin{equation*}
		\lVert \vc_k^N - \vc_{k}^0\rVert \lesssim \left( \dfrac{\log N}{N} \right)^{1/2}, \quad \lVert \mGamma_k^N - \mGamma_{k}^0\rVert \lesssim \left( \dfrac{\log N}{N} \right)^{1/2}.
	\end{equation*}
	\begin{equation*}
		\lVert \mA_k^N - \mA_{k}^0\rVert \lesssim \left( \dfrac{\log N}{N} \right)^{1/2}, \quad \lVert \vb_k^N - \vb_{k}^0\rVert \lesssim \left( \dfrac{\log N}{N} \right)^{1/2}, \quad \lVert \mSigma_k^N - \mSigma_{k}^0\rVert \lesssim \left( \dfrac{\log N}{N} \right)^{1/2}.
	\end{equation*}
	Hence, we can show that
	\begin{align}
		&\left| \dfrac{1}{(\pi_i^N)^{-1} + (\pi_j^N)^{-1}}(\lVert \vc_{i}^N - \vc_{j}^N\rVert^2 + \lVert \mGamma_{i}^N - \mGamma_{j}^N\rVert + \lVert \mA_{i}^N - \mA_{j}^N\rVert + \lVert \vb_{i}^N - \vb_{j}^N\rVert^2 + \lVert \mSigma_{i}^N - \mSigma_{j}^N\rVert )\right|\\
		&-\left|
		\dfrac{1}{(\pi_i^0)^{-1} + (\pi_j^0)^{-1}}( \lVert \vc_{i}^0 - \vc_{j}^0\rVert^2 + \lVert \mGamma_{i}^0 - \mGamma_{j}^0\rVert + \lVert \mA_{i}^0 - \mA_{j}^0\rVert + \lVert \vb_{i}^0 - \vb_{j}^0\rVert^2 + \lVert \mSigma_{i}^0 - \mSigma_{j}^0\rVert ) \right| \\
		&\lesssim \left( \dfrac{\log N}{N} \right)^{1/2}.\label{proof_height_exact_fitted}
	\end{align}
	Hence, for large $N$ enough, the optimal choice to merge $\widehat{G}_N^{(K_0)} $ will be the same as $G_0$. 
	Assume two atoms $i, j$ are merged together and let $(\vc_*^N, \mGamma_*^N, \mA_*^N, \vb_*^N, \mSigma_*^N)$ be merged atom. Likewise, $(\vc_*^0, \mGamma_*^0, \mA_*^0, \vb_*^0, \mSigma_*^0)$ be merged atom for true mixing measure, We can show that 
	\begin{equation*}
		|\pi_*^N - \pi_*^0 | =   \left | \pi_i^N + \pi_j^N - \pi_i^0 - \pi_j^0  \right| \lesssim \left( \dfrac{\log N}{N} \right)^{1/2},
	\end{equation*}
	\begin{align*}
		\left \lVert\dfrac{\pi_i^N}{\pi_*^N}\vc_i^N + \dfrac{\pi_j^N}{\pi_*^N}\vc_j^N - \dfrac{\pi_i^0}{\pi_*^0}\vc_i^0 - \dfrac{\pi_j^0}{\pi_*^0}\vc_j^0 \right \rVert \leq \left \lVert\dfrac{\pi_i^N}{\pi_*^N}\vc_i^N  - \dfrac{\pi_i^0}{\pi_*^0}\vc_i^0 \right \rVert +  \left \lVert\dfrac{\pi_j^N}{\pi_*^N}\vc_j^N  - \dfrac{\pi_j^0}{\pi_*^0}\vc_j^0 \right \rVert .
	\end{align*}
	First of all, we will obtain the following bound
	\begin{align*}
		\left | \dfrac{\pi_i^N}{\pi_*^N} - \dfrac{\pi_i^0}{\pi_*^0} \right| = \left | \dfrac{\pi_i^N\pi_j^0 - \pi_i^0 \pi_j^N}{\pi_*^N \pi_*^0} \right | &\leq \dfrac{1}{c_0^2}|\pi_i^N\pi_j^0 - \pi_i^0 \pi_j^N|\\
		&\lesssim |\pi_i^N||\pi_j^0 - \pi_j^N| + |\pi_j^N||\pi_i^N - \pi_i^0|  \lesssim \left( \dfrac{\log N}{N} \right)^{1/2},
	\end{align*}
	we should note that all $\pi_i^N $ is bounded from below by $c_0$. Now, consider
	\begin{equation*}
		\left \lVert\dfrac{\pi_i^N}{\pi_*^N}\vc_i^N  - \dfrac{\pi_i^0}{\pi_*^0}\vc_i^0 \right \rVert \leq  \left|\dfrac{\pi_i^N}{\pi_*^N}\right| \left \lVert \vc_i^N - \vc_i^0 \right \rVert + \lVert \vc_i^0 \rVert \left | \dfrac{\pi_i^N }{\pi_*^N} - \dfrac{\pi_i^0}{\pi_*^0} \right| \lesssim \left( \dfrac{\log N}{N} \right)^{1/2}
	\end{equation*}
	with a note that parameters are supposed to be bounded. Hence, applying similar arguments, we get that
	\begin{equation*}
		\left \lVert\dfrac{\pi_i^N}{\pi_*^N}\vc_i^N + \dfrac{\pi_j^N}{\pi_*^N}\vc_j^N - \dfrac{\pi_i^0}{\pi_*^0}\vc_i^0 - \dfrac{\pi_j^0}{\pi_*^0}\vc_j^0 \right \rVert \lesssim \left( \dfrac{\log N}{N} \right)^{1/2}.
	\end{equation*}
	Similarly,
	\begin{equation*}
		\left |\dfrac{\pi_i^N}{\pi_*^N}\vb_i^N + \dfrac{\pi_j^n}{\pi_*^N}\vb_j^N - \dfrac{\pi_i^0}{\pi_*^0}\vb_i^0 - \dfrac{\pi_j^0}{\pi_*^0}\vb_j^0 \right | \lesssim \left( \dfrac{\log N}{N} \right)^{1/2},
	\end{equation*}
	\begin{align*}
		&\left |\dfrac{\pi_i^N}{\pi_*^N}(\mGamma_i^N + (\vc_i^N-\vc_*^N)(\vc_i^N-\vc_*^N)^{\top}) + \dfrac{\pi_j^N}{\pi_*^N}(\mGamma_j^N + (\vc_j^N-\vc_*^N)(\vc_j^N-\vc_*^N)^{\top} ) \right.\\
		&\left.- \dfrac{\pi_i^0}{\pi_*^0}(\mGamma_i^0 + (\vc_i^0-\vc_*^0)(\vc_i-\vc_*^0)^{\top}) - \dfrac{\pi_j^0}{\pi_*^0}(\mGamma_j^0 + (\vc_j^0-\vc_*^0)(\vc_j^0-\vc_*^0)^{\top}) \right | \lesssim \left( \dfrac{\log N}{N} \right)^{1/2},
	\end{align*}
	\begin{align*}
		&\left |\dfrac{\pi_i^N}{\pi_*^N}(\mSigma_i^N + (\vb_i^N-\vb_*^N)(\vb_i^N-\vb_*^N)^{\top}) + \dfrac{\pi_j^N}{\pi_*^N}(\mSigma_j^N + (\vb_j^N-\vb_*^N)(\vb_j^N-\vb_*^N)^{\top} ) \right.\\
		&\left.- \dfrac{\pi_i^0}{\pi_*^0}(\mSigma_i^0 + (\vb_i^0-\vb_*^0)(\vb_i-\vb_*^0)^{\top}) - \dfrac{\pi_j^0}{\pi_*^0}(\mSigma_j^0 + (\vb_j^0-\vb_*^0)(\vb_j^0-\vb_*^0)^{\top}) \right | \lesssim \left( \dfrac{\log N}{N} \right)^{1/2},
	\end{align*}
	and
	\begin{align*}
		&\left |\dfrac{\pi_i^N}{\pi_*^N}(\mA_i^N+ (\vb_i^N-\vb_*^N)(\vc_i^N-\vc_*^N)^{\top}) + \dfrac{\pi_j^N}{\pi_*^N}(\mA_j^N + (\vb_j^N-\vb_*^N)(\vc_j^N-\vc_*^N)^{\top} ) \right.\\
		&\left.- \dfrac{\pi_i^0}{\pi_*^0}(\mA_i^0 + (\vb_i^0-\vb_*^0)(\vc_i-\vc_*^0)^{\top}) - \dfrac{\pi_j^0}{\pi_*^0}(\mA_j^0 + (\vb_j^0-\vb_*^0)(\vc_j^0-\vc_*^0)^{\top}) \right | \lesssim \left( \dfrac{\log N}{N} \right)^{1/2}.
	\end{align*}
	Thus,denote $\vtheta_i^N = (\vc_i^N, \mGamma_i^N, \mA_i^N, \vb_i^N, \mSigma_i^N)$, $\vtheta_i^0 = (\vc_i^0, \mGamma_i^0, \mA_i^0, \vb_i^0, \mSigma_i^0)$ and after merging two atoms $i, j$, we call $\vtheta_*^N = (\vc_*^N, \mGamma_*^N, \mA_*^N, \vb_*^N, \mSigma_*^N) $ as well as $\vtheta_*^0 = (\vc_*^0, \mGamma_*^0, \mA_*^0, \vb_*^0, \mSigma_*^0) $, we deduce that
	\begin{equation*}
		W_1(\widehat{G}_N^{(K_0-1)}, G_0^{(K_0-1)}) \lesssim \left( \dfrac{\log N}{N} \right)^{1/2}.
	\end{equation*}
	Follow by (\ref{proof_height_exact_fitted}), we can conclude that
	\begin{equation*}
		|\height_N^{(K_0)} - \height_0^{(K_0)}| \ \lesssim \left( \dfrac{\log N}{N} \right)^{1/2}.
	\end{equation*}
	The rest is followed by induction.\\
	
	\subsection{Proof of Theorem \ref{theorem_likelihood_DIC_GLLiM}} \label{sec_proof_theorem_likelihood_DIC_GLLiM}
	\subsubsection{Preliminary on Empirical Process Theory}
	Suppose $\rvx_1, \dots, \rvx_N \sim P_{G_0}.$ Denote $P_N := \dfrac{1}{N}\sum_{n = 1}^N\delta_{\rvx_n}$ is the empirical measure.
	Denote the empirical process for $G:$
	\begin{equation*}
		\nu_N(G) := \sqrt{N}(P_N - P_{G_0})\log\dfrac{p_{\bar{G}}}{p_{G_0}}.
	\end{equation*}
	The following results is important in proof below.
	\begin{fact}[Theorem 5.11 in \cite{geer2000empirical}]{\label{emerical-process-bounding}}
		Let positive numbers $R, C, C_1, a$ satisfy:
		\begin{equation*}
			a \leq C_1\sqrt{N}R^2 \wedge 8\sqrt{N}R,
		\end{equation*}
		and
		\begin{equation*}
			a \geq \sqrt{C^2(C_1 + 1)} \left( \int_{a/(2^6\sqrt{N}}^R H_B^{1/2}\left( \dfrac{u}{\sqrt{2}}, \lbrace p_G : G \in \mathcal{O}_K, h(p_G, p_{G_0}) \leq R \rbrace, \nu\right)du \vee R\right),
		\end{equation*}
		then 
		\begin{equation*}
			\mathbb{P}_{G_0}\left( \sup_{\substack{G \in \mathcal{O}_K \\
					\hel(p_G, p_{G_0}) \leq R}
			}|\nu_N(G)| \geq a \right) \leq C\exp\left( -\dfrac{a^2}{C^2(C_1 + 1)R^2} \right).
		\end{equation*}
	\end{fact}
	\textbf{Proof of Theorem \ref{theorem_likelihood_DIC_GLLiM}.} We divide into three cases.
	
	\textbf{Case 1: } $\kappa \geq K_0$. The empirical measure is denoted by $P_N = \dfrac{1}{N}\sum_{n=1}^N\delta_{\rvy_n, \rvx_n}$. For any $G$, denoting $P_G$ by the distribution of $p_G$. Using the property of concave function, 
	\begin{equation*}
		\dfrac{1}{2}\log\dfrac{p_G}{p_{G_0}} \leq \log \dfrac{p_G + p_{G_0}}{2p_{G_0}} = \log\dfrac{\bar{p}_G}{p_{G_0}}.
	\end{equation*}
	Thus, for all $\kappa \geq K_0$,
	\begin{align*}
		\dfrac{1}{2}P_n\log\dfrac{p_{\widehat{G}_N^{(\kappa)}}}{p_{G_0}} &\leq P_n \log \dfrac{\bar{p}_{\widehat{G}_N^{(\kappa)}}}{p_{G_0}} \\
		&= (P_n - P_{G_0})\log \dfrac{\bar{p}_{\widehat{G}_N^{(\kappa)}}}{p_{G_0}}  - \KL ( p_{G_0} \Vert \bar{p}_{\widehat{G}_N^{(\kappa)}} )\\
		&\leq (P_n - P_{G_0})\log \dfrac{\bar{p}_{\widehat{G}_N^{(\kappa)}}}{p_{G_0}} .
	\end{align*}
	Hence, 
	\begin{align*}
		P_N\log p_{\widehat{G}_N^{(\kappa)}} - P_{G_0}\log p_{G_0} &= P_N\log\dfrac{p_{\widehat{G}_N^{(\kappa)}}}{p_{G_0}} + (P_N - P_{G_0})\log p_{G_0}\\
		&\leq 2(P_N - P_{G_0})\log \dfrac{\bar{p}_{\widehat{G}_N^{(\kappa)}}}{p_{G_0}}  + (P_N - P_{G_0})\log p_{G_0}.
	\end{align*}
	In the theorem \ref{theorem_rate_DIC_GLLiM}, we know that $\VD(\widehat{G}_N^{(\kappa)}, G_0) \lesssim (\log N/N)^{1/2}$ and clearly we have $W_{\bar{r}(\widehat{G}_N)}^{\bar{r}(\widehat{G}_N)}(\widehat{G}_N^{(\kappa)}, G_0)  \leq \VD(\widehat{G}_N^{(\kappa)}, G_0) $ ,hence there is a constant $D$ so that
	\begin{equation*}
		\mathbb{P}_{G_0}\left(W_{\bar{r}(\widehat{G}_N)}(\widehat{G}_N^{(\kappa)}, G_0) \leq D \left( \dfrac{\log N}{N} \right)^{1/2\bar{r}(\widehat{G}_N)} \right) \geq 1 - c_1N^{-c_2}, \quad \kappa \in [K_0 , K].
	\end{equation*}
	Using the fact that $\hel \lesssim W_{2}$, see for example, \cite{nguyen_convergence_2013}, with a note that for a probability $q_{ij}$, we have $q_{ij} \leq q_{ij}^{2/{\bar{r}(\widehat{G}_N)}}  $, since $2/{\bar{r}(\widehat{G}_N)} \leq2/4 < 1.$ Combining with all norms on finite space is equivalent, we obtain that 
	\begin{equation*}
		\left( \sum_{i,j}q_{ij}\norm{\evtheta_k - \evtheta'_l}^2 \right)^{1/2}\leq  \left( \sum_{i,j}q_{ij}^{2/{\bar{r}(\widehat{G}_N)}}\norm{\evtheta_k - \evtheta'_l}^2 \right)^{1/2} \lesssim \left( \sum_{i,j}q_{ij}\norm{\evtheta_k - \evtheta'_l}^{{\bar{r}(\widehat{G}_N)}}\right)^{1/{\bar{r}(\widehat{G}_N)}}
	\end{equation*}
	Then, we get $\hel \lesssim W_{\bar{r}(\widehat{G}_N)}$, we also have
	\begin{equation*}
		\mathbb{P}_{G_0}\left(\hel(p_{\widehat{G}_N^{(\kappa)}}, p_{G_0}) \leq D \left( \dfrac{\log N}{N} \right)^{1/2\bar{r}(\widehat{G}_N)} \right) \geq 1 - c_1N^{-c_2}, \quad \kappa \in [K_0 , K].
	\end{equation*}
	Define, $\alpha := 1/2\bar{r}(\widehat{G}_N) \leq 1/4$, substitute $R = D\left( \dfrac{\log N}{N} \right)^{\alpha}, a = D\dfrac{\log^{\alpha + 1/2} N}{N^{\alpha}}$, we have $a \leq \sqrt{N}R^2 \leq \sqrt{N}R$, for large N enough, and with the fact that $H_B^{1/2}(u, P_K, h) \lesssim \log(1/u) $( Lemma 4 in \cite{nguyen_towards_2024}), we obtain that
	\begin{equation*}
		a \geq R\left( \log \left( \dfrac{2^6 \sqrt{N}}{a} \right) \right) \geq \int_{a / (2^6\sqrt{N})}^{R}H_B^{1/2}\left(\dfrac{u}{\sqrt{2}}, \lbrace p_G: G \in \mathcal{O}_K, \hel(p_G, p_{G_0}) \leq R \rbrace , \nu \right)du.
	\end{equation*}
	Applying \cref{emerical-process-bounding}, we get
	\begin{equation*}
		\mathbb{P}_{G_0}\left( \sup_{\hel(p_G, p_{G_0}) \leq D(\log N / N)^{\alpha}}\left| \sqrt{N} (P_N - P_{G_0}) \log\dfrac{p_G}{p_{G_0}}
		\right| \geq D\dfrac{\log^{\alpha + 1/2}N}{N^{\alpha}} \right) \leq N^{-c_2}.
	\end{equation*}
	Combinining with the bound on Hellinger distance, we have
	\begin{align*}
		&\mathbb{P}_{G_0}\left( \left| (P_N - P_{G_0})\log \dfrac{p_{\widehat{G}_N^{(\kappa)}}}{p_{G_0}} \right| \geq D\dfrac{\log^{\alpha + 1/2}N}{N^{\alpha + 1/2}}\right) \leq \mathbb{P}_{G_0}((\hel(p_{\widehat{G}_N^{(\kappa)}}, p_{G_0}) \geq D(\log n/ n)^{\alpha})\\
		&+ \mathbb{P}_{G_0}\left( \left| (P_N - P_{G_0})\log \dfrac{p_{\widehat{G}_N^{(\kappa)}}}{p_{G_0}} \right| \geq D\dfrac{\log^{\alpha + 1/2}N}{N^{\alpha + 1/2}},\hel(p_{\widehat{G}_N^{(\kappa)}}, p_{G_0}) \leq D(\log N/ N)^{\alpha} \right) \\
		&\leq c_1N^{-c_2} + \mathbb{P}_{G_0}\left( \sup_{\hel(p_G, p_{G_0}) \leq D(\log N / N)^{\alpha}}\left| \sqrt{n} (P_N - P_{G_0}) \log\dfrac{p_G}{p_{G_0}}
		\right| \geq D\dfrac{\log^{\alpha + 1/2}N}{N^{\alpha}} \right)\\
		&\leq c_1'N^{-c_2}.
	\end{align*}
	Furthermore, applying Chebyshev inequality, 
	\begin{equation}{\label{chebyshev ineq}}
		\mathbb{P}_{G_0}(|(P_N - P_{G_0})\log p_{G_0}| \geq t ) \leq \dfrac{Var(\log p_{G_0})}{Nt^2}.
	\end{equation}
	By choosing, $t \gtrsim N^{-1/2}$, we can also bound this term. Choose $t = (\log N / N)^{\alpha}$, we have
	\begin{equation*}
		\mathbb{P}_{G_0}(|(P_N - P_{G_0})\log p_{G_0}| \leq (\log N / N)^{\alpha}  ) \geq 1 - c_1N^{-c_2},
	\end{equation*}
	Hence, we conclude that
	\begin{equation*}
		\mathbb{P}_{G_0}( P_n\log p_{\widehat{G}_N^{(\kappa)}} - P_{G_0}\log p_{G_0} \leq (\log N / N) ^{1/2\bar{r}(\widehat{G}_N)}) \geq 1 - c_1N^{-c_2}.
	\end{equation*}
	\textbf{Case 2:} $\kappa = K_0$. At the exact-fitted level, we have
	\begin{equation*}
		W_1(\widehat{G}_N^{(K_0)}, G_0) \lesssim \left( \dfrac{\log N}{N} \right)^{1/2}.
	\end{equation*}
	Let $\widehat{G}_N^{(K_0)} = \sum_{k = 1}^{K_0}\pi_k^N\delta_{(\vc_k^N, \mGamma_k^N, \mA_k^N, \vb_k^N, \mSigma_k^N)}$, follow by definition of $\VD(\widehat{G}_N^{(K_0)}, G_0)$, we get for large N enough
	\begin{equation*}
		|\pi_k^N - \pi_k^0| \lesssim \left( \dfrac{\log N}{N} \right)^{1/2},
	\end{equation*}
	\begin{equation*}
		\lVert \vc_k^N - \vc_{k}^0\rVert \lesssim \left( \dfrac{\log N}{N} \right)^{1/2}, \quad \lVert \mGamma_k^N - \mGamma_{k}^0\rVert \lesssim \left( \dfrac{\log N}{N} \right)^{1/2},
	\end{equation*}
	\begin{equation*}
		\lVert \mA_k^N - \mA_{k}^0\rVert \lesssim \left( \dfrac{\log N}{N} \right)^{1/2}, \quad \lVert \vb_k^N - \vb_{k}^0\rVert \lesssim \left( \dfrac{\log N}{N} \right)^{1/2}, \quad \lVert \mSigma_k^N - \mSigma_{k}^0\rVert \lesssim \left( \dfrac{\log N}{N} \right)^{1/2}.
	\end{equation*}
	Let $\epsilon_N = (\log N /N)^{1/2} \rightarrow 0$, from condition $(K)$, there exist $c_\alpha$ and $c_\beta > 0$ such that
	\begin{equation*}
		f_{\cL}(\vx| \vc_k^N, \mGamma_k^N)f_{\cD}(\vy| \mA_k^N\vx + \vb_k^N, \mSigma_k^N) \geq \left(f_{\cL}(\vx| \vc_k^0, \mGamma_k^0)f_{\cD}(\vy| \mA_k^0\vx + \vb_k^0, \mSigma_k^0\right)^{(1 + c_\beta \epsilon_N) }e^{-c_\alpha \epsilon_N}.
	\end{equation*}
	Besides, we can find constant $c_p > 0$ such that $\pi_k^N \geq (1 - c_p\epsilon_N)\pi_k^0$ for all $k \in [K_0].$ With the fact that $g(t) = t^{1 + c_\beta \epsilon_n}$ is convex function, we have
	\begin{align*}
		p_{\widehat{G}_N^{(K_0)}} &\geq (1- c_p\epsilon_N)\sum_{k = 1}^{K_0}\pi_k^0 \left(f_{\cL}(\vx| \vc_k^0, \mGamma_k^0)f_{\cD}(\vy| \mA_k^0\vx + \vb_k^0, \mSigma_k^0\right)^{(1 + c_\beta \epsilon_N) }e^{-c_\alpha \epsilon_N}\\
		&\geq (1- c_p\epsilon_N)p_{G_0}^{(1 + c_\beta\epsilon_N)}e^{-c_\alpha \epsilon_N}.
	\end{align*}
	The last inequality followed by convex property, Therefore, we have
	\begin{equation*}
		\dfrac{1}{N}\sum_{i=1}^{N}\log\dfrac{p_{\widehat{G}_N^{(K_0)}}}{p_{G_0}}(\rvx_i, \rvy_i) \geq \log (1- c_p\epsilon_N) - (c_\alpha \epsilon_N) + (c_\beta\epsilon_N)\dfrac{1}{N}\sum_{i=1}^N\log p_{G_0}(\rvx_i, \rvy_i).
	\end{equation*}
	then
	\begin{equation}{\label{eq-exact-avgl}}
		\bar{l}_N(p_{\widehat{G}_N^{(K_0)}}) - \cL(p_{G_0}) \geq \log(1-c_p\epsilon_N) - c_\alpha\epsilon_N + c_\beta\epsilon_NP_{G_0}\log p_{G_0} + (1 + c_\beta \epsilon_N)(P_N - P_{G_0})\log p_{G_0}.
	\end{equation}
	Now, we will bound the right hand side of above equation, from Chebyshev inequality from (\ref{chebyshev ineq}), choose $t = (\log N / N)^{1/2}$, we get that
	\begin{equation*}
		\mathbb{P}_{G_0}\left( |(P_N - P_{G_0})\log p_{G_0}| \geq \left(\dfrac{\log N}{N}\right)^{1/2}\right) \leq \dfrac{Var(\log p_{G_0})}{\log N}.
	\end{equation*}
	Clearly, the terms $\log(1-c_p\epsilon_N) - c_\alpha\epsilon_N + c_\beta\epsilon_NP_{G_0}\log p_{G_0} \lesssim \epsilon_N = (\log N / N)^{1/2}$, thus there exits a constant $C > 0$ such that $\log(1-c_p\epsilon_N) - c_\alpha\epsilon_N + c_\beta\epsilon_NP_{G_0}\log p_{G_0}  \geq -C (\log N / N)^{1/2}$. Then, for some constant $C_e > 0$,
	\begin{equation*}
		\mathbb{P}_{G_0}\left( \text{RHS of (\ref{eq-exact-avgl})} \geq -C_{e}\left(\dfrac{\log N}{N}\right)^{1/2}\right) \geq 1- \dfrac{Var(\log p_{G_0})}{\log N}.
	\end{equation*}
	Call the event under above case is B, then we obtain that
	\begin{align*}
		\mathbb{P}_{G_0}\left(  \bar{l}_N(p_{\widehat{G}_N^{(K_0)}}) - \cL(p_{G_0}) \geq -C_e (\log N / N)^{1/2}  \right) &\geq P_{G_0}(A \cap B) = P_{G_0}(B) - P_{G_0}(B \cap A^c) 
		\\
		&\geq P_{G_0}(B) - P_{G_0}(A^c)
		=1 - \dfrac{Var(\log p_{G_0})}{\log N} -  c_1N^{-c_2}.
	\end{align*}
	approach 1 when $N \rightarrow \infty.$ So, combine both results, we can conclude that
	\begin{equation*}
		|\bar{l}_N(p_{\widehat{G}_N^{(K_0)}}) - \cL(p_{G_0})| \lesssim (\log N / N) ^{1/2\bar{r}(\widehat{G}_N)}.
	\end{equation*}
	\\
	\textbf{Case 3: }$\kappa < K_0.$ Since $|\log p_{G}(\vx, \vy)| \leq m(\vx, \vy)$ for a measurable function for all $G \in O_k$, we can use uniform law of large number~(Theorem 9.2 from \cite{keener2010theoretical}) to have that
	\begin{equation*}
		\sup_{G \in O_k}|\bar{l}_N(G) - P_{G_0}\log p_{G}| \overset{\sP}{\rightarrow} 0,
	\end{equation*}
	where $\overset{\sP}{\rightarrow}$ means convergence in probability. Therefore, 
	\begin{equation*}
		\left |\bar{l}_N(\widehat{G}_N^{(\kappa)}) - P_{G_0}\log p_{\widehat{G}_N^{(\kappa)}}\right| \overset{\sP}{\rightarrow} 0.
	\end{equation*}
	We know that $\log p_{\widehat{G}_N^{(\kappa)}} \rightarrow \log p_{\widehat{G}_0^{(\kappa)}}$ in probability, by application of Dominated Convergence theorem, we obtain
	\begin{equation*}
		P_{G_0}\log p_{\widehat{G}_N^{(\kappa)}} \overset{\sP}{\to} P_{G_0}\log p_{\widehat{G}_0^{(\kappa)}}.
	\end{equation*}
	Combines above results together, we get
	\begin{equation*}
		\bar{l}_N(\widehat{G}_N^{(\kappa)}) \overset{\sP}{\to} P_{G_0}\log p_{\widehat{G}_0^{(\kappa)}} = \cL(\log p_{\widehat{G}_0^{(\kappa)}} ).
	\end{equation*}
	\textbf{Checking condition K}. Now via \cref{lem_condition_K},  we check if condition $(K)$ is satisfied for $f(\vx, \vy| \vtheta) = f_{\cL}(\vx|\vc, \mGamma) f_{\cD}(\vy| \mA ^{\top} \vx + \vb, \mSigma)$ from our GGMoE model.
	\begin{lemma}[Condition K for GGMoE model]\label{lem_condition_K}
		The condition $(K)$ is satisfied for $f(\vx, \vy| \vtheta) = f_{\cL}(\vx|\vc, \mGamma) f_{\cD}(\vy| \mA ^{\top} \vx + \vb, \mSigma)$, where $\vtheta = (\vc, \mGamma, \mA, \vb, \mSigma) \in \Theta \in \mathbb{R}^{D} \times \mathcal{S}_D^{+} \times \mathbb{R}^{1 \times D} \times \sR \times \sS_1^{+} $ and $\mathcal{X}$ are bounded as from the initial setup, and the eigenvalues of $\mGamma, \mSigma$ is bounded below and above by the positive constants $\sigma_{\min}$ and $\sigma_{\max}$.
	\end{lemma}
	\begin{proof}[Proof of \cref{lem_condition_K}]
		When $\lVert \vtheta - \vtheta_0 \rVert\leq \epsilon$, by the equivalence of the norm, we can consider the cases where $\lVert \vc - \vc_0 \rVert  , \lVert \mGamma - \mGamma_0 \rVert , \lVert \mA - \mA_0 \rVert , \lVert \vb - \vb_0 \rVert , \lVert \mSigma - \mSigma_0 \rVert \leq \epsilon$. We want to show that for sufficiently small $\epsilon$, there exist $c_\alpha, c_\beta > 0$ such that
		\begin{equation*}
			\log \left( f_{\cL}(\vx|\vc, \mGamma)f_{\cD}(\vy| \mA \vx + \vb, \mSigma) \right) \geq (1 + c_\beta\epsilon)\log \left( f_{\cL}(\vx|\vc_0, \mGamma_0)f_{\cD}(\vy| \mA_0 \vx + \vb_0, \mSigma_0) \right) - c_\alpha \epsilon.
		\end{equation*}
		which is equivalent to 
		\begin{align*}
			&(1 + c_\beta\epsilon)\log(|\mGamma_0|) - \log(|\mGamma|) + (1 + c_\beta \epsilon)(\vx - \vc_0)^{\top}(\mGamma_0)^{-1}(\vx - \vc_0) - (\vx - \vc)^{\top}\mGamma^{-1}(\vx - \vc) \\
			+&(1 + c_\beta\epsilon)\log(|\mSigma_0|) - \log(|\mSigma|) \\
			+ &(1 + c_\beta \epsilon)(\vy - \mA_0 \vx - \vb_0)^{\top}(\mSigma_0)^{-1}(\vy - \mA_0\vx - \vb_0) - (\vy - \mA \vx - \vb)^{\top}(\mSigma)^{-1}(\vy - \mA\vx - \vb) + 2c_\alpha \epsilon \geq 0.
		\end{align*}
		Note that 
		\begin{equation*}
			\dfrac{d\log(|\mGamma|)}{d\mGamma} = \mGamma^{-1},
		\end{equation*}
		and if  $\lVert \mGamma \rVert$ is bounded above and below far from 0 ( which satisfies due to $\mGamma$ is positive definite), then the map $\mGamma \to \log(|\mGamma|)$ is Lipschitz, or there exists constant $c_\sigma$ such that 
		\begin{equation*}
			|\log(|\mGamma_0|) - \log(|\mGamma|) | \leq c_\sigma \lVert \mGamma_0 - \mGamma \rVert.
		\end{equation*}
		Moreover, we know that $|\mGamma|\geq \sigma_{\min}^{D}$ and  $|\mSigma| \geq \sigma_{min}$. Hence, for all $c_\beta >  \dfrac{c_\sigma}{\log(\sigma_{\min})}\max \left\lbrace\dfrac{1}{D},\dfrac{1}{1}\right\rbrace $, then we have 
		\begin{equation*}
			c_\beta\epsilon\log(|\mGamma_0|) \geq c_\sigma \epsilon \geq c_\sigma \lVert \mGamma - \mGamma_0\rVert \geq |\log(|\mGamma|) - \log(|\mGamma_0|) |.
		\end{equation*}
		So that
		\begin{equation*}
			(1 + c_\beta\epsilon) \log(|\mGamma_0|) \geq \log(|\mGamma|).
		\end{equation*}
		Similarly, we obtain that
		\begin{equation*}
			(1 + c_\beta\epsilon) \log(|\mSigma_0|) \geq \log(|\mSigma|).
		\end{equation*}
		We want to choose $c_\alpha > 0$ such that 
		\begin{equation*}
			(1 + c_\beta \epsilon)(\vx - \vc_0)^{\top}(\mGamma_0)^{-1}(\vx - \vc_0) - (\vx - \vc)^{\top}\mGamma^{-1}(\vx - \vc) + c_\alpha\epsilon \geq 0.
		\end{equation*}
		let $u := \vx - \vc_0, \Delta u := \vc_0 - \vc$, using the boundedness of $\mGamma, \mSigma$, there exist $c_\mSigma$ such that
		\begin{equation*}
			\mGamma_0^{-1} \geq c_\mSigma \mGamma^{-1}, \quad \mSigma_0^{-1} \geq  c_\mSigma \mSigma^{-1},
		\end{equation*}
		then we only need to prove
		\begin{equation*}
			(1 + c_\beta\epsilon)c_\mSigma u^{\top}\mGamma^{-1}u - (u + \Delta u)^{\top}\mGamma^{-1}(u + \Delta u) + c_\alpha \epsilon \geq 0,
		\end{equation*}
		which is equivalent to 
		\begin{align*}
			&c_\beta \epsilon c_\mSigma u^{\top} \mGamma^{-1}u - u^{\top}\mGamma^{-1}\Delta u - (\Delta u)^{\top}\mGamma^{-1}u - (\Delta u)^{\top} \mGamma^{-1} \Delta u + c_\alpha \epsilon \geq 0\\
			=& \epsilon c_\beta c_\mSigma \left( u - \dfrac{\Delta u}{\epsilon c_\beta c_\mSigma} \right)^{\top} \mGamma^{-1}\left( u - \dfrac{\Delta u}{\epsilon c_\beta c_\mSigma} \right) + c_\alpha \epsilon \geq \left(1 + \dfrac{1}{\epsilon c_\beta c_\mSigma}  \right)(\Delta u)^{\top}\mGamma^{-1}(\Delta u).
		\end{align*}
		We can bound the RHS of above equation as follow
		\begin{equation*}
			\left(1 + \dfrac{1}{\epsilon c_\beta c_\mSigma}  \right)(\Delta u)^{\top}\mGamma^{-1}(\Delta u) \leq  \left(1 + \dfrac{1}{\epsilon c_\beta c_\mSigma}  \right) \dfrac{\lVert \Delta u\rVert^2}{\sigma_{\min}} \leq  \left(1 + \dfrac{1}{\epsilon c_\beta c_\mSigma}  \right)\dfrac{\epsilon^2}{\sigma_{\min}}.
		\end{equation*}
		Hence, we only need to choose $c_\alpha$ such that
		\begin{equation*}
			c_\alpha \geq  \left(1 + \dfrac{1}{\epsilon c_\beta c_\mSigma}  \right)\dfrac{\epsilon}{\sigma_{\min}} = \dfrac{\epsilon}{\sigma_{\min}} + \dfrac{1}{c_\beta c_\mSigma \sigma_{\min}}.
		\end{equation*}
		Similarly, we want to choose $c_\alpha > 0$ such that
		\begin{equation*}
			(1 + c_\beta \epsilon)(\vy - \mA_0 \vx - \vb_0)^{\top}(\mSigma_0)^{-1}(\vy - \mA_0\vx - \vb_0) - (\vy - \mA \vx - \vb)^{\top}(\mSigma)^{-1}(\vy - \mA\vx - \vb) + c_\alpha \epsilon \geq 0.
		\end{equation*}
		Define $u := \vy - \mA_0 \vx - \vb_0, \Delta u := \mA_0 \vx + \vb_0 - \mA \vx - \vb$, same as previous argument, we only need to prove
		\begin{equation*}
			(1 + c_\beta\epsilon)c_\mSigma u^{\top}\mSigma^{-1}u - (u + \Delta u)^{\top}\mSigma^{-1}(u + \Delta u) + c_\alpha \epsilon \geq 0.
		\end{equation*}
		Likewise, we obtain
		\begin{equation*}
			\epsilon c_\beta c_\mSigma \left( u - \dfrac{\Delta u}{\epsilon c_\beta c_\mSigma} \right)^{\top} \mSigma^{-1}\left( u - \dfrac{\Delta u}{\epsilon c_\beta c_\mSigma} \right) + c_\alpha \epsilon \geq \left(1 + \dfrac{1}{\epsilon c_\beta c_\mSigma}  \right)(\Delta u)^{\top}\mSigma^{-1}(\Delta u).
		\end{equation*}
		Remember that $\mathcal{X}$ is bounded, thus $\lVert \Delta u\rVert \leq \lVert \mA_0  - \mA \rVert \lVert \vx \rVert + \lVert \vb_0 - \vb\rVert \leq R \epsilon $, for some $R > 0$ so we can bound the RHS of above equation as follow
		\begin{equation*}
			\left(1 + \dfrac{1}{\epsilon c_\beta c_\mSigma}  \right)(\Delta u)^{\top}\mSigma^{-1}(\Delta u) \leq  \left(1 + \dfrac{1}{\epsilon c_\beta c_\mSigma}  \right) \dfrac{\lVert \Delta u\rVert^2}{\sigma_{\min}}
			\leq  \left(1 + \dfrac{1}{\epsilon c_\beta c_\mSigma}  \right)\dfrac{\epsilon^2 R^2}{\sigma_{\min}}.
		\end{equation*}
		So we only need to choose $c_\alpha$ such that 
		\begin{equation*}
			c_\alpha \geq \left(1 + \dfrac{1}{\epsilon c_\beta c_\mSigma}  \right)\dfrac{\epsilon R^2}{\sigma_{\min}}  = \dfrac{\epsilon R^2}{\sigma_{\min}} + \dfrac{R^2}{c_\beta c_\mSigma \sigma_{\min}}.
		\end{equation*}
		Combine both ways to obtain $c_\alpha$, we complete the proof.
	\end{proof}
	\subsection{Proof of Theorem \ref{DIC}}
	Note that entropy $ H(p_{G_0}) = - \cL(p_{G_0})$. We have 
	\begin{equation*}
		\height_N^{(\kappa)} = \begin{cases}
			O\left(\left( \dfrac{\log N}{N} \right)^{1/ \bar{r}(\widehat{G}_N)}\right), \hspace{22pt}\text{if } \kappa > K_0,\\
			\height_0^{(\kappa)} + O\left(\left( \dfrac{\log N}{N} \right)^{1/ 2}\right), \hspace{10pt} \text{if } \kappa \leq K_0.
		\end{cases}
	\end{equation*}
	and in the proof of Theorem $\ref{theorem_likelihood_DIC_GLLiM}$, we get
	\begin{equation*}
		\begin{cases}
			\bar{l}_N^{(\kappa)} \leq -H(p_{G_0}) +  O\left(\left( \dfrac{\log N}{N} \right)^{1/ 2\bar{r}(\widehat{G}_N)}\right), \hspace{10pt} \text{if } \kappa > K_0,\\
			\bar{l}_N^{(\kappa)} = -H(p_{G_0}) +  O\left(\left( \dfrac{\log N}{N} \right)^{1/ 2\bar{r}(\widehat{G}_N)}\right), \hspace{10pt} \text{if } \kappa = K_0,\\
			\bar{l}_N^{(\kappa)} = -H(p_{G_0)} - D_{KL}(p_{G_0} || p_{G_0}^{(\kappa)}) + o(1) , \hspace{16pt} \text{if } \kappa < K_0.
		\end{cases}
	\end{equation*}
	Then we have
	\begin{equation*}
		\begin{cases}
			DSC_N^{(\kappa)} \geq\omega_N H(p_{G_0}) + O\left(\omega_N\left( \dfrac{\log N}{N} \right)^{1/ 2\bar{r}(\widehat{G}_N)}\right),\hspace{37pt}\text{if } \kappa > K_0,\\
			DSC_N^{(\kappa)} = \omega_N H(p_{G_0}) - \height_0^{(\kappa)} + O\left(\omega_N \left( \dfrac{\log N}{N} \right)^{1/ 2\bar{r}(\widehat{G}_N)}\right), \hspace{10pt}\text{if } \kappa = K_0,\\
			DSC_N^{(\kappa)} = \omega_N H(p_{G_0}) + \omega_N D_{KL}(p_{G_0}|| p_{G_0}^{(\kappa)}) - \height_0^{(\kappa)} + o(\omega_N), \hspace{7pt} \text{if } \kappa < K_0.
		\end{cases}
	\end{equation*}
	Since $\omega_N \rightarrow \infty$, $\omega_N (\log N / N)^{1/2\bar{r}(\widehat{G}_N)} \rightarrow 0$ and $ D_{KL}(p_{G_0}|| p_{G_0}^{(\kappa)}) > 0$, then as $N \rightarrow \infty$, $\dsc_{N}^{K_0}$ is the smallest number. Hence, $\mathbb{P}_{p_{G_0}}(\widehat{K}_N =  K_0) \geq \mathbb{P}_{p_{G_0}(A_N)}\rightarrow 1$ as $N \rightarrow \infty$ , or $\widehat{K}_N \rightarrow K_0$ in probability.
	
	%%%%%%%%%%%%%%%%%%%%%%%%%%%%%%%%%%%%%%%%%%%%%%%%%%%%%%%%%%%%

	\section{Additional Experiments}\label{sec_additional_experiments}
	
	As a complement to the results presented in \cref{fig_dendrogram_illustration}, we include a representative simulation of a dataset generated from the GGMoE model in \cref{fig_typical_simulation}.
	
	% \iffalse
	\begin{figure*}[!ht]
		\centering
		\begin{subfigure}[t]{0.5\textwidth}
			\centering
			\includegraphics[width=\linewidth]{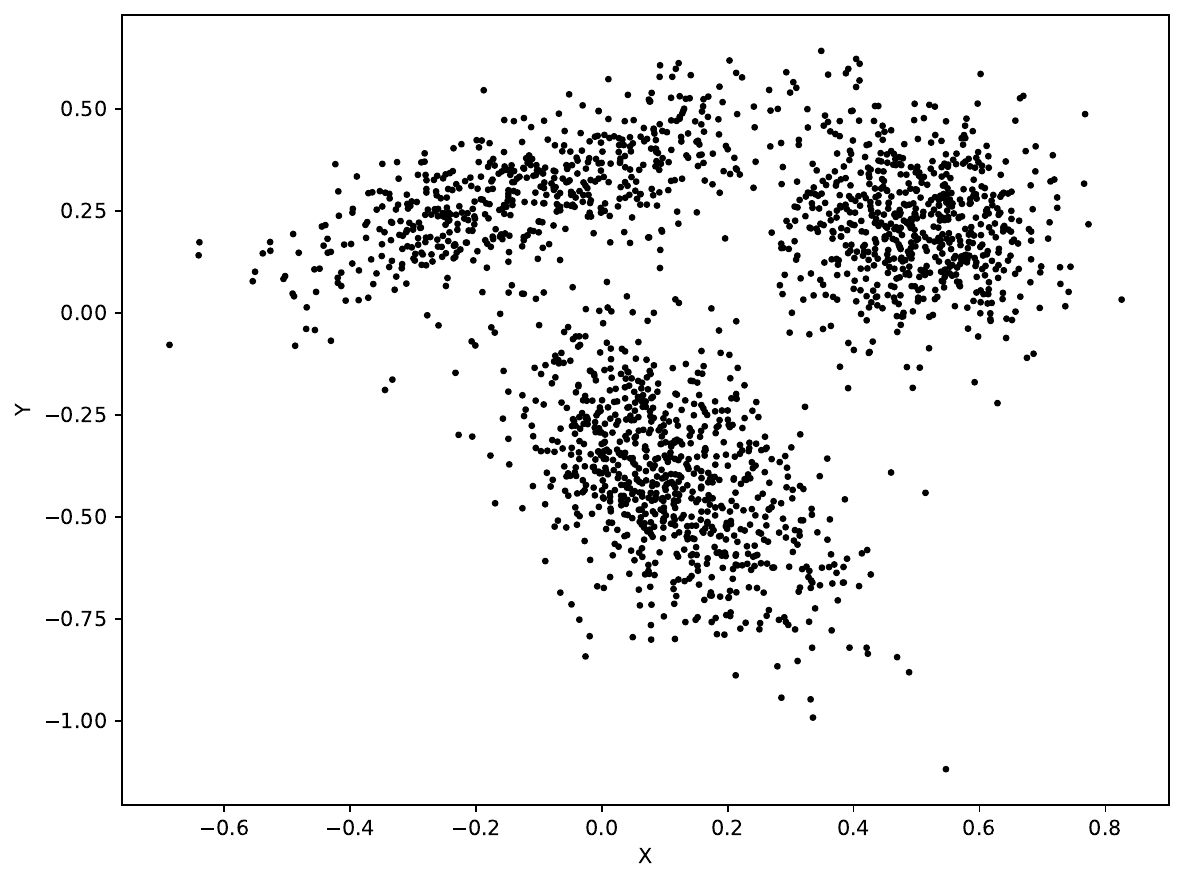}
			\caption{A typical realization of the simulated dataset.}
			\label{fig_draw_data_GGMoE1}
		\end{subfigure}%
		~ 
		\begin{subfigure}[t]{0.5\textwidth}
			\centering
			\includegraphics[width=\linewidth]{figure/plot_true_GGMoE_model_1_K10_nmin100_nmax10000_rep20_nnum21_initAtom1.pdf}
			\caption{True clustering structure with local means.}
			\label{fig_true_cluster_mean_GGMoE1}
		\end{subfigure}
		~
		\begin{subfigure}[t]{0.475\textwidth}
			\centering
			\includegraphics[width=\linewidth]{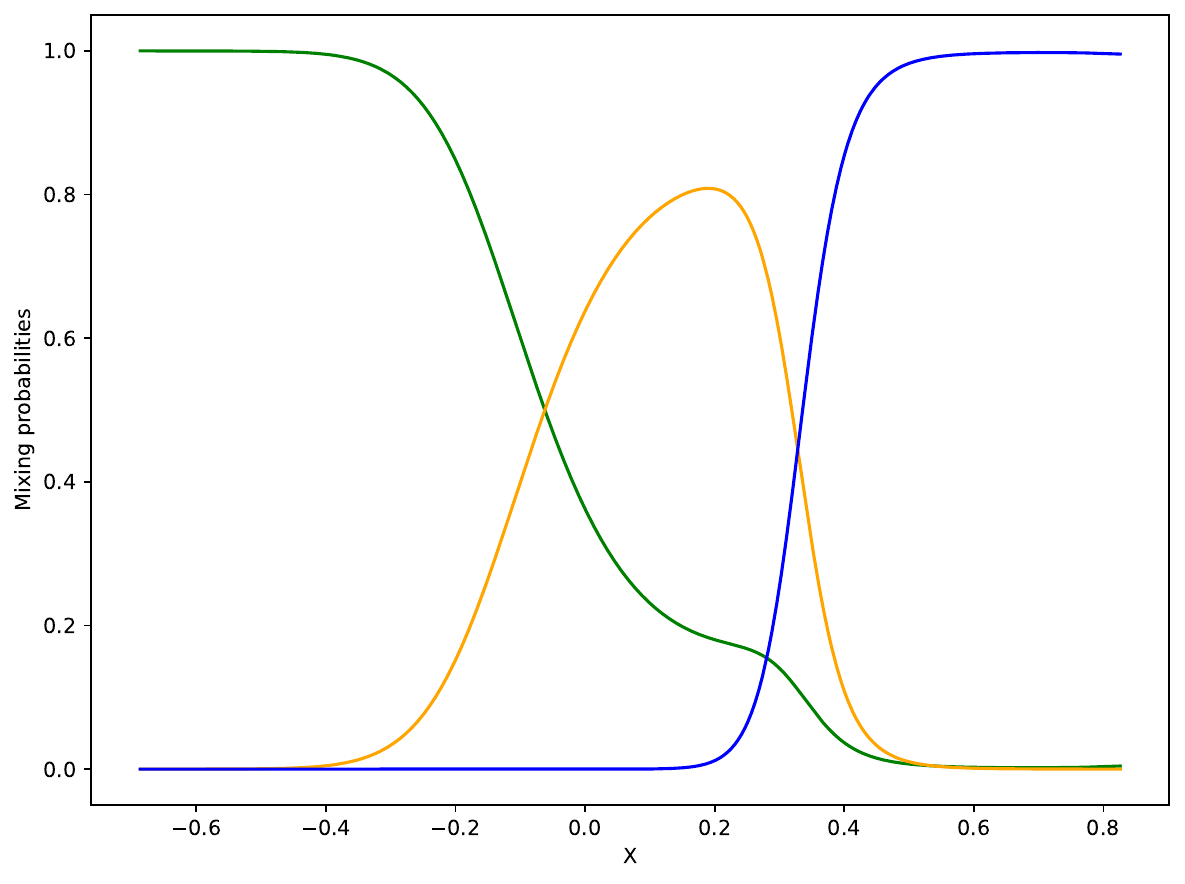}
			\caption{True Gaussian-gated functions.}
			\label{fig_true_Gaussian-gated_GGMoE1}
		\end{subfigure}
		~ 
		\begin{subfigure}[t]{0.5\textwidth}
			\centering
			\includegraphics[width=\linewidth]{figure/plot_estimated_GGMoE_saved_data_model_1_kappa3_nmin100_nmax10000_rep20_nnum21_initAtom1.pdf}
			\caption{Estimated clustering structure with local means.}
			\label{fig_estimated_cluster_mean_GGMoE1}
		\end{subfigure}
		\caption{Typical simulation of data set from GGMoE model.}\label{fig_typical_simulation}
	\end{figure*}

	\bibliography{reference}
	\bibliographystyle{abbrv}

\end{document}